\documentclass[nohyperref]{article}
\usepackage[usenames,dvipsnames]{xcolor}
\usepackage{lmodern}
\usepackage[accepted]{icml2023}
\usepackage{microtype}
\usepackage{booktabs}
\usepackage{graphics}
\usepackage{graphicx}
\usepackage{multirow}

\usepackage{url}
\usepackage{xparse}
\usepackage{amssymb}
\usepackage{xspace}
\usepackage{nicefrac}
\usepackage{microtype}
\usepackage{amsmath}
\usepackage{amsthm}
\usepackage{amsfonts}
\usepackage{enumerate}
\usepackage{setspace}
\usepackage{xspace}
\usepackage{tabularx}
\usepackage{enumitem}
\usepackage{wrapfig}
\usepackage{microtype}
\usepackage{graphicx}
\usepackage{subfigure}
\usepackage{booktabs} %
\usepackage{xfrac}
\usepackage{tikz}
\usepackage[most]{tcolorbox}
\usepackage{enumitem}

\usepackage{hyperref}
\hypersetup{
 colorlinks=True,
 linkcolor=magenta,
 citecolor=NavyBlue,
 urlcolor=magenta}

\usepackage{balance}

\icmltitlerunning{LIV: Language-Image Value Learning}

\newcommand{\BE}{\mathbb{E}}

\newtheorem{proposition}{Proposition}
\newtheorem*{proposition*}{Proposition}

\usepackage[colorinlistoftodos]{todonotes}
\begin{document}

\twocolumn[

\icmltitle{LIV: Language-Image Representations and Rewards for Robotic Control}

\begin{icmlauthorlist}
\icmlauthor{Yecheng Jason Ma}{penn}
\icmlauthor{William Liang}{penn}
\icmlauthor{Vaidehi Som}{penn} \\
\icmlauthor{Vikash Kumar}{meta}
\icmlauthor{Amy Zhang}{meta}
\icmlauthor{Osbert Bastani}{penn}
\icmlauthor{Dinesh Jayaraman}{penn}
\end{icmlauthorlist}

\icmlaffiliation{penn}{University of Pennsylvania}
\icmlaffiliation{meta}{Meta AI}

\icmlcorrespondingauthor{Jason Ma}{jasonyma@seas.upenn.edu}

\icmlkeywords{representation learning, pre-training, offline reinforcement learning}

\vskip 0.3in
]

\printAffiliationsAndNotice{}  %

\begin{abstract}
We present \textbf{L}anguage-\textbf{I}mage \textbf{V}alue learning (LIV), a unified objective for vision-language representation and reward learning from action-free videos with text annotations. 
Exploiting a novel connection between dual reinforcement learning and mutual information contrastive learning, the LIV objective trains a  multi-modal representation that implicitly encodes a universal value function for tasks specified as language or image goals. 
We use LIV to pre-train the first control-centric vision-language representation from large human video datasets such as EpicKitchen. 
Given only a language or image goal, the pre-trained LIV model can assign dense rewards to each frame in videos of unseen robots or humans attempting that task in unseen environments. Further, when some target domain-specific data is available, the same objective can be used to fine-tune and improve LIV and even other pre-trained
representations for robotic control and reward specification in that domain.
In our experiments on several simulated and real-world robot environments, 
LIV 
models consistently outperform the best prior input state representations for imitation learning, as well as reward specification methods for policy synthesis. Our results validate the advantages of joint vision-language representation and reward learning within the unified, compact LIV framework. Project website: \href{https://penn-pal-lab.github.io/LIV}{penn-pal-lab.github.io/LIV} 
\end{abstract}

\section{Introduction}
\label{sec:introduction}

What are the key machine learning challenges in building a general-purpose robot? Consider a home robot for non-expert end users. Such a robot must acquire common-sense knowledge applicable to generic homes, 
permitting it to operate from visual observations with some minimal proficiency right off the shelf. 
Then, it must be able to quickly and robustly adapt to the specifics of the user's home, conditioning its language understanding in the particular visual context of its new habitat.
Finally, it must be able to autonomously learn arbitrary new skills specified by its user, most naturally in plain language.

\begin{figure*}[t!]
\centering 
\includegraphics[width=0.9\textwidth]{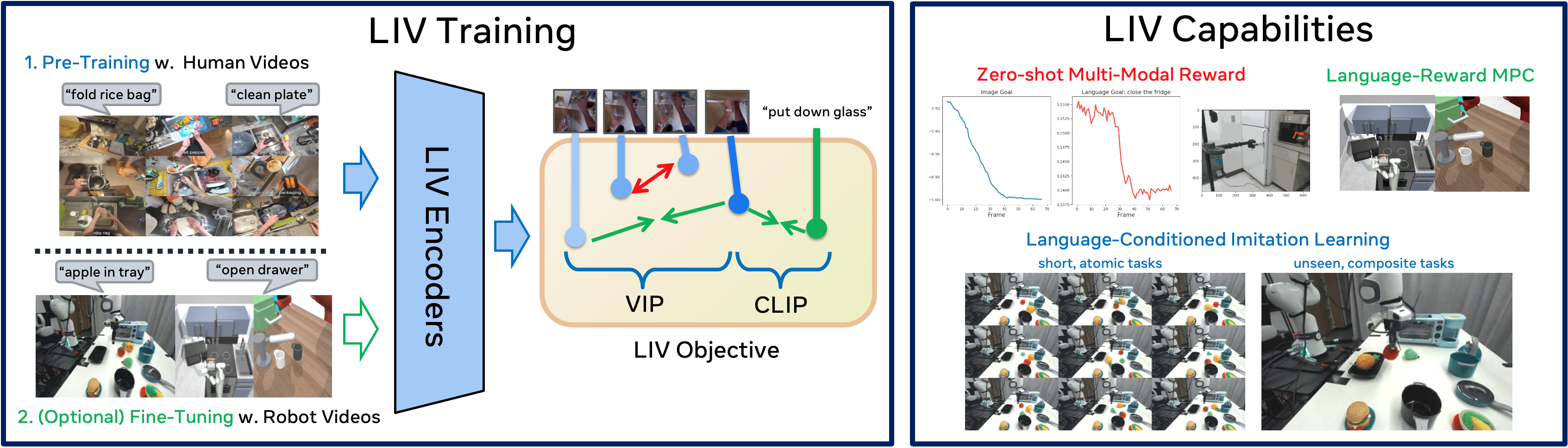}
\caption{\textbf{Language-Image Value Learning (LIV)} for vision-language reward and representation learning. Using the same objective for pre-training and fine-tuning, LIV induces a cross-modal embedding with both temporal coherence and semantic alignment. LIV's multi-modal value representations enable diverse visuomotor control applications.}
\vspace{-0.5cm}
\label{figure:concept}
\end{figure*}

Motivated by such considerations, we identify three key desiderata for control-aware vision-language representations. The first two deal with qualities of the trained representation: (1) It must \textbf{align the two modalities to permit grounding language specifications} for effective task representation, and (2) It must \textbf{capture task-directed progress grounded in language} to supply intermediate learning signals for autonomous skill acquisition. The last desideratum is concerned with how these control-aware VLMs must be trained. Language grounding is commonly context-dependent, so effective representations must be domain-aware. On the flip side, domain-specific data is typically expensive to collect and therefore scarce in robotics settings, making any domain-specific fine-tuning of large models challenging. Our third criterion for our vision-language representation then is that: (3) It must permit \textbf{both extensive domain-generic pre-training as well as domain-specific fine-tuning} on small datasets.

To achieve all three criteria, we propose \textbf{L}anguage-\textbf{I}mage \textbf{V}alue Learning (LIV), a unified objective 
for joint vision-language representation and reward learning.
LIV can flexibly pre-train representations on 
arbitrary video activity datasets with text annotations, even including purely observational datasets of human activity, for which there are several large and conveniently available options~\citep{damen2018scaling, grauman2022ego4d, goyal2017something}. Afterwards, the very same objective can be used to fine-tune those representations on small datasets of in-domain robot data, to overcome domain gaps and ground language in context-specific ways.

At a technical level, LIV builds on our prior work Value-Implicit Pre-Training (VIP)~\citep{ma2022vip}, a self-supervised approach for learning visual goal-conditioned value functions and representations from videos, generalizing it to learning multi-modal, vision-language values and representations from \textit{language-aligned} videos, as described above. Interestingly, we show that LIV is a more general variation of the well-known mutual-information based image-text contrastive representation learning objective,
as used in CLIP~\citep{radford2021learning}; this observation simplifies LIV's practical implementation to a simple yet principled combination of the VIP and CLIP objectives; see Figure~\ref{figure:concept} for an overview of LIV.

We perform extensive experimental evaluations on several simulated and \textit{real-world} household robotic manipulation settings. Our experiments evaluate LIV vision-language representations not only in their capacity as input state representations for language-conditioned behavior cloning of task policies, but also to directly ground language-based task specifications into visual state-based rewards for robot trajectory optimization, thereby stress-testing alignment across modalities. In many cases, the pre-trained LIV model, without ever seeing robots in its pre-training human video dataset, can \textit{zero-shot} produce dense language-conditioned reward on unseen robot videos. Along another axis of evaluation, we assess both the ``generic'' representations pre-trained on large human video datasets as well as the specialized representations fine-tuned on in-domain robot data. Our results comparing to several representative recent works from the three distinct categories of pre-training, fine-tuning, and reward learning, confirm the advantages of the LIV objective for joint vision-language representation and reward learning for control.

\section{Related Work}
\label{sec:related-work}

\textbf{Pre-trained Representations for Control.}
Our work is related to the literature on pre-training representations for control~\citep{shah2021rrl, cui2022can, parisi2022unsurprising, nair2022r3m, xiao2022masked, ma2022vip, fan2022minedojo, majumdar2023we}. These works all seek to use pre-existing data, typically out-of-domain, to pre-train effective representations for downstream unseen robotic tasks. Conceptually, VIP~\citep{ma2022vip} is closest to LIV in learning an implicit value function as joint reward and representation, but VIP focuses only on visual pre-training. Likewise, \citet{nair2022r3m} uses a language alignment loss~\citep{nair2022learning} with respect to a fixed language encoder~\citep{sanh2019distilbert} to shape the visual representation temporally, but the learned representation itself is still uni-modal. In this context, our work is the first to propose a multi-modal vision-language objective that is simultaneously suitable for pre-training, fine-tuning, and reward-learning for language-conditioned robotic control.
A concurrent work~\citep{karamcheti2023language} has proposed joint pixel and language reconstruction as a pre-training objective and focuses more on high-level robotics reasoning tasks, such as grasp affordance prediction and referring expression grounding.

\textbf{Fine-Tuning Pre-Trained Representations.}
Several recent works study how to adapt pre-trained representations for downstream tasks~\citep{kumar2022fine,wortsman2022model,ilharco2022patching,lee2022surgical,kirichenko2022last,goyal2022finetune,dong2022clip}, motivated by the emergence of large pre-trained models~\citep{radford2021learning,brown2020language} capable of zero-shot transfer. Most closely related to our work are a few concurrent works that find using the CLIP objective to fine-tune CLIP is more effective than alternative fine-tuning approaches~\citep{goyal2022finetune, dong2022clip}. However, in the setting of vision-language robotic control, we find CLIP to be a sub-optimal fine-tuning choice, even for the original CLIP model itself, due to the static image-text alignment nature of the CLIP objective that ignores the temporal structure in video data and discards non-goal frames in videos when fine-tuning image-level visual representation.
In contrast, LIV induces a natural synergy between VIP and CLIP that does not require any hyperparameter tuning and can fine-tuning various pre-trained vision-language representations for downstream robotic control.

\textbf{Language-Conditioned Robotic Manipulation.}
There has been a surge of interest in language-conditioned vision-based robotic manipulation systems~\citep{lynch2020language, stepputtis2020language, ahn2022can, jang2022bc, lynch2022interactive, shridhar2022cliport, brohan2022rt, shridhar2022perceiver, guhur2022instruction, liu2022instruction, mees2022calvin}. While several works have considered policy learning on top of pre-trained vision-language representations~\citep{shridhar2022cliport, liu2022instruction, mees2022matters}, they do not consider how a better representation can be learned in the first place by leveraging large-scale out-of-domain text-annotated video data. Our work is the first to study how to pre-train new control-centric vision-language representations that surpass existing representations such as CLIP for language-conditioned visuomotor robotic control tasks. 

Further, our solution, LIV, also affords flexibility in policy synthesis algorithms. 
Most existing works in this area focus on language-conditioned behavior cloning (LCBC)
~\citep{lynch2020language,stepputtis2020language}. This paradigm demands the expensive collection and text labeling of demonstration data, which can take months to complete~\citep{jang2022bc,lynch2022interactive, brohan2022rt}. In contrast, not only is LIV an effective pre-trained representation for LCBC, it can also be used as a language-conditioned visual reward model that supports autonomous skill acquisition via reinforcement learning~\citep{goyal2021pixl2r, nair2022learning, pmlr-v162-mahmoudieh22a}. Our experiments show that LIV outperforms prior state-of-art language-conditioned reward models~\citep{nair2022learning, nair2022r3m} in model-based planning settings.

\vspace{-0.2cm}
\section{Preliminaries \& Problem Setting}
\label{sec:preliminaries}

In this section, we review the VIP algorithm and describe our problem setting.

\textbf{Value Implicit Pre-Training (VIP).} 
VIP~\citep{ma2022vip} learns the optimal goal-conditioned value function via the dual goal-conditioned RL formulation~\citep{ma2022smodice, ma2022far}: 
\begin{equation}
\label{eq:vip}
{\small
\begin{split}
&\mathcal{L}(\phi)  = \BE_{p(g)}[(1-\gamma) \BE_{\mu_0(o;g)}\left[- \mathcal{S}(\phi(o);\phi(g))\right] \\ 
+& \log \BE_{(o,o';g) \sim D}\left[\mathrm{exp}\left(\mathcal{S}(\phi(o);\phi(g)) + 1 - \gamma \mathcal{S}(\phi(o');\phi(g)) \right)\right]],
\end{split}}
\end{equation}
where $\mu_0(o;g)$ is the distribution of initial frame conditioned on the goal frame $g$ and $D(o,o';g)$ is the goal-conditioned distribution of two successive intermediate frames.
In VIP, the value function is implicitly parameterized as a similarity metric (e.g., $L_2$ distance) in the embedding space $V(o;g) := \mathcal{S}(\phi(o);\phi(g))$, making it both a representation learning and a reward learning algorithm. Since it does not depend on actions, VIP can be pre-trained on large-scale human video datasets. The resulting implicit value function serves the dual purposes of (1) visual representation for unseen robot tasks, and (2) goal-conditioned dense reward specification. In particular, given a goal $g$, VIP assigns a potential-based reward at each time $t$: 
\begin{equation}
\label{eq:vip-reward}
R(o_t,o_{t+1};g) := \mathcal{S}(\phi(o_{t+1});\phi(g)) - \mathcal{S}(\phi(o_t);\phi(g))
\end{equation}

\paragraph{Vision-Language Representation Pre-Training for Control.} 
We assume access to a dataset of language-annotated videos $D = \{v_i := (o_1^i,...,g^i; l^i)\}_{i=1}^N$, where each $o \in O:=\mathbb{R}^{H \times W \times 3}$ is a raw RGB image, $g^i$ the last frame of the video, and $l^i$ is the textual annotation associated with $v^i$, describing the video outcome in $g^i$.
As the video dataset can be out-of-domain with respect to our robot agent (e.g., human videos), we do not assume access to action labels. Datasets of this nature, such as human daily activity videos~\citep{damen2018scaling, miech2019howto100m,grauman2022ego4d}, are readily available for research use. A pre-training algorithm $\mathcal{A}$ ingests this training data and
outputs vision-language encoders
$(\phi, \psi) := \mathcal{A}(D)$, where the vision encoder $\phi : \mathbb{R}^{H \times W \times 3} \rightarrow \mathbb{R}^K$ and the language encoder $\psi: L \rightarrow \mathbb{R}^K$, where $L$ is the space of natural strings, each map to the same $K$-dimensional vision-language representation space.

A standard way to learn a vision-language representation is by learning a cross-modal joint-embedding~\citep{lecun2022path} that aligns the modalities via contrastive learning. Specifically, the two modalities are semantically aligned by minimizing the InfoNCE objective~\citep{oord2018representation}:
\vspace{-0.2cm}
\begin{equation}
\label{eq:infonce-first}
{\small
\mathcal{L}_{\text{InfoNCE}}(\phi, \psi) = \BE_{p(o,l)}\left[-\log \frac{ e^{\mathcal{S}(\phi(o);\psi(l))}}{\BE_{D(o')}\left[e^{\mathcal{S}(\phi(o');\psi(l))}\right]} \right],}
\end{equation}
where $\mathcal{S}$ is a choice of similarity metric. Intuitively, this objective aims to attract the representations of matching image-text pairs $(o,l)$, while repelling mismatching pairs. Many state-of-art vision-language models~\citep{radford2021learning, jia2021scaling, li2022blip} train with this InfoNCE objective at scale to deliver strong zero-shot performance on a myriad of vision-language tasks.

\vspace{-0.3cm}
\section{LIV: Language-Image Value Learning}
\label{sec:method}

\subsection{Algorithm}
We begin by extending the VIP framework to multi-modal goal specifications. This is straightforward given the goal-conditioned nature of Eq.~\eqref{eq:vip}, since we can simply replace encoded image goal $\phi(g)$ with encoded text goal $\psi(l)$ and optimize for a \textit{multi-modal} VIP objective:
\begin{equation}
\label{eq:vip2}
{\small
\begin{split}
&\mathcal{L}(\phi, \psi)  = \\ 
&+\BE_{p(g)}[(1-\gamma) \BE_{\mu_0(o;g)}[-\mathcal{S}(\phi(o);\phi(g))] \\ 
& \underbrace{+ \log \BE_{(o,o';g) \sim D}\left[\mathrm{exp}\left(\mathcal{S}(\phi(o);\phi(g)) + 1 - \gamma \mathcal{S}(\phi(o');\phi(g)) \right)\right]]}_{\text{VIP-I}} \\ 
& + \BE_{p(l)}[(1-\gamma) \BE_{\mu_0(o;l)}\left[- \mathcal{S}(\phi(o);\psi(l))\right] \\ 
& \underbrace{+ \log \BE_{(o,o';l) \sim D}\left[\mathrm{exp}\left(\mathcal{S}(\phi(o);\psi(l)) + 1 - \gamma \mathcal{S}(\phi(o');\psi(l)) \right)\right]]}_{\text{VIP-L}}
\end{split}}
\end{equation}

As shown, this objective consists of two independent components; VIP-I (Image) encourages the representation to encode an \textit{image} goal-conditioned value function, 
and likewise, VIP-L (Language) for \textit{language} goal. %

At first glance, this objective does not appear to be directly optimizing for semantic alignment between goals in the two modalities, 
as the respective modality-specific VIP objectives are independently optimized. Without alignment, semantically equivalent goals in the respective modality may actually be distant in the representation space. This is undesirable for reward specification, which requires visual grounding of linguistic task descriptions. 
Intriguingly, in the next section, we show that such semantic alignment is in fact automatically achieved from optimizing Eq.~\eqref{eq:vip2}.

\subsection{Theoretical Analysis}
\label{sec:theoretical-analysis}
Now, we show that Eq.~\eqref{eq:vip2} in fact naturally optimizes semantic alignment between multi-modal goals with a simple data augmentation applied to the training videos.
Specifically, if we were to consider a \textit{degenerate} distribution of videos, i.e., videos consisting of solely static text-aligned frames $v=((o,o);l)$, we recover InfoNCE from VIP-L:
\begin{proposition}
\label{proposition:vip-to-clip}
Let the video distribution consist of solely degenerate videos of repeated frames that align with the text annotation, $D := \{v:=((g,g;l))\}$. Then, the VIP-L objective is equivalent to the InfoNCE objective up to a constant: 
\begin{equation}
\label{eq:infonce}
\resizebox{\columnwidth}{!}{
$
\mathcal{L}_{\text{VIP-L}}(\phi, \psi) = \BE_{p(g,l)}\left[-\log \frac{ e^{(1-\gamma) \mathcal{S}(\phi(g);\psi(l))}}{\BE_{D(g')}\left[e^{(1-\gamma) \mathcal{S}(\phi(g');\psi(l))}\right]} \right]+1,$}
\end{equation}
where $p(g,l)$ is the distribution of goal frame and text pair.
\end{proposition}
The proof is in Appendix~\ref{appendix:proof}. This result, though simple to derive, has several important implications. First, note that Eq.~\eqref{eq:infonce} is precisely what CLIP~\citep{radford2021learning} optimizes (Eq.~\eqref{eq:infonce-first}, modulo the constant factor) by contrastively learning similarity between matching image-text pairs. The fact that this objective can be obtained by optimizing VIP-L with a degenerate video distribution suggests that VIP-L is a natural generalization of the InfoNCE objective to the decision making setting, where the data is temporal. In practice, as we will show, this degenerate video distribution can be trivially obtained by augmenting any existing annotated video in the dataset by repeating the last frame.

This finding also directly suggests a method for \textit{fine-tuning} pre-trained contrastive vision-language models (e.g., CLIP) for control: use LIV on in-domain labeled videos such as text-annotated robot demonstrations. 
Several concurrent works~\citep{goyal2022finetune,dong2022clip} have suggested that fine-tuning a pre-trained model using the same objective (in particular, using CLIP objective to fine-tune CLIP model)
can be more effective than fine-tuning using the downstream task objective.
When working with CLIP-like pre-trained contrastive joint-embeddings, it is then natural to fine-tune them for control with the LIV objective, which is but a natural extension of CLIP that exploits sequential, goal-directed video data. 

As we show in our experiments, using the LIV objective to fine-tune the pre-trained CLIP model is far more effective than using the CLIP objective. CLIP fine-tuning aligns the last frame in the video to its text description but fails to leverage earlier frames from the same video sequence.

\begin{algorithm*}[t]
\caption{Language-Image Value Learning (LIV)}
\label{algo:vip}
\begin{algorithmic}[1]
\small
\STATE \textbf{Require}: Offline text-annotated (human) videos $D = \{ (o^i_1,...,g^i;l^i)\}_{i=1}^N$, vision-language architecture $(\phi, \psi)$ 
\FOR{\text{number of training iterations}}
\STATE Sample sub-trajectories $\{o^i_{t}, ..., o^i_{k}, o^i_{k+1},..., g^i; l^i\}_{i=1}^B\sim D, t \in [1, h_i-1], t\leq k < h_i, 
\forall i$

\STATE
\hspace{-0.15cm} \resizebox{0.95\textwidth}{!}{
$\mathcal{L}_{\text{VIP-I}}(\phi) :=\frac{1-\gamma}{B} \sum_{i=1}^B \left[-\mathcal{S}(\phi(o^i_t);\phi(g^i))\right] + \log \frac{1}{B} \sum_{i=1}^B \mathrm{exp}\left[\mathcal{S}(\phi(o^i_k);\phi(g^i)) +1 - \gamma \mathcal{S}(\phi(o^i_{k+1});\phi(g^i))\right]$}
\STATE 
$\mathcal{L}_{\text{InfoNCE}}(\phi, \psi) := \frac{1-\gamma}{B} \sum_{i=1}^B \left[-\log \frac{ e^{(1-\gamma)\mathcal{S}(\phi(g^i);\psi(l^i))}}{\frac{1}{B} \sum_{j = 1}^B\left[e^{(1-\gamma)\mathcal{S}(\phi(g^j);\psi(l^i))}\right]} \right] $ 

\STATE Update $(\phi, \psi)$ using SGD: $(\phi, \psi) \leftarrow (\phi, \psi) - \alpha\nabla (\mathcal{L}_{\text{VIP-I}}(\phi) + \mathcal{L}_{\text{InfoNCE}}(\phi, \psi))$
\ENDFOR
\end{algorithmic}
\end{algorithm*}

\subsection{Implementation}
Based on the analysis above, we see that, despite initial appearances, 
Eq.~\eqref{eq:vip2} does in fact naturally induce semantic alignment between visual and language goals.
In particular, it is implicitly optimizing for a pathway that connects the two modalities via mutual information maximization. Given this pathway that makes goals interchangeable across modalities, the final LIV objective optimizes the VIP objective in just one modality in conjunction with the vision-language InfoNCE objective in Eq.~\eqref{eq:infonce}, which we find to be a simple yet effective objective:

\vspace{-0.5cm}
\begin{equation}
\label{eq:vip2-practical}
{\small
\begin{split}
&\mathcal{L}_{\text{LIV}}(\phi, \psi)  = \\ 
&+\BE_{p(g)}[(1-\gamma) \BE_{\mu_0(o;g)}[- \mathcal{S}(\phi(o);\phi(g))] \\ 
& \underbrace{+ \log \BE_{(o,o';g) \sim D}\left[\mathrm{exp}\left(\mathcal{S}(\phi(o);\phi(g)) + 1 - \gamma \mathcal{S}(\phi(o');\phi(g)) \right)\right]]}_{\text{VIP-I}} \\ 
& \underbrace{+ \BE_{p(g,l)}\left[-\log \frac{ e^{(1-\gamma) \mathcal{S}(\phi(g);\psi(l))}}{\BE_{D(g')}\left[e^{(1-\gamma) \mathcal{S}(\phi(g');\psi(l))}\right]} \right]}_{\text{InfoNCE}},
\end{split}}
\end{equation}
We use a $\gamma$-weighted cosine similarity metric for $\mathcal{S}(\phi(\cdot), \psi(\cdot)):= \frac{1}{1-\gamma}CS(\phi(\cdot), \psi(\cdot))$ so it represents a valid value function.
Pseudocode is presented in Algorithm~\ref{algo:vip}. 
In each gradient step, a minibatch of video sub-clip consisting of initial, intermediate, and final frames are sampled along with the corresponding text annotations. These samples are used to estimate the VIP-I and InfoNCE losses, which then update the vision-language architecture via back-propagation. 

\textbf{LIV Pre-Training.} We have shown above that the LIV objective subsumes CLIP-style contrastive objectives. In implementing LIV, we stay close to CLIP architecture and design choices to allow fair comparison to pre-trained CLIP with ResNet50~\citep{he2016deep} vision backbone. 
 Initialized with CLIP weights, we pre-train LIV on EpicKitchen~\citep{damen2018scaling}, a text-annotated ego-centric video dataset of humans completing tasks in diverse household kitchens; this dataset consists of 90k video segments, totalling 20M frames and 20k unique text annotations, and offers diverse camera views and action-centric videos, making it an ideal choice for vision-language pre-training. See Appendix~\ref{appendix:model-details} for more pre-training details.

 \begin{figure*}[t!]
\centering

\subfigure[\texttt{open cabinet} (Human)]{\label{fig:opencabinet}\includegraphics[trim=13cm 0cm 0cm 0cm, clip, width=0.49\textwidth]{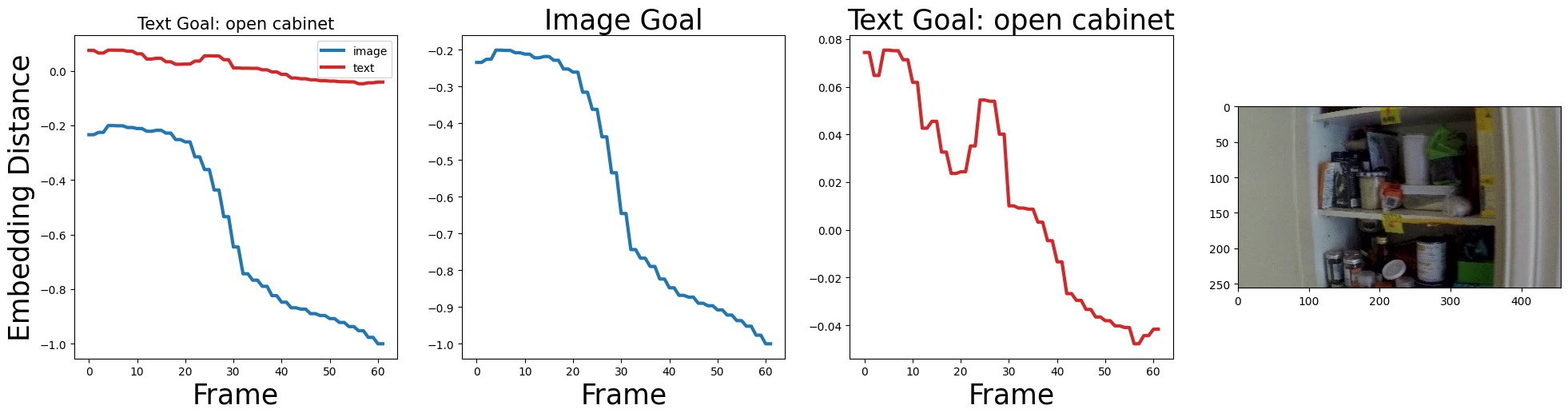}}
\subfigure[\texttt{open microwave} (Human)]{\label{fig:openmicrowave}\includegraphics[trim=13cm 0cm 0cm 0cm, clip, width=0.49\textwidth]{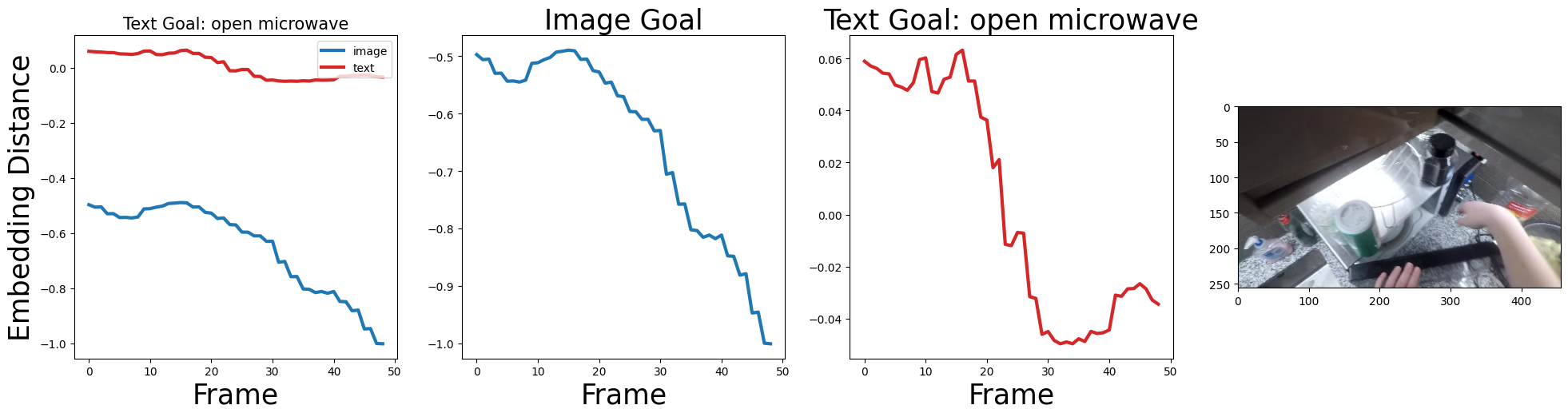}}

\subfigure[\texttt{close the fridge} (Robot)]{\label{fig:closefridge}\includegraphics[trim=18cm 0cm 0cm 0cm, clip, width=0.49\textwidth]{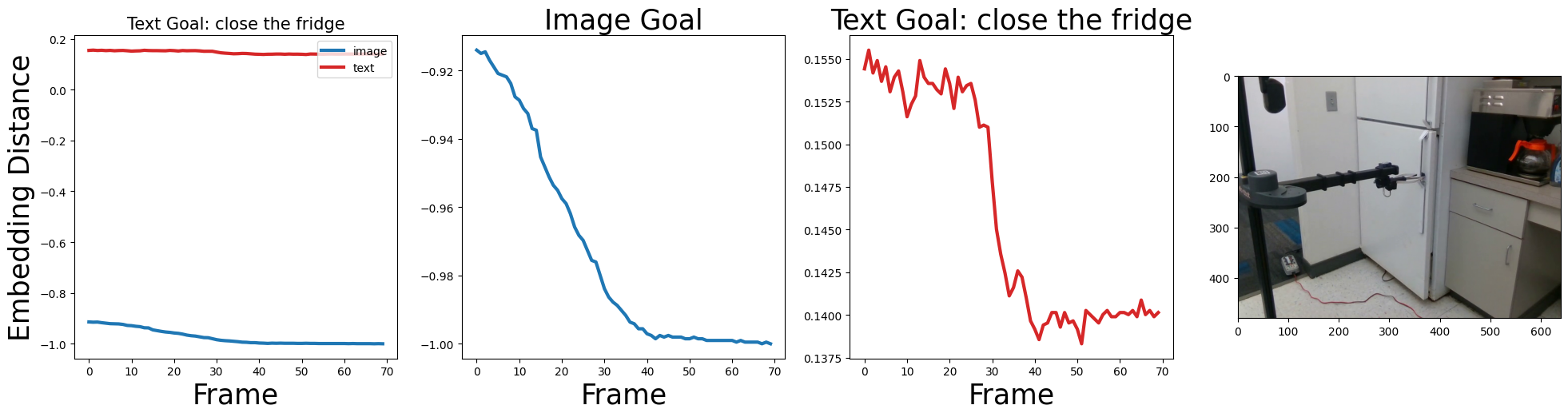}}
\subfigure[\texttt{put the hat on the bottle} (Robot)]{\label{fig:puthatonbottle}\includegraphics[trim=18cm 0cm 0cm 0cm, clip, width=0.49\textwidth]{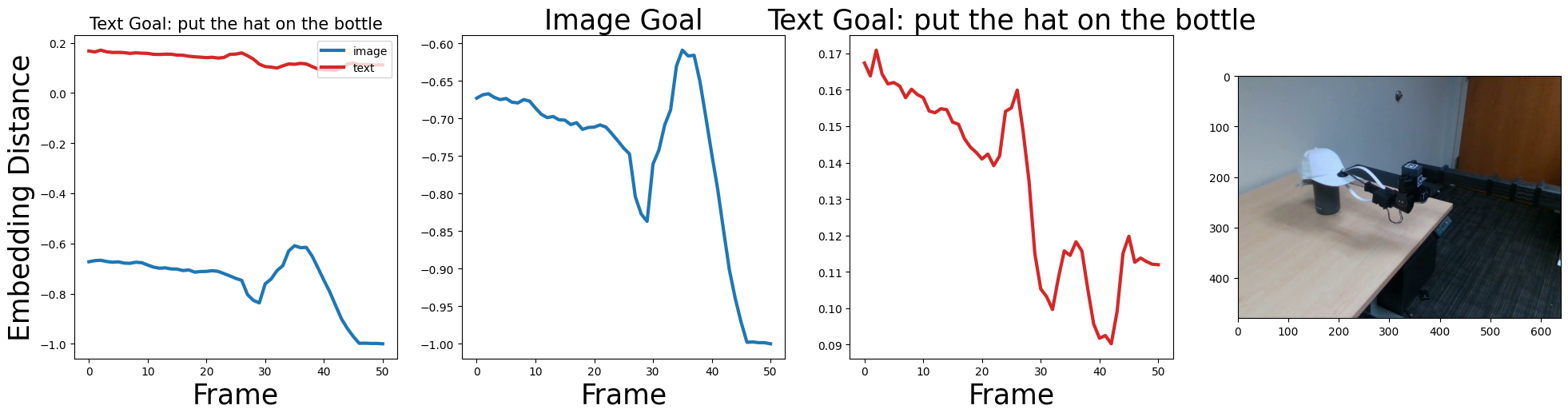}}

\caption{\textbf{LIV Zero-Shot Multi-Modal Cost on Unseen Human and Robot Videos.} The y-axis is the real-valued negative cosine similarities computed in the LIV embedding space.}
\label{fig:zero-shot-reward}
\end{figure*}

\section{Experiments}
\label{sec:experiments}

Our experiments aim to answer the following questions: 
\begin{enumerate}[topsep=0pt,itemsep=0ex,partopsep=1ex,parsep=0ex,leftmargin=3ex]
\item Can LIV produce multi-modal goal-conditioned rewards? 
\item Does pre-trained LIV enable effective vision-language representations for control?
\item Can LIV successfully fine-tune pre-existing vision-language models?
\end{enumerate}

To assess LIV's reward learning capability, we assess whether the pre-trained LIV can zero-shot provide multi-modal rewards for unseen text-annotated human and robot videos (Section~\ref{section:pre-training-reward}) and use its reward function for model-based planning to solve language-specified tasks (Section~\ref{section:reward-learning}). We evaluate LIV's effectiveness for pre-training (Section~\ref{section:pre-training}) and fine-tuning (Section~\ref{section:fine-tuning}) by using the resulting representations as the vision-language backbone in language-conditioned imitation learning (LCBC) in both simulations and a real robot platform.
LIV model and training code are released: \href{https://github.com/penn-pal-lab/LIV}{github.com/penn-pal-lab/LIV.} Our qualitative results, including animated reward curves and real-robot videos, are best viewed on our project website: \href{https://penn-pal-lab.github.io/LIV}{penn-pal-lab.github.io/LIV}

\begin{figure}[t!]
\centering
\subfigure[MetaWorld]{\label{fig:metaworld}\includegraphics[width=0.45\columnwidth]{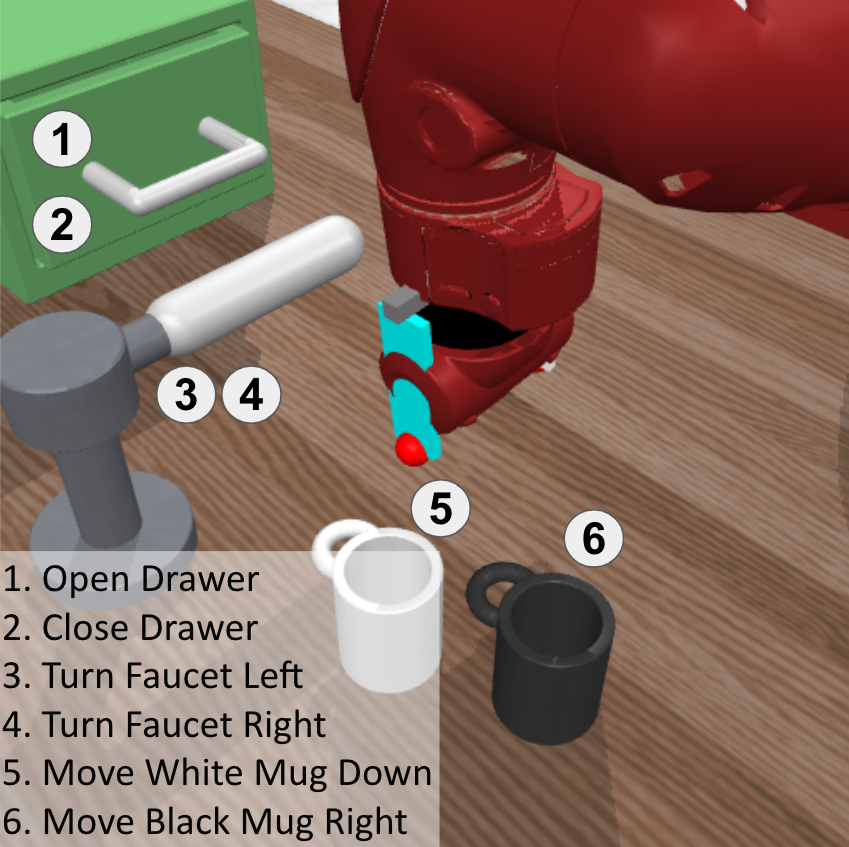}}
\subfigure[FrankaKitchen]{\label{fig:frankakitchen}\includegraphics[width=0.455\columnwidth]{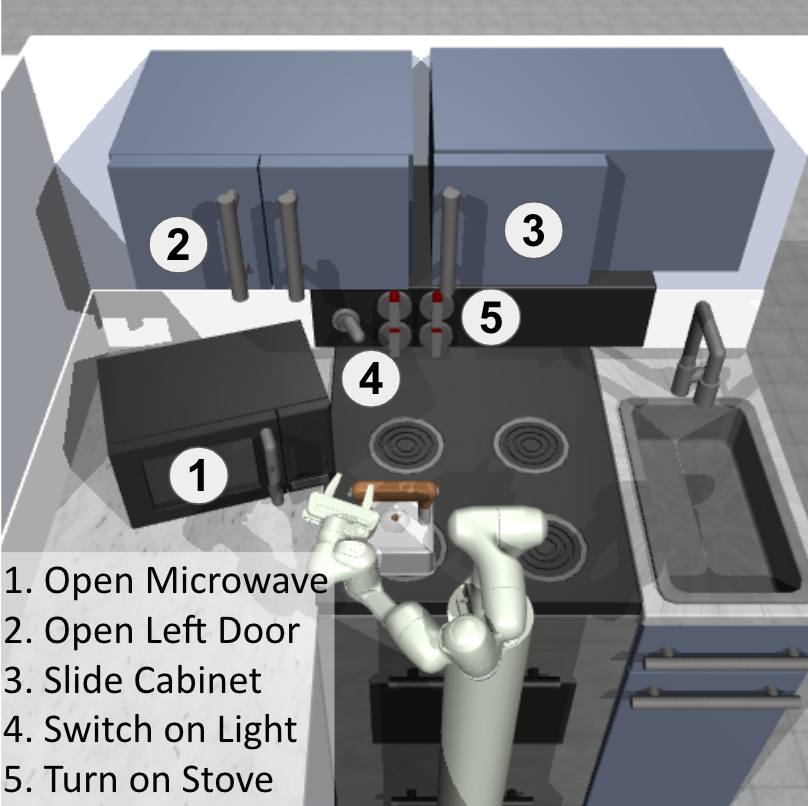}}
\subfigure[RealRobot]{\label{fig:realworld}\includegraphics[width=0.9\columnwidth]{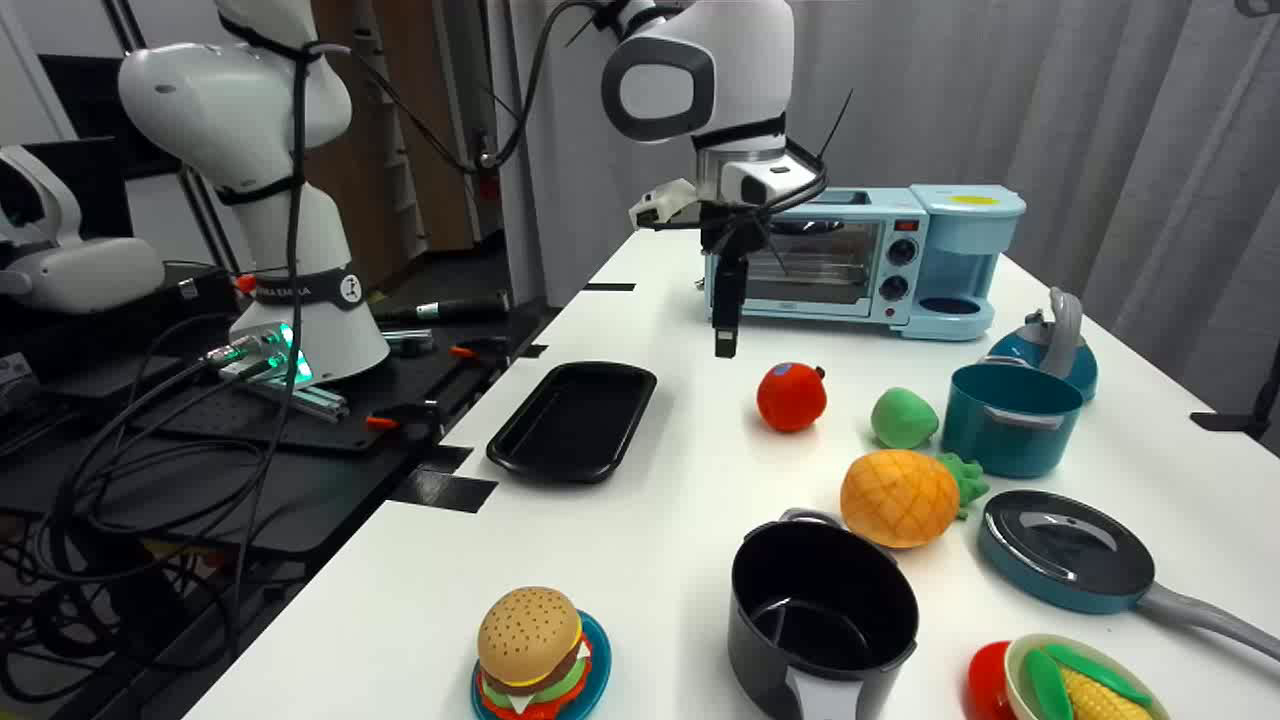}}
\resizebox{\columnwidth}{!}{
\begin{tabular}{llrrrrrrrrrrr}
\toprule
Environment & Train Tasks & Test Tasks & Horizon & Dataset Size & Dataset Type     \\
\midrule
MetaWorld & 1000 & 6 & 20 & 1M & Random \\ 
FrankaKitchen & 5 & 5 & 50 & 12.5K & Machine Demos \\ 
RealRobot & 9 & 9 & 100 & 90k & Human Teleoperation \\ 
\bottomrule 
\end{tabular}}
\caption{Multi-Task Vision-Language Environments.}
\label{figure:environments}
\end{figure}

\begin{figure*}[t!]
\centering
\includegraphics[width=0.9\textwidth]{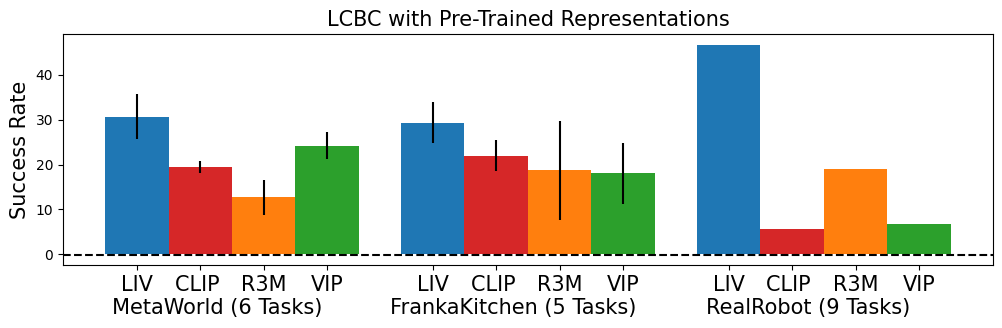}
\vspace{-0.3cm}
\caption{\textbf{Pre-Trained Representations for Language-Conditioned Behavior Cloning:} LIV achieves the highest average success rates across three distinct environments.}
\label{fig:pre-training-results}
\end{figure*}

\subsection{Pre-Trained LIV as Zero-Shot Reward}
\label{section:pre-training-reward} 
Recall that LIV objective encourages similarities in the embedding space to encode multi-modal goal-conditioned value functions. Intuitively, on videos depicting direct progress towards a goal, the distances (negative cosine similarities $-\mathcal{S}(\phi(o), \phi(g))$ or $-\mathcal{S}(\phi(o), \psi(l))$) between video frames $o$ and image or language goals $g$ or $l$, should steadily fall over time, reflecting progression towards those goals. Before using the embeddings for policy learning, we therefore first ask: to what extent does LIV pre-training achieve this property, and does it hold on new, unseen domains with robots?
For unseen, goal-directed human videos from the test split of EpicKitchen, Figure~\ref{fig:zero-shot-reward} (a\&b) plot frame distances from image and language goals, validating that this property indeed holds within the training distribution. However, to use LIV for training robots, we are interested in generalization to robot videos. This is challenging, since these pre-trained LIV models have never encountered any robot data. Yet, as Figure~\ref{fig:zero-shot-reward} (c\&d) show, LIV embedding distances are still informative of task progress, indicating transferability to unseen viewpoints and embodiments. Interestingly, in Figure~\ref{fig:zero-shot-reward} (d), we observe a bump in the middle of both cost curves; upon examination, this corresponds to a portion of the video where the robot lifts the hat unnecessarily high for the task of putting the hat on the bottle. In other words, the cost bump is indicative of sub-optimal actions.

In Appendix~\ref{appendix:qualitative-hellorobot}, we include more examples, such as untrimmed robot videos where LIV rewards can effectively distinguish the opposite actions (e.g., \texttt{open the fridge} and \texttt{close the fridge}) that both appear in the video, and some failure cases.  
These results are promising; the fact that LIV embedding distances track task value functions means that we can use them to assign dense rewards (Eq.~\eqref{eq:vip-reward}) based purely on image or language goal specifications. Throughout these results, the distance plots are less smooth with respect to the \textit{language} goal than the \textit{image} goal, due to the language grounding gap; we show later (Section~\ref{section:reward-learning}) that this can be effectively resolved with in-domain LIV fine-tuning.
Finally, we also plot these distance progression plots with CLIP in Appendix~\ref{appendix:qualitative-epickitchen} and find that LIV's zero-shot reward capability is largely absent in CLIP. 
In the rest of our experiments, we present several ways in which the LIV pre-trained model and training objective can aid in language-conditioned robotic manipulation.

\subsection{Policy Learning Environments}
We consider three multi-task language conditioned visual manipulation environments, spanning two simulation and one real robot setup. The two simulated environments extend the MetaWorld~\citep{yu2020meta} and FrankaKitchen~\citep{gupta2019relay} benchmarks. The MetaWorld benchmark is taken from~\citet{nair2022learning}, which has also released a dataset consisting of 1M transitions collected via random actions for policy learning; the trajectories are labeled with task descriptions based on true environment state. The FrankaKitchen benchmark takes existing tasks supported in the environment but makes them specified via fixed language descriptions; we use the tasks and the dataset from~\citet{nair2022r3m}. 

The real-world environment (referred to as RealRobot) consists of a table-top toy kitchen setup, in which a Franka robot is tasked with placing various fruits, \texttt{\{apple, pear, pineapple\}} in various containers, \texttt{\{tray, black pot, green pot\}} in the scene given a sentence task description (e.g., \texttt{apple in black pot}). Unlike prior works that have considered simplified action space~\citep{shridhar2022cliport} or reduced action frequency~\citep{brohan2022rt, mees2022matters}, we use 6-DOF end-effector displacement as the action space with 15Hz control. For each task, we collect 100 trajectories using human teleoperation with the fruits randomly initialized in the center workspace of the table for each trajectory. This environment is more challenging than the simulated ones because it requires grounding language to fine-grained spatial understanding of the scene in order to pick up the correct fruit and place it into the correct container. The environments and associated tasks are illustrated in Figure~\ref{figure:environments}; for all three environments, the visual observation consists of the 3rd-person view shown in the figure. For RealRobot, we additionally include a robot wrist view. See Appendix~\ref{appendix:environments} for more details on these environments and datasets.

\subsection{Pre-Trained LIV as Representation}
\label{section:pre-training} 
Pre-trained on diverse annotated human videos accomplishing daily household tasks, we posit that LIV can serve as an effective multi-modal representation for language-conditioned robotics manipulation.

\paragraph{Baselines.} We compare against \textbf{CLIP}~\citep{radford2021learning}, a state-of-art vision-language representation that has seen wide adoption in various robotics tasks~\citep{shridhar2022cliport,cui2022can,khandelwal2022simple, tam2022semantic}; as LIV is trained using the CLIP architecture and initialized with CLIP weights, this is the closest comparison. 
We also compare against \textbf{R3M}~\citep{nair2022r3m} and \textbf{VIP}~\citep{ma2022vip}, two state-of-art pre-trained visual representations. While unimodal, both are strong baselines; they are pre-trained on ego-centric videos similar to EpicKitchen~\citep{grauman2022ego4d} and employ similar vision architecture (ResNet50) as LIV. We adapt them to the vision-language setting by coupling them with a pre-trained DistilBERT encoder~\citep{sanh2019distilbert} to process language input. We note that R3M does employ this very same model for shaping its visual representation during training, making DistilBERT a natural design choice.

\paragraph{Policy Learning and Evaluation.} 
Our experimental protocols closely follow prior works on evaluating pre-trained representations for robotic manipulation both in simulation~\citep{nair2022r3m, xiao2022masked} and on a real robot platform~\citep{ma2022vip}. At a high level, we perform language-conditioned behavior cloning (LCBC), where a single multi-task policy, which takes concatenated current observation embedding and language task embedding as input, is trained for all tasks within an environment using the given environment dataset. The representations are kept frozen during policy learning, and we employ a simple MLP architecture on top of the pre-trained representations for the policy network. For the simulation environments, as in~\citep{nair2022r3m}, we report the success rate of the best training checkpoints on 20 evaluation rollouts per task.
For the real environment, we evaluate only the final checkpoint on 10 evaluation rollouts per task due to the cost of real-world policy evaluation~\citep{mandlekar2021matters}. For each backbone representation, we train policies using 3 random seeds in simulation and report the mean and the standard deviation of the success rates over all evaluation episodes; on RealRobot, we train one policy per backbone representation. See Appendix~\ref{appendix:lcbc} for additional details on training hyperparameters.

\begin{figure*}[t!]
\centering
\includegraphics[width=0.9\textwidth]{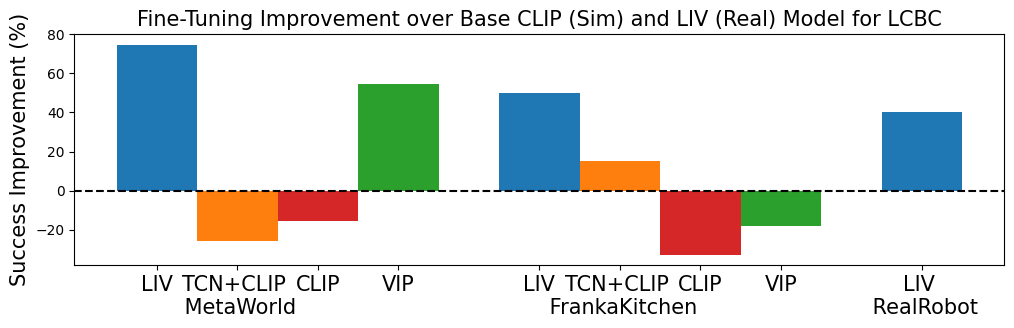}
\vspace{-0.3cm}
\caption{\textbf{LIV is an effective fine-tuner for pre-trained vision-language representations for robotic control.} Compared to the base model performance (dash line at 0), LIV fine-tuning consistently improves success rate by more than 40\% in all environments.}
\label{figure:fine-tuning-aggregate}
\end{figure*}

\paragraph{Results.} 
Full results are reported in Figure~\ref{fig:pre-training-results} (full numeric results in Appendix~\ref{appendix:lcbc}). 
As shown, our pre-trained LIV model, without any in-domain fine-tuning, performs best in all environments. In particular, LIV gains are largest 
on RealRobot, which is as expected since it is pre-trained on real data.
R3M and VIP are also both pre-trained on human videos, yet they do not achieve the same level of performance as LIV, suggesting the importance of a proper vision-language pre-training objective for control that cannot be substituted by ad-hoc combinations of vision-only pre-trained representations and language models. This finding holds true especially on the real-world environment. There, given the lack of any ``shortcut'' visual cues to infer the correct action (e.g., objects of interest always appear centrally in the image), an effective vision-language representation needs to facilitate learning language-grounded hand-eye coordination for effective policy learning. Indeed, in Figure~\ref{figure:realrobot_loss}, Appendix~\ref{appendix:lcbc-additional-results}, we visualize the BC training loss each policy realizes on RealRobot and find that the policy with LIV backbone achieves significantly lower training loss than baselines; this result indicates that the baselines underfit as they fail to distinguish the correct action based on the conditioning text description and consequently achieve much lower success rates. In Appendix~\ref{appendix:lcbc-additional-results}, we also include an additional experiment ablating language task encoding with one-hot encoding. There, we find that LIV provides the best contextual representation~\citep{sodhani2021multi} that improves policy learning even for our simulation tasks that do not heavily rely on language for task disambiguation.

Building on these real-world results, in Appendix~\ref{appendix:realrobot-longhorizon}, we further find that the LIV policy can even \textit{zero-shot} generalize to long-horizon composite tasks consisting of chained atomic tasks. This requires out-of-distribution generalization to unseen starting configurations for the atomic tasks.

\subsection{Fine-Tuning}
\label{section:fine-tuning}

Next, we show that the LIV objective can also be used to effectively fine-tune pre-trained vision-language models for downstream policy learning. Specifically, we first take the same in-domain task data as in Section~\ref{section:pre-training} to fine-tune the pre-trained representations using the LIV objective (Algorithm~\ref{algo:vip}) as well as several alternatives (see below). Then, as before, we freeze the fine-tuned representations and train policies on top using LCBC.

\textbf{Baselines.} The LIV objective can be understood as a principled combination of VIP and CLIP, so we consider a baseline that amounts to an alternative combination of a visual self-supervised learning (SSL) objective and the CLIP objective. In particular, we employ time contrastive learning (TCN)~\citep{sermanet2018time} because TCN is a SSL objective that also utilizes temporal information and has been incorporated in a prior work~\citep{nair2022r3m}. 
For this baseline, we sweep through $\alpha=\{0.1,0.3,0.5,0.8\}$ and report the best value for weighing the two terms in the combined objective: $\alpha \text{TCN} + (1-\alpha)\text{CLIP}$. In addition, we consider using the \textbf{CLIP} InfoNCE objective (Eq.~\eqref{eq:infonce}) in isolation, as well as the \textbf{VIP-I} objective (Eq.~\eqref{eq:vip}); these fine-tuning methods are ablations of LIV that focus only on semantic (CLIP) or temporal value (VIP) alignment. To isolate the impact of LIV fine-tuning independently of LIV pre-training, we use the pre-trained CLIP as the base model to be fine-tuned using LIV and the baselines; In Appendix~\ref{appendix:fine-tuning-additional-results}, we also present results on fine-tuning the pre-trained LIV. 
On RealRobot, due to the cost of real-world evaluation, we evaluate only LIV fine-tuning to see whether a strong fine-tuning algorithm vetted in simulation can remain effective in the real world. Likewise, we use the pre-trained LIV as the base model since CLIP LCBC performs poorly (Section~\ref{section:pre-training}). 

\textbf{Results.} The relative performance improvements in percentages are shown in Figure~\ref{figure:fine-tuning-aggregate}. LIV fine-tuning substantially improves policy success rates on all three environments that vary significantly in terms of in-domain dataset size and quality. On RealRobot, we find that LIV fine-tuning remains effective and significantly improves the already strong base LIV model, validating its efficacy for real-world usage in both the pre-training and the fine-tuning stages. in Appendix~\ref{appendix:realrobot-reward-curves}, we compare the embedding distance curves of LIV before and after LIV fine-tuning and show that LIV fine-tuning indeed acquires a smoother multi-modal representation. Qualitatively, we observe that the policy trained using fine-tuned LIV generates more coordinated and smoother grasping and placing motions that are critical for task success on RealRobot. 

The weak performance of CLIP and VIP demonstrates that neither semantic or temporal-perception fine-tuning alone is sufficient for robotics manipulation; in the case of FrankaKitchen, due to the smaller dataset size, both ablations overfit and in fact decrease policy performance. TCN+CLIP similarly delivers mixed results, highlighting that LIV's fine-tuning capability cannot be easily obtained by combining another SSL objective, even one that considers temporal information, with the CLIP objective. Furthermore, given that TCN and CLIP do not bear natural connection to one another, we find their combination to be quite sensitive to the weighting parameter $\alpha$; see Appendix~\ref{appendix:fine-tuning-additional-results} for the numerical results for a comprehensive sweep over $\alpha$ for TCN+CLIP. In particular, whereas $\alpha = 0.5$ works best on FrankaKitchen, the same value leads to diverged training on MetaWorld; decreasing $\alpha$ to $0.1$ prevents divergence on MetaWorld, but the resulting policy is even worse than that of the base CLIP model on MetaWorld and substantially worse than $\alpha=0.5$ on FrankaKitchen. LIV does not suffer from this issue; our theoretical result in Section~\ref{sec:theoretical-analysis} suggests that the VIP and CLIP components in LIV should be weighed exactly in one-to-one, and indeed, we find that this ratio works well on all environments, requiring no hyperparameter tuning. In Appendix~\ref{appendix:fine-tuning-additional-results}, we further show that LIV fine-tuning is effective for different base models (e.g., LIV) and different in-domain dataset sizes, showcasing its versatility over a large spectrum of domain specificities.

\begin{figure}[t!]
\centering
\subfigure[Pre-Trained LIV]{\label{fig:liv-pre-trained}\includegraphics[trim=1cm 0cm 25cm 0.8cm, clip, width=0.32\columnwidth]{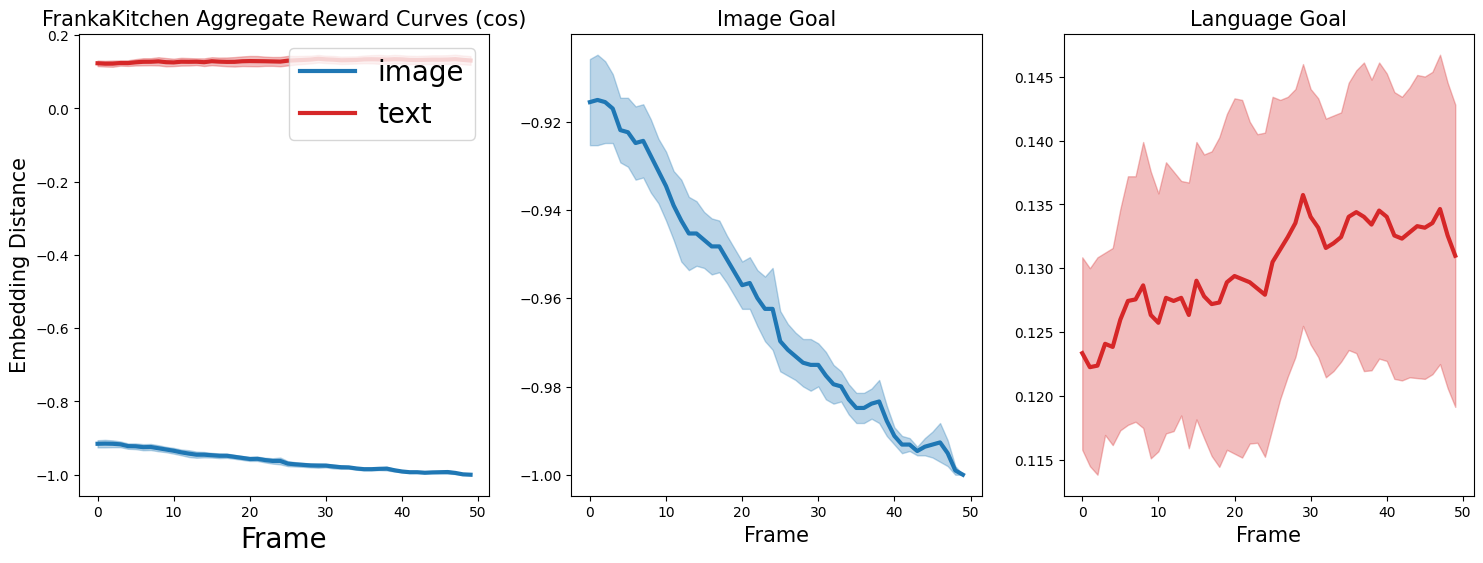}}
\subfigure[LIV FT]{\label{fig:liv-fine-tuning}\includegraphics[trim=1cm 0cm 25cm 0.8cm, clip, width=0.32\columnwidth]{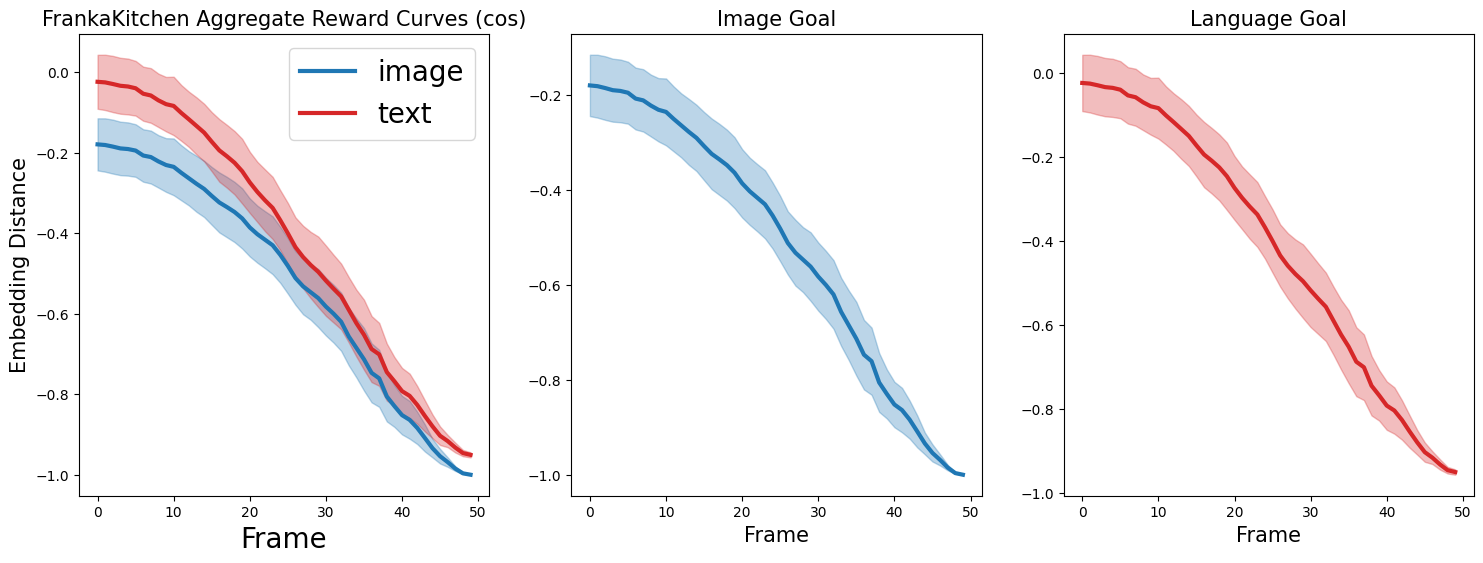}}
\subfigure[CLIP FT]{\label{fig:clip-fine-tuning}\includegraphics[trim=1cm 0cm 25cm 0.8cm, clip, width=0.32\columnwidth]{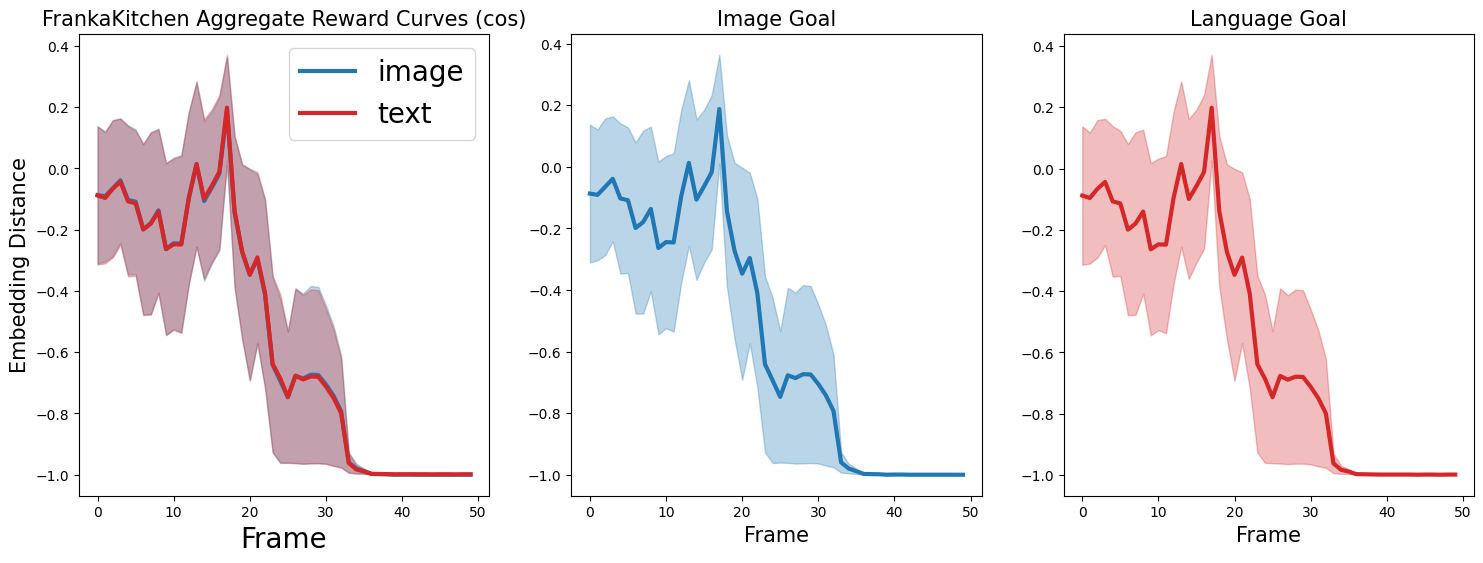}}
\vspace{-0.2cm}
\caption{\textbf{LIV fine-tuning (FT) improves both temporal coherence and semantic alignment of the multi-modal cost curves;} in contrast, CLIP fine-tuning over-aggressively aligns the goal frame-text pair that damages representations of earlier frames.}
\label{figure:fine-tuning-comparison}

\end{figure}

\paragraph{Qualitative Analysis.} 
To better understand LIV fine-tuning's empirical gains, we visually compare the LIV models fine-tuned by LIV and CLIP objectives by overlaying their multi-modal cost curves (Section~\ref{section:pre-training-reward}) over the in-domain demonstrations used for fine-tuning on FrankaKitchen. We use one demonstration from each distinct task and average over all curves to produce Figure~\ref{figure:fine-tuning-comparison}; individual task reward curves as well as the plots for TCN+CLIP are included in Appendix~\ref{appendix:frankakitchen-reward-curves}. As shown, the cost curves for the pre-trained LIV are far apart, illustrative of the real-to-sim domain gap that handicaps in-domain language grounding. LIV (Fig.~\ref{fig:liv-fine-tuning}) naturally preserves a structured representation, in which the visual and text similarity curves are near-linear, monotonic, and converge to similar locations, suggesting that the representation has successfully constructed a latent value function that aligns goals in different modalities and preserve temporal coherence due to the recursive nature of value functions. In contrast, CLIP fine-tuning (Fig.~\ref{fig:clip-fine-tuning}), as intended, maps the goal frame and the goal text to a near identical point in the representation space as the two curves almost perfectly overlap. However, the CLIP similarity scores of the intermediate frames exhibit uneven trends and high variance, indicating that the representation lacks temporal coherence possibly due to over-prioritizing semantic alignment. 
This temporal consistency is crucial for effective representation as it automatically prevents incorrect observation aliasing and preserves feature scale across time for effective policy learning~\citep{ma2022vip,nair2022r3m}. Yet, we have already shown that temporal consistency alone is not sufficient, evident from VIP-I's poor fine-tuning performance on FrankaKitchen due to overfitting. As such, LIV's unique effectivenss can be attributed to its principled combination of VIP and CLIP that enables the two objectives to regularize each other and work together in learning a structured, multi-modal latent value function.
On our project website, we also animate these reward curves generated using successive fine-tuning checkpoints to qualitatively capture the learning dynamics of LIV fine-tuning.

\begin{table}[t]
\caption{\textbf{Planning with Learned Reward:} LIV-EPIC is both the strongest zero-shot and adapted reward model.}
\centering 
\resizebox{1.0\columnwidth}{!}{
\begin{tabular}{llrrrrrrrrrrr}
\toprule
Model &  FrankaKitchen & MetaWorld  \\
\midrule
LIV (Pre-Trained) & 1.3 $\pm$ {\scriptsize 0.8} & 29.7 $\pm$ {\scriptsize 4.7} \\
LIV (LIV Fine-Tuned) & \textbf{20.0} $\pm$ {\scriptsize 4.5} & \textbf{55.2} $\pm$ {\scriptsize 5.5} \\
\midrule 
CLIP & 0 $\pm$ {\scriptsize 0.0} & 18.2 $\pm$ {\scriptsize 4.4} \\ 
CLIP (LIV Fine-Tuned) & 15.2 $\pm$ {\scriptsize 4.6} & 45.3 $\pm$ {\scriptsize 2.5}\\
CLIP (CLIP Fine-Tuned) & 3.2 $\pm$ {\scriptsize 0.9} & 30.7 $\pm$ {\scriptsize 3.3} \\
\midrule 
LOREL & 9.6 $\pm$ {\scriptsize 3.0}& 47.9 $\pm$ {\scriptsize 3.2} \\ 
LOREL (R3M Initialized) & 16.8 $\pm$ {\scriptsize 3.8} & 47.5 $\pm$ {\scriptsize 12.7}  \\
\midrule 
R3M (Pre-Trained) & $8.8^*$ $\pm$ {\scriptsize 2.7} & 18.3 $\pm$ {\scriptsize 7.7} \\ 
R3M (R3M Fine-Tuned) & 16.1 $\pm$ {\scriptsize 4.2} & 43.9 $\pm$ {\scriptsize 3.2} \\ 
\bottomrule
\end{tabular}}
\label{table:reward-learning-results}
\end{table}

\subsection{LIV Reward-Based Behavior Synthesis}
\label{section:reward-learning}

Finally, we demonstrate how to use LIV's dense goal-conditioned reward generation capability to directly acquire new language conditioned skills, in particular, via language-reward model predictive control.

\paragraph{Baselines.} 
We compare to LOREL~\citep{nair2022learning}, a state-of-art language-conditioned reward learning method that learns a classifier $f_\theta(o_0, o_t, l)$ for whether the progression from $o_0$ to $o_t$ completes task description $l$. In addition, we compare to R3M~\citep{nair2022r3m}, which incorporates a similar language-progression score function trained via contrastive learning. As the original LOREL does not leverage pre-trained visual representations, we also consider a variant of LOREL initialized with R3M model weights to improve its performance. Similarly, to circumvent the out-of-domain language grounding problem for pre-trained R3M, we consider a variant where we fine-tune the pre-trained R3M using the R3M objective on the same in-domain data used for LIV fine-tuning.  

\paragraph{Evaluations.}
We evaluate all reward models in a model-based planning setup, in which a trajectory optimizer synthesizes a sequence of actions to be executed in the true environment based on scores from the utilized reward function. For all LIV models (pre-trained and fine-tuned), we use the potential-based reward (Eq.~\ref{eq:vip-reward}) as in~\citet{ma2022vip} using encoded language goal: 
On MetaWorld, we use the identical experimental setup as in~\citet{nair2022learning}, whereas on FrankaKitchen, we closely follow the experimental protocol of~\citet{ma2022vip}. See Appendix~\ref{appendix:reward-learning} for more details on our model-based planning experiments.

\paragraph{Results.} The aggregated success rate over all test instances are reported in Table~\ref{table:reward-learning-results}. LIV fine-tuning significantly improves the success rate over the base pre-trained LIV and CLIP models, and the fine-tuned LIV achieves the best performance overall across both benchmarks; LIV fine-tuning's performance also steadily improves as the quality of the base model improves.  LOREL and R3M models both perform adequately with the respective modifications we have introduced, but they still trail behind LIV; in Appendix~\ref{appendix:reward-learning-additional-results}, we present additional analysis on these results. In conclusion, we have shown that LIV's implicit value learning paradigm gracefully combines both reward and representation learning in one unified objective and results in a flexible combined model that is highly effective across all evaluation settings.

\section{Conclusion}
\label{sec:conclusion}

We have presented the Language-Image Value Learning (LIV) algorithm. LIV is at once the first pre-training objective for control-oriented vision-language representations, a fine-tuning objective for domain-specific language grounding, and a language-conditioned task reward function. Trained on large generic human video datasets and fine-tuned on small robotics datasets, LIV outperforms state-of-the-art approaches in each of three distinct evaluation settings and successfully operates on real-world robotic tasks.

\section*{Acknowledgements}\label{sec:acks}
This work was supported in part by ONR grant number N00014-22-1-2677 and gift funding from
NEC Laboratories.

\section*{Author Contributions}
YJM conceived the project idea, implemented models and experiments, and wrote the paper. WL and VS assisted on the real-world experiments. VK, AZ, OB, and DJ provided feedback on the project and edited the writing.

 \clearpage 
 \newpage 
 
\bibliography{references}
\bibliographystyle{icml2023}
\newpage
\appendix 
\onecolumn
\label{sec:appendix}

\section{Proof of Proposition~\ref{proposition:vip-to-clip}}
\label{appendix:proof}
In this section, we provide a full proof of Proposition~\ref{proposition:vip-to-clip} in the main text. For ease of reading, we begin by reproducing the proposition. 

\begin{proposition*}
Let the video distribution consist of solely degenerate videos of repeated frames that align with the text annotation, $D := \{v:=((g,g;l))\}$. Then, the VIP-L objective is equivalent to the InfoNCE objective up to a constant: 
\begin{equation}
\mathcal{L}_{\text{VIP-L}}(\phi, \psi) = \BE_{p(g,l)}\left[-\log \frac{ e^{(1-\gamma) \mathcal{S}(\phi(g);\psi(l))}}{\BE_{D(g')}\left[e^{(1-\gamma) \mathcal{S}(\phi(g');\psi(l))}\right]} \right]+1,
\end{equation}
where $p(g,l)$ is the distribution of goal frame and text pair.
\end{proposition*}

\begin{proof}
We begin with the VIP-L objective:
\begin{equation}
\label{eq:initial-step}
\BE_{p(l)}[(1-\gamma) \BE_{\mu_0(o;l)}\left[-\mathcal{S}(\phi(o);\psi(l))\right] + \log \BE_{(o,o';l) \sim D}\left[\mathrm{exp}\left(\mathcal{S}(\phi(o);\psi(l)) + 1 - \gamma \mathcal{S}(\phi(o');\psi(l)) \right)\right]]
\end{equation}

We can massage this expression as follow:
\begin{equation}
\label{eq:intermediate-step}
\BE_{p(l)}[\BE_{\mu_0(o;l)}\left[ - (1-\gamma) \mathcal{S}(\phi(o);\psi(l))\right] + \log \BE_{(o,o';l) \sim D}\left[\mathrm{exp}\left(1 + (1- \gamma) \mathcal{S}(\phi(o);\psi(l)) \right)\right]],
\end{equation}
assuming $o=o'$ in the log-sum-exp term.

Now, the joint distribution of language and initial-frame $p(l)\mu_0(o;l)$ reduces to the marginal distribution of goal-frame and text distribution $p(g,l)$ when the videos are just concatenations of the goal frames. Similarly, 
The language-conditioned distribution of successive intermediate frames $D(o,o';l)$ reduces to the marginal distribution of goal frames $D(g')$ in the dataset. Plugging these substitution back into Equation~\eqref{eq:intermediate-step} gives

\begin{equation}
\label{eq:l-vip-simplification}
\begin{split}
 & \BE_{p(g,l)}\left[-\log \frac{ e^{(1-\gamma) \mathcal{S}(\phi(g);\psi(l))}}{\BE_{D(g')}\left[\mathrm{exp}\left(1 + (1- \gamma) \mathcal{S}(\phi(g');\psi(l))\right)\right]}\right] \\ 
 =  & \BE_{p(g,l)}\left[-\log \frac{ e^{(1-\gamma) \mathcal{S}(\phi(g);\psi(l))}}{\BE_{D(g')}\left[e \cdot \mathrm{exp}\left((1- \gamma) \mathcal{S}(\phi(g');\psi(l))\right)\right]}\right] \\ 
  =  & \BE_{p(g,l)}\left[-\log \frac{ e^{(1-\gamma) \mathcal{S}(\phi(g);\psi(l))}}{\BE_{D(g')}\left[\mathrm{exp}\left((1- \gamma) \mathcal{S}(\phi(g');\psi(l))\right)\right]}\right] + 1 \\ 
= &  \BE_{p(g,l)}\left[-\log \frac{ e^{(1-\gamma) \mathcal{S}(\phi(g);\psi(l))}}{\BE_{D(g')}\left[e^{(1-\gamma) \mathcal{S}(\phi(g');\psi(l))}\right]} \right] + 1
    \end{split}
\end{equation}
\end{proof}

\section{LIV Model Details}
\label{appendix:model-details}
We implement LIV using the open-sourced CLIP architecture\footnote{\href{https://github.com/openai/CLIP}{https://github.com/openai/CLIP}} without modifications; we use the modified ResNet50~\citep{he2016deep} from CLIP for the vision encoder, and the CLIP Transformer~\citep{vaswani2017attention, radford2019language} architecture for the language encoder. The training hyperparameters used during the pre-training and fine-tuning stages are listed in Table~\ref{table:vip-hyperparameters}. During pre-training, we also incorporate the VIP-L objective, which we find to produce better pre-trained LIV models in our preliminary experiments; we hypothesize that adding the explicit language-based VIP loss is instrumental in shaping the representation with semantic structure early on. During the fine-tuning stage, the same set of fine-tuning hyperparameters is used for fine-tuning CLIP as well as the ablation fine-tuning methods presented in Section~\ref{section:fine-tuning}.

Since LIV uses $-1$ as the constant fixed reward for all observations, the range of valid state value is $[\frac{-1}{1-\gamma}, 0]$; however, cosine similarity, as used in CLIP, has range of $[-1,1]$. Thus, to be able to represent all possible values, we set $\mathcal{S}(\phi(\cdot), \psi(\cdot)) := \frac{1}{1-\gamma} \texttt{CosineSimilarity}(\phi(\cdot), \psi(\cdot))$. Coincidentally, with this choice of $\mathcal{S}$, the InfoNCE objective in LIV reduces to precisely the InfoNCE objective used in CLIP.

We pre-train LIV on EpicKitchen~\citep{damen2018scaling}. We use the \texttt{EPIC-KITCHENS-100} version of the data and only utilize the RGB frames and text annotations from the dataset; the default frame rate in the raw dataset is used. The pre-training takes place on a node of 8 NVIDIA V100 GPUs. 

\begin{table}[ht]
\centering
\caption{VIP Architecture \& Pre-Training Hyperparameters.}\label{table:vip-hyperparameters}
\begin{tabular}{clll}
\toprule
& & Pre-Training & Fine-Tuning \\
\midrule
& Model Initialization & CLIP & \{LIV-EPIC, CLIP, Random\} \\ 
 & Optimizer & Adam~\citep{kingma2014adam} & Adam\\
               & Gradient Steps & 200000 & 10000\\ 
              & Batch Size      & 512 & 64\\
              & Learning Rate & 0.00001 & 0.00001 \\ 
              & Weight Recay & 0.001 & 0.001 \\
              & Discount Factor $\gamma$     & 0.98 & \{0.98, 0.96\} \\
            & VIP-L objective & Yes & No \\ 
\bottomrule
\end{tabular}
\end{table}

\section{Environment Details}
\label{appendix:environments} 

\begin{figure}[t!]
\centering
\subfigure[3rd-Person View]{\label{fig:realrobot-3rdperson-view}\includegraphics[width=0.30\textwidth]{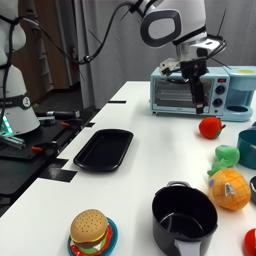}}
\subfigure[Wrist View]{\label{fig:realrobot-wrist-view}\includegraphics[width=0.30\textwidth]{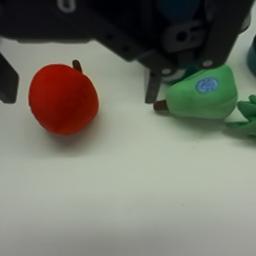}}
\caption{RealRobot Visual Inputs.}
\label{figure:realrobot-views}
\end{figure}

\textbf{MetaWorld.}
The MetaWorld environment consists of a tabletop scene with a Sawyer robot that can interact with 4 objects, including a drawer, faucet, and two mugs distinguished by color. The dataset is collected by running a random policy for 50000 episodes with episode length $20$; each episode is labeled with procedurally generated language descriptions that it achieves via computing pre-defined success criterion for each language-specified task. A single episode can solve many distinct tasks. In that case, the labeled description will be a concatenation of all atomic instructions that the episode has solved. The whole dataset contains 2311 unique descriptions, and the evaluation tests on $6$ atomic instructions: 
\texttt{close drawer}, \texttt{open drawer}, \texttt{turn faucet right}, \texttt{turn faucet left}, \texttt{move black mug right}, \texttt{move the white mug left}. 

\textbf{FrankaKitchen}
The FrankaKitchen environment consists of a kitchen scene with a Franka robot that can interact with a variety of common household kitchen objects. We use the same 5-task split that was evaluated in~\citet{nair2022r3m} for visual imitation learning; the tasks as well as their language commands are listed in Table~\ref{table:frankakitchen-tasks}. For each task, we include $50$ demonstrations, so the total size of the dataset is 250 episodes, where each episode is $50$ environment steps long. 

\textbf{RealRobot.} Our RealRobot environment consists of a toy kitchen tabletop scene with a Franka robot that can interact with 3 soft fruit toys (\texttt{apple, pineapple, pear}) and is tasked with placing the fruit objects into the correct container (\texttt{tray, black pot, green pot}) based on language input, creating 9 tasks in total. We use two cameras for visual inputs, one mounted on the robot gripper and one mounted on the side of the workspace table; see Figure~\ref{figure:realrobot-views} The dataset is collected via human teleoperation with 100 demonstrations per task.

\begin{table}[ht]
\centering
\caption{FrankaKitchen Task Mapping}\label{table:frankakitchen-tasks}
\begin{tabular}{clll}
\toprule
Environment ID & Language Task  \\
\midrule 
\texttt{kitchen\_micro\_open-v3} & \texttt{open microwave} \\
\texttt{kitchen\_sdoor\_open-v3} & \texttt{slide cabinet} \\ 
\texttt{kitchen\_ldoor\_open-v3} & \texttt{open left door} \\ 
\texttt{kitchen\_knob1\_on-v3} & \texttt{turn on stove} \\ 
\texttt{kitchen\_light\_on-v3} & \texttt{switch on light} \\ 
\bottomrule
\end{tabular}
\end{table}

\section{Language-Conditioned Imitation Learning with Pre-Trained Representations} 
\label{appendix:lcbc}
We present the LCBC imitation learning hyperparameters in Table~\ref{table:lcbc-hyperparameters}.
Because the dataset size in MetaWorld is significantly larger, we use a larger MLP architecture with bigger batch size. For each distinct evaluation task, we rollout for 50 episodes and record the success rate. 
\begin{table}[ht]
\centering
\caption{LCBC Hyperparameters.}\label{table:lcbc-hyperparameters}
\begin{tabular}{cllll}
\toprule
& & MetaWorld & FrankaKitchen & RealRobot\\
\midrule 
& MLP Architecture & [1024, 1024, 1024] & [256, 256] & [256,256]\\ 
& Non-Linear Activation & ReLU & ReLU & ReLU\\ 
\midrule
 & Optimizer & Adam & Adam & Adam\\
& Gradient Steps & 200000 & 200000 & 1M\\ 
& Batch Size & 4096 & 32 & 64 \\ 
& Learning Rate & 0.001 & 0.001 & 0.001\\ 
& Proprioception & No & Yes & No\\ 
& Augmentation & No & No & Yes \\ 
\bottomrule
\end{tabular}
\end{table}

\begin{figure}[t!]
\centering
\includegraphics[width=0.6\columnwidth]{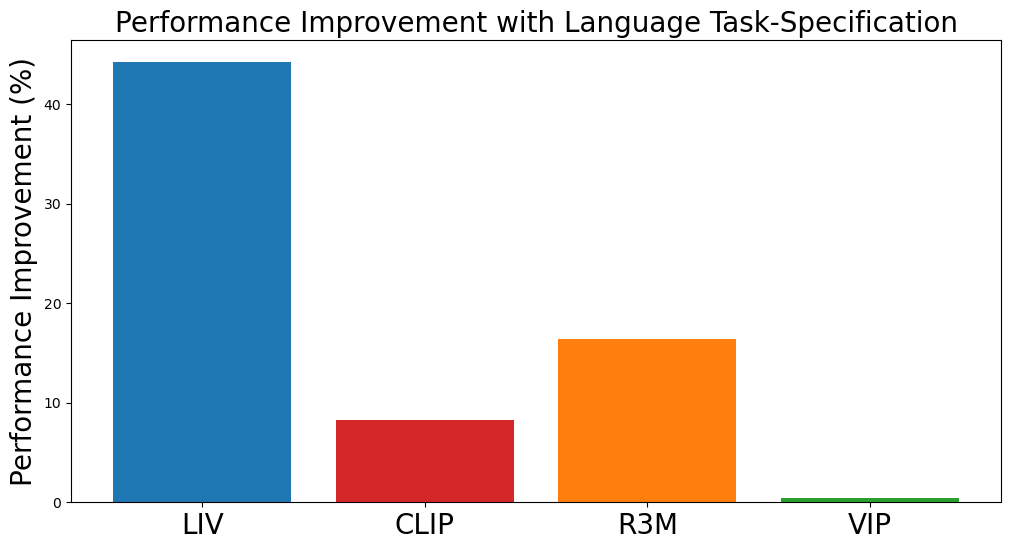}
\caption{\textbf{Comparison Between Language vs. One-Hot Task Encoding:} LIV benefits the most from using language task-specification, resulting in near $45\%$ gain in absolute success rates.}
\label{figure:languageonhot-aggregate}
\end{figure}

\begin{figure}[t!]
\centering
\includegraphics[width=0.6\columnwidth]{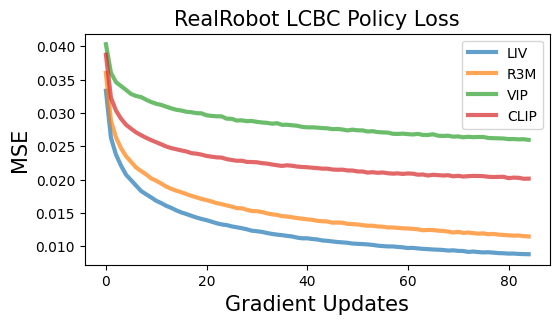}
\caption{LIV LCBC achieves lower training loss and consequently achieves higher success rate on our RealRobot multi-task suite.}
\label{figure:realrobot_loss}
\end{figure}

\subsection{Additional Results}
\label{appendix:lcbc-additional-results}
Given that our simulation environments require less sophisticated scene understanding from language, is language still helpful for policy learning? To probe this, we ablate each choice of pre-trained representations by replacing their language embeddings with one-hot task encoding to assess whether language provides contextual knowledge~\citep{sodhani2021multi} that facilitates policy learning and generalization. The full numeric results comparing the policy performance with and without language task encoding are presented in Table~\ref{table:pre-training-results}; the relative improvement in percentage from using language task encoding is also displayed in Figure~\ref{figure:languageonhot-aggregate}. as shown, LIV-EPIC consistently benefits from its jointly trained language representation. In particular, on both benchmarks, while LIV-EPIC and the strongest baselines (CLIP and VIP) all perform similarly with one-hot encoding, LIV-EPIC realizes much greater gain when language task-specification is used. In fact, using language task-specification \textit{hurts} all baselines on MetaWorld. We hypothesize that this is due to the fact that the MetaWorld dataset contains many episodes whose annotations are long descriptions that consist of concatenation of shorter atomic instructions; for example, \texttt{"close drawer turn faucet right push black mug right"} is a valid annotation that contains 3 atonomic instructions. Therefore, the language embeddings from pure language model (e.g, DistilBERT) or language model trained from image-text only datasets (e.g., CLIP) may fail to disambiguate these instructions, leading to incorrect task aliasing that hampers policy learning. In contrast, one-hot encoding treats every description as distinct and does not have this aliasing problem.  Together, these results highlight the challenges of adapting pure image-text representations and uni-modal visual representations to language-conditioned robotic control, thereby affirming LIV's unique effectiveness in vision-language pre-training for language-conditioned visuomotor control.

\begin{table}[t]
\caption{\textbf{Pre-Trained Representations for Language-Conditioned Imitation Learning:} LIV-EPIC makes most effective use of its language embedding and realizes greater gains when using language-based contextual representation.}
\centering \small
\begin{tabular}{llrrrrrrrrrrr}
\toprule
model & FrankaKitchen  & MetaWorld     \\
\midrule
\textbf{LIV-EPIC} & \textbf{29.3 }$\pm$ {\scriptsize 4.6} & \textbf{30.6} $\pm$ {\scriptsize 5.0}\\
LIV-EPIC (One-Hot) & 17.6 $\pm$ {\scriptsize 5.0} & 26.1 $\pm$ {\scriptsize 5.5} \\ 
\midrule 
CLIP & 22 $\pm$ {\scriptsize 3.5} & 19.4 $\pm$ {\scriptsize 1.3} \\
CLIP (One-Hot) & 14.8 $\pm$ {\scriptsize 0.7} & 28.6 $\pm$ {\scriptsize 1.3} \\ 
\midrule 
VIP (BERT) &  18.0 $\pm$ {\scriptsize 6.9} & 24.2 $\pm$ {\scriptsize 3.0}  \\ 
VIP (One-Hot) & 15.6 $\pm$ {\scriptsize 6.2} & 28.3 $\pm$ {\scriptsize 0.8} \\ 
\midrule 
R3M (BERT) & 18.7 $\pm$ {\scriptsize 11.0} & 12.7 $\pm$ {\scriptsize 3.9}  \\
R3M (One-Hot) & 11.5 $\pm$ {\scriptsize 1.9} & 18.1 $\pm$ {\scriptsize 5.5} \\
\bottomrule
\end{tabular}
\label{table:pre-training-results}
\end{table}

\subsection{Zero-Shot Long-Horizon Task Generalization}
\label{appendix:realrobot-longhorizon}
On RealRobot, we test whether the LIV policy can solve long-horizon, composite tasks that require solving the atomic tasks in some specified sequences. We allow the policy to spend a fixed number of 100 control steps to solve the current task; then, the robot will be reset to its initial position, and the current task will transition to the next one in the specified sequence, which again allocates the policy 100 steps before transitioning. We note that the RealRobot dataset contains only demonstrations for the short-horizon, atomic tasks, and the demonstrations are never collected in configurations where some fruits have already been placed into some containers. As such, solving more than one task strictly requires the policy to generalize to unseen tabletop configuration, as success for an earlier task will change the scene into a novel configuration for the later tasks.

Given this, we report the number of trials out of 10 in which the policy is able achieve \textit{partial success}, namely, solving at least 2 tasks out of the 3 tasks in the specified sequence. The comparison between LIV and R3M is shown in Table~\ref{table:zero-shot-long-horizon}. We see that LIV policy is capable of generalizing to composite tasks that require generalization in both visual and semantic levels. In some cases, LIV policy can solve all 3 tasks. On our project website, we provide videos of the policy rollouts.

\begin{table}[t]
\caption{\textbf{Partial Success Counts on Unseen Long-Horizon Composite Tasks.}}
\centering 
\begin{tabular}{llr}
\toprule
Task Sequence & LIV  & R3M     \\
\midrule
\{\texttt{Pineapple in Tray, Apple in Tray, Pear in Tray}\} & 5/10 & 0/10 \\ 
\{\texttt{Pineapple in Black Pot, Apple in Tray, Pear in Green Pot}\} & 3/10 & 1/10 \\ 
\{\texttt{Pineapple in Tray, Apple in Green Pot, Pear in Black Pot}\} & 5/10 & 1/10 \\ 
\bottomrule
\end{tabular}
\label{table:zero-shot-long-horizon}
\end{table}

\section{Fine-Tuning}
\label{appendix:fine-tuning} 

\subsection{Fine-Tuning on RealRobot}
\label{appendix:fine-tuning-realrobot}
In our real world environment, since we use a 3rd-person and a wrist camera for policy learning, the LIV fine-tuning procedure differs slightly from our simulation experiments. In particular, since the wrist view provides a local view of the table that is significantly out-of-distribution with respect to the pre-training dataset, we elect to fine-tune using just the 3rd-person camera view. Then, we use the pre-trained LIV to process visual observations from the wrist camera view and the fine-tuned LIV to process visual observation from the 3rd-person camera view. The 3rd-person view embedding vectors are further passed through a learnable shallow MLP adapter as the feature scale of pre-trained and fine-tuned LIV varies. Then, the embeddings are conatenated with the language embedding from the fine-tuned LIV, and the final concatenated embedding vector is the input to the MLP policy network. For our real-world experiment, we also employ trajectory level frame random-cropping as data augmentation during policy learning, which we find to improve all methods. 

\subsection{Additional Results}
\label{appendix:fine-tuning-additional-results}
In this section, we present additional experiments probing LIV's fine-tuning capability in simulation. 

\textbf{Can LIV effectively fine-tune base models of varying quality?} We first study how LIV and its ablations (CLIP and VIP-I) fare with base models of varying quality. To this end, in addition to the base CLIP model we consider in the main text, we also include pre-trained LIV as well as a Random model (i.e., randomly initialized weights on the same CLIP architecture) as base models to be fine-tuned. The full results are presented in Table~\ref{table:fine-tuning}. We see that LIV fine-tuning is effective for all three model initializations, whereas the baseline ablations deliver mixed results. In particular, CLIP fine-tuning degrades performance in all cases except on the Random model in MetaWorld. This difference in dataset sizes also explains why VIP-I fine-tuning is reasonable on MetaWorld but very poor on FrankaKitchen, consistent with the findings in~\citet{ma2022vip}. As such, we have demonstrated that both terms in the LIV are indispensable for effective fine-tuning, and LIV is uniquely effective at fine-tuning vision-language models under varying pre-training objectives, pre-trained model qualities, and fine-tuning dataset sizes. The final LIV fine-tuned models perform better when they started from better pre-trained models, so that the best combined system simply uses the LIV objective for both phases, pre-training as well as fine-tuning.

\textbf{How does in-domain dataset size affect LIV fine-tuning performance?} 
Now, we probe whether LIV fine-tuning can work with even smaller in-domain robot datasets, representative of the few-shot settings that several recent works have studied~\citep{nair2022r3m, ma2022vip}. On the FrankaKitchen environment, instead of 50 trajectories per task, we repeat the same fine-tuning+LCBC experiment with just 10 and 25 trajectories per task. For these experiments, we include TCN+CLIP as a baseline, as it is the only baseline that actually improve policy performance with 50 trajectories per task for fine-tuning. The results are presented in Table~\ref{table:fine-tuning-fewshots}. As shown, across different dataset sizes, LIV fine-tuning consistently provides large gains, whereas TCN-CLIP is able to realize much smaller gains. Notably, in the most challenging setting of few-shot, 10 demos per task, LIV fine-tuning is able to match CLIP (no fine-tuning) with 50 demonstrations per task and improves the base model performance by more than 200\%. These results demonstrate that LIV is fully capable of effective fine-tuning even in a very low data regime, showcasing its versatility and sample-efficiency.

\textbf{The sensitivity of TCN+CLIP to objective weighting $\alpha$.} We present the full hyperparameter search over $\alpha$ for TCN+CLIP on both FrankaKitchen and MetaWorld and report the results in Table~\ref{table:fine-tuning-tcnclip}. As shown, while higher $\alpha$ values work better on FrankaKitchen, only the lowest $\alpha$ value of $0.1$ results in a TCN-CLIP fine-tuned model that did not diverge during MetaWorld training; however, this converged TCN-CLIP-$\alpha$0.1 model yields significantly lower downstream policy performance than LIV fine-tuning. Note that in practice, hyperparameter tuning for offline visual imitation learning/RL is fairly difficult because of the high computational footprint that limits the amount of hyperparameter tuning and, more fundamentally, the offline setting does not permit online rollouts for evaluation. As such, LIV’s lack of dependency on tuning the balance between its SSL and CLIP objective is a significant advantage over these baselines in addition to its already superior performance.

\begin{table}
\caption{\textbf{Fine-Tuning Vision-Language Representations:} LIV consistently improves the performance of pre-trained vision-language models regardless of their initial capabilities, the sizes and the qualities of the in-domain fine-tuning datasets.}
\label{table:fine-tuning}
\centering 
\resizebox{0.7\textwidth}{!}{
\begin{tabular}{llrrrrrrrrrrr}
\toprule
MetaWorld & & & & \\
\midrule 
Model/Method & Pre-Trained & \textbf{LIV} & CLIP & VIP-I   \\
\midrule 
Random & 20.6 $\pm$ {\scriptsize 1.0}& 27.8 $\pm$ {\scriptsize 4.1} & 30.8 $\pm$ {\scriptsize 2.2}& 30.6 $\pm$ {\scriptsize 3.5} & \\ 
CLIP & 19.4 $\pm$ {\scriptsize 1.3} &  33.9 $\pm$ {\scriptsize 7.5} &  16.4 $\pm$ {\scriptsize 4.3} & 30.0 $\pm$ {\scriptsize 2.2} &  \\
LIV-EPIC & 30.6 $\pm$ {\scriptsize 5.0} & \textbf{35.8} $\pm$ {\scriptsize 1.4} & 21.4 $\pm$ {\scriptsize 5.7}& 20.3 $\pm$ {\scriptsize 3.4}& \\
\midrule 
FrankaKitchen & & & & \\ 
\midrule 
Model/Method  & Pre-Trained & \textbf{LIV} & CLIP & VIP-I   \\
\midrule 
Random & 17.7 $\pm$ {\scriptsize 3.9} & 19.2 $\pm$ {\scriptsize 3.8} & 17.1 $\pm$ {\scriptsize 2.2}& 3.2 $\pm$ {\scriptsize 0.7} & \\ 
CLIP & 22 $\pm$ {\scriptsize 3.5} &  26.8 $\pm$ {\scriptsize 4.9} & 14.0 $\pm$ {\scriptsize 6.8} & 14.8 $\pm$ {\scriptsize 1.3}   \\
LIV-EPIC & 29.3 $\pm$ {\scriptsize 4.6} & \textbf{32.3} $\pm$ {\scriptsize 5.8} & 15.1 $\pm$ {\scriptsize 4.3} & 17.3 $\pm$ {\scriptsize 6.6}& \\ 
\bottomrule
\end{tabular}}
\end{table}

\begin{table}
\caption{\textbf{Few-Shot Fine-Tuning Vision-Language Representations:} LIV fine-tuning can consistently improve over base model even with handful of demonstrations per task.}
\label{table:fine-tuning-fewshots}
\centering 
\resizebox{0.7\textwidth}{!}{
\begin{tabular}{llrr}
\toprule
Number of Demos per Task & CLIP & CLIP (LIV fine-tuned) & CLIP (TCN+CLIP fine-tuned)   \\
\midrule 
10 & 7.3 $\pm$ {\scriptsize 1.2} & \textbf{22.0} $\pm$ {\scriptsize 6.0} & 13.3 $\pm$ {\scriptsize 2.3} \\ 
20 & 12.7 $\pm$ {\scriptsize 1.5} & \textbf{30.7} $\pm$ {\scriptsize 1.7} & 23.3 $\pm$ {\scriptsize 1.2} \\ 
50 & 22.0 $\pm$ {\scriptsize 3.5} & \textbf{33.0} $\pm$ {\scriptsize 1.4} & 25.3 $\pm$ {\scriptsize 4.1} \\ 
\bottomrule
\end{tabular}}
\end{table}

\begin{table}
\caption{TCN+CLIP is sensitive to the relative weighting between the two components in the combined objective.}
\label{table:fine-tuning-tcnclip}
\centering 
\resizebox{0.7\textwidth}{!}{
\begin{tabular}{llrrr|r}
\toprule
& $\alpha=0.1$ & $\alpha=0.3$ & $\alpha=0.5$  & $\alpha=0.8$ & LIV fine-tuning \\
\midrule 
FrankaKitchen & 11.3 $\pm$ {\scriptsize 4.2} & 24.0 $\pm$ {\scriptsize 1.0} & 25.3 $\pm$ {\scriptsize 4.1} & 24.7 $\pm$ {\scriptsize 4.2} & 33.0 $\pm$ {\scriptsize 1.4}\\ 
MetaWorld & 14.4 $\pm$ {\scriptsize 2.7} & Diverged & Diverged & Diverged & 35.8 $\pm$ {\scriptsize 1.4}\\ 
\bottomrule
\end{tabular}}
\end{table}

\section{Reward Learning}
\label{appendix:reward-learning}
We describe our model-based planning experimental details. On MetaWorld, we use a cross-entropy Method (CEM)~\citep{rubinstein2004cross} planner to propose action sequences and employ the open-sourced SV2P~\citep{babaeizadeh2017stochastic} visual dynamics model trained on the demonstration data to rollout the action sequences for optimization. On FrankaKitchen, as in~\citet{ma2022vip}, we use the ground-truth environment dynamics to for action rollouts and employ a model-path predictive integral (MPPI)~\citep{williams2017model} planner. On FrankaKitchen, due to the exploration challenge, we also warmstart the action search with a fixed open-loop sequence that brings the robot end-effector to the vicinity of the task object but does not perform the full commanded task.

\subsection{Hyperparameters}

On MetaWorld, we use the open-sourced implementation of Cross-Entropy Method (CEM) on this environment released by~\citep{nair2022learning}. On FrankaKitchen, we follow the practice of~\citet{ma2022vip} and use a publicly available implementation of MPPI\footnote{\href{https://github.com/aravindr93/trajopt/blob/master/trajopt/algos/mppi.py}{https://github.com/aravindr93/trajopt/blob/master/trajopt/algos/mppi.py}} with the default hyperparameters. 

\begin{table}[ht]
\centering
\caption{Model-Based Planning Hyperparameters.}\label{table:planning-hyperparameters}
\begin{tabular}{clll}
\toprule
& & MetaWorld & FrankaKitchen \\
\midrule 
& Planner & CEM & MPPI~\citep{williams2017model} \\ 
& Planning Horizon & 20 & 50 \\ 
& \# Proposed Action Sequences & 200 & 128 \\ 
& Optimization Iteration & 1 & 1 \\ 
& Dynamics Model & SV2P trained on in-domain dataset & Ground truth simulation \\ 
\bottomrule
\end{tabular}
\end{table}

\subsection{Additional Results \& Analysis} 
\label{appendix:reward-learning-additional-results}

\textbf{Additional analysis of Table~\ref{table:reward-learning-results}.}
 These results illustrate the orthogonal, if not competing, nature of reward and representation capability of a vision-language model. While CLIP (CLIP Fine-Tuned) exhibits improved reward learning performance over pre-trained CLIP, in Section~\ref{section:fine-tuning}, we have shown that CLIP fine-tuning leads to far inferior representation backbone for policy learning. We believe this is because CLIP fine-tuning aligns the last frames with text goals, and the model-based planners we use evaluate action sequences based on only the reward of the last observation. In contrast, in imitation learning, the representation needs to be well-behaved for all intermediate observations, which CLIP fine-tuning impairs, as shown in Figure~\ref{figure:fine-tuning-comparison}. LOREL is a reward learning algorithm, yet it is prone to overfitting when trained from scratch on small in-domain data (i.e., FrankaKitchen) and is most performant when initialized with a pre-trained representation. Finally, though R3M training involves learning a language-reward predictor, this predictor is trained only in service of the core visual representation training. We find that this predictor is inferior to even purely in-domain trained LOREL on MetaWorld.

\textbf{How does increasing planning budget affect model performance?}
To further assess the capability of the various learned reward models, we repeat the model-based planning experiment on MetaWorld by increasing the CEM optimization iteration from $1$ to $3$. The results are shown in Table~\ref{table:reward-learning-additional-results}. We see that almost all models that are trained or fine-tuned on the in-domain data see performance increase with the fine-tuned LIV-EPIC standing as the best model. However, the pre-trained models (LIV-EPIC, CLIP, R3M), with the exception of LIV-EPIC, see performance degradation, suggesting that their reward models are in fact exploited by the stronger optimizer. Finally, we observe that LIV with 1 CEM iteration already performs as well as LOREL with 3 CEM iterations, suggesting that LOREL is more prone to ``false nagatives'', i.e. assigning low scores to good trajectories. These results highlight both LIV's ability for zero-shot and fine-tuning reward model.

\begin{table*}[t]
\caption{LIV models consistently improve with increased planning budget; in contrast, baselines report mixed results.}
\label{table:reward-learning-additional-results}
\centering 
\begin{tabular}{llrrrrrrrrrrr}
\toprule
Model  & MetaWorld (CEM Iterations=1) & MetaWorld (CEM Iterations=3)   \\
\midrule
LIV-EPIC & 29.7 & 34 \\
LIV-EPIC (LIV Fine-Tuned) & \textbf{55.2} & \textbf{57.8} \\
\midrule 
CLIP & 18.2 & 14.7 \\ 
CLIP (LIV Fine-Tuned) &  45.3& 44.4 \\
CLIP (CLIP Fine-Tuned) & 30.7 & 34.4 \\
\midrule 
LOREL & 47.9 & 55.4 \\ 
LOREL (R3M Initialized)  & 47.5 & 50.6 \\ 
R3M & 18.3 & 18.1 \\ 
R3M (R3M Fine-Tuned) & 43.9& 50.8 \\ 
\bottomrule
\end{tabular}
\end{table*}

\textbf{Why does R3M work well zero-shot on FrankaKitchen?}
Interestingly, we find R3M to perform well zero-shot on FrankaKitchen (Table~\ref{table:reward-learning-results}, achieving $\approx 9\%$ success rate without any in-domain fine-tuning. 
Upon investigating this outcome however, we find that this result is an artifact of the specific way in which R3M was trained. 
In particular, R3M's pre-trained reward predictor has a bias for actions that induce visual change in the environment because it was pre-trained to output higher scores for frames that are farther apart in time, which typically correlate with larger visual changes in the scene. To confirm this, we repeat the same experiment on FrankaKitchen but this time with \textit{random} language goals. The results are shown in Table~\ref{table:random-goals}. We see that R3M's performance remains surprisingly high, indicating that it does not depend at all on the language-based task specification.
In contrast, other models' performance catastrophically decline. This indicates that R3M's language grounding is limited and often confuses completion of specific tasks with any indiscriminate visual changes in the environment. This finding is further supported by R3M's poor performance on the MetaWorld environment, in which random actions are enough to move the objects and induce large visual changes, and task completion requires more directed action, driven by 
more sophisticated language understanding.
LIV-EPIC significantly outperforms R3M on MetaWorld and is the best zero-shot reward model overall on this benchmark. 

\begin{table}
\caption{Performance Comparison Between Correct and Random Language Goals.}
\label{table:random-goals}
\centering 
\begin{tabular}{llrrrrrrrrrrr}
\toprule
Model &  Correct Goal & Random Goal  \\
\midrule
LIV-EPIC & 1.3 & 1.0\\
LIV-EPIC (LIV Fine-Tuned) & \textbf{20.0} & 0.0 \\
\midrule 
LOREL & 9.6 & 0.0 \\ 
LOREL (R3M Initialized) & 16.8 & 0.0 \\
\midrule 
R3M & $8.8$ & 12.1 \\ 
R3M (R3M Fine-Tuned) & 16.1 & 0.0 \\ 
\bottomrule
\end{tabular}
\end{table}

\section{Representation Qualitative Results}
\label{appendix:qualitative}
In this section, we provide additional qualitative results on our pre-trained and fine-tuned models. For animated version of the qualitative reward curves, please visit our project website: \href{https://penn-pal-lab.github.io/LIV}{penn-pal-lab.github.io/LIV}

\subsection{EpicKitchen (Real)}
\label{appendix:qualitative-epickitchen}
We first visualize pre-trained LIV-EPIC on representative seen and unseen EpicKitchen videos by plotting the embedding curves with respect to the image (final frame of the video) and the text goal. In both seen and unseen splits, the three videos have annotations \texttt{open cabinet}, \texttt{open door}, and \texttt{open microwave}, respectively. The results are in Figure~\ref{figure:liv-epickitchen-qualitative} and \ref{figure:liv-epickitchen-qualitative-unseen}. For comparison purpose, we include the results for the CLIP model in Figure~\ref{figure:clip-epickitchen-qualitative} and \ref{figure:clip-epickitchen-qualitative-unseen}.

\subsection{HelloRobot (Real)}
\label{appendix:qualitative-hellorobot}

In Figure~\ref{figure:liv-hellorobot-success}, we present additional examples where pre-trained LIV is able to capture language-conditioned progress in unseen robot videos; in Figure~\ref{figure:clip-hellorobot-success}, we present the reward curves for CLIP on the same set of videos; as shown, CLIP's zero-shot language-reward is much more noisy. In Figure~\ref{figure:liv-hellorobot-long}, we conduct the same reward curve analysis on untrimmed videos in which the robot completes a sequence of opposite actions in the same video (e.g., \texttt{open the fridge} and \texttt{close the fridge}; note that these results are best viewed on our project website. Since the image goal and the language goal semantically refer to opposite actions, where image goal specifies the action accomplished in the last frame and the language goal specified the action accomplished in the middle of the video, we see that LIV's reward curves exhibit inverted trend across two modalities. This demonstrates that LIV has the ability to detect fine-grained, action-induced object state changes in videos.

Finally, in Figure~\ref{figure:liv-hellorobot-failure}, we also present several failure examples, where LIV language rewards fail to capture language-based progression. There are several reasons why these failures may occur, such as network capacity, and the distribution shift presented in these videos with respect to camera viewpoint, embodiment, and language commands. Given LIV's self-supervised nature, we are hopeful that LIV's zero-shot capability will only improve with more expressive network architecture and diverse datasets.

\subsection{RealRobot (Real)}
\label{appendix:realrobot-reward-curves}
In Figure~\ref{figure:liv-realrobot-qualitative}, \ref{figure:liv-livfinetuned-realrobot-qualitative}, we present the reward curves for LIV (pre-trained) and LIV (LIV fine-tuned). As shown, LIV (pre-trained) produces reasonable visual reward progress but suffers from domain gap that renders its language reward progress ineffective. LIV (LIV fine-tuned) remedies this issue and smoothens the representation in both modalities.

\subsection{FrankaKitchen (Sim)}
\label{appendix:frankakitchen-reward-curves}

In Figure~\ref{figure:liv-frankakitchen-qualitative}, \ref{figure:liv-livfinetune-frankakitchen-qualitative}, \ref{figure:liv-clipfinetune-frankakitchen-qualitative}, \ref{figure:liv-tcnclipfinetune-frankakitchen-qualitative}. we present the reward curves for LIV-EPIC, LIV-EPIC (LIV fine-tuned), LIV-EPIC (CLIP fine-tuned), and LIV-EPIC (TCN+CLIP fine-tuned) on the FrankaKitchen tasks, respectively. For each model, the first plot is the averaged reward curve for all 5 tasks, whereas the succeeding 5 plots are the task-specific reward curves. As shown, LIV-EPIC, without any in-domain fine-tuning, is able to competently capture visual progress but lacks language grounding to capture language goal progress. LIV fine-tuning captures fine-grained language-conditioned progression while simultaneously improving visual temporal alignment. CLIP fine-tuning over-aggressively aligns the representations of the last frame and the text goal and collapses intermediate representations. TCN+CLIP lacks temporal smoothness in the learned representation that is crucial for both vision-language representation for control (Section~\ref{section:pre-training} and language-conditioned reward-specification (Section~\ref{section:reward-learning}.

\clearpage 

\begin{figure} 
\centering
  \includegraphics[width=0.7\linewidth]{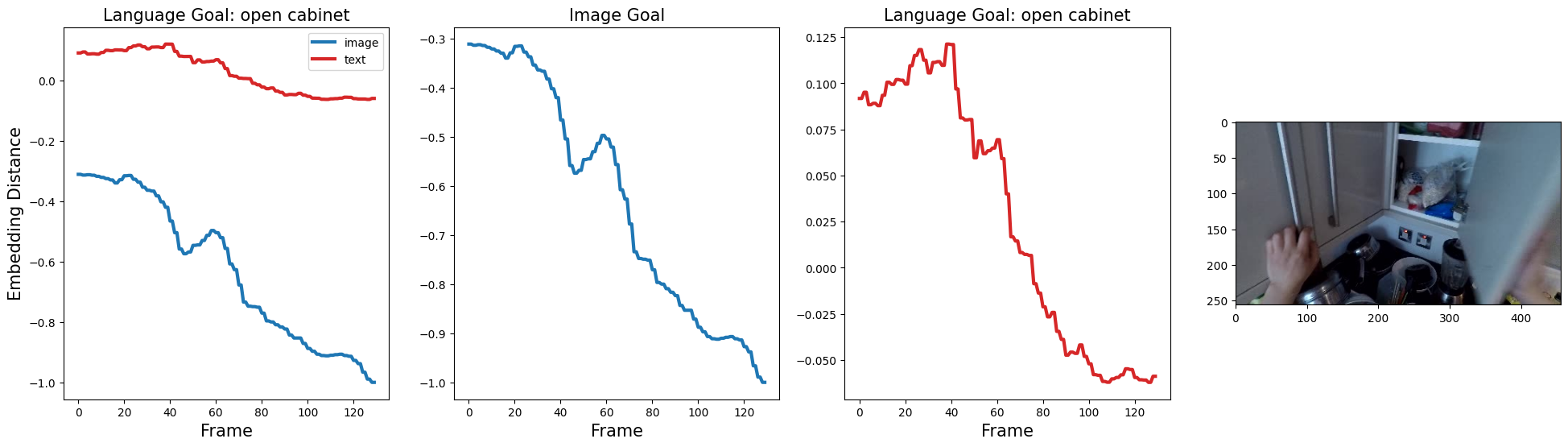}\hfill
  \includegraphics[width=0.7\linewidth]{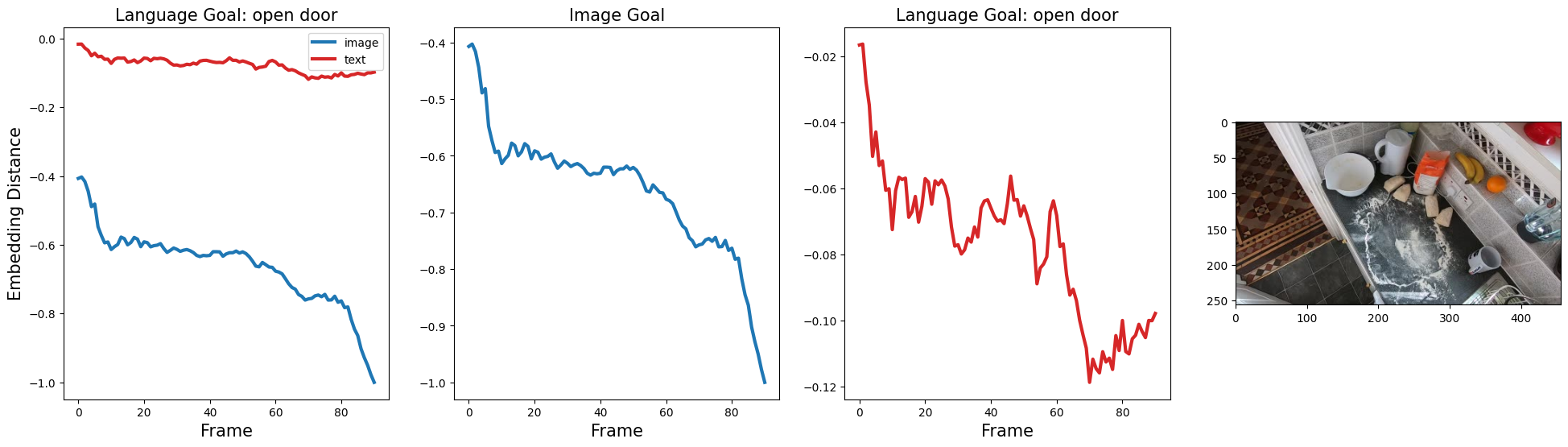}\hfill
  \includegraphics[width=0.7\linewidth]{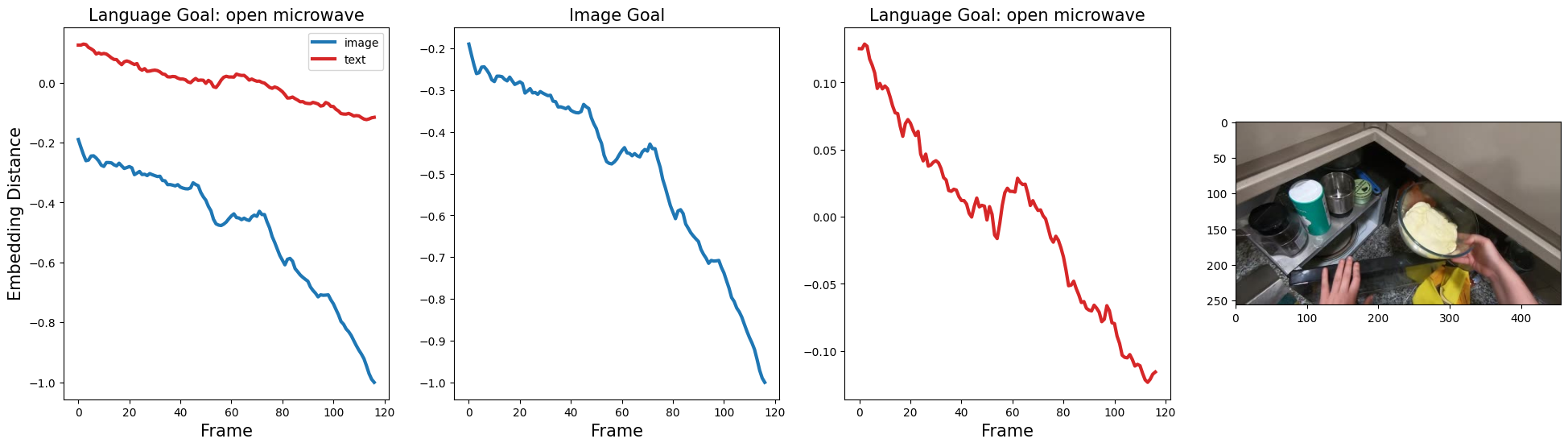}\hfill
  \caption{Pre-trained LIV-EPIC image and language goal reward curves on (seen) EpicKitchen videos.} 
\label{figure:liv-epickitchen-qualitative}
\end{figure}

\begin{figure} 
\centering
  \includegraphics[width=0.7\linewidth]{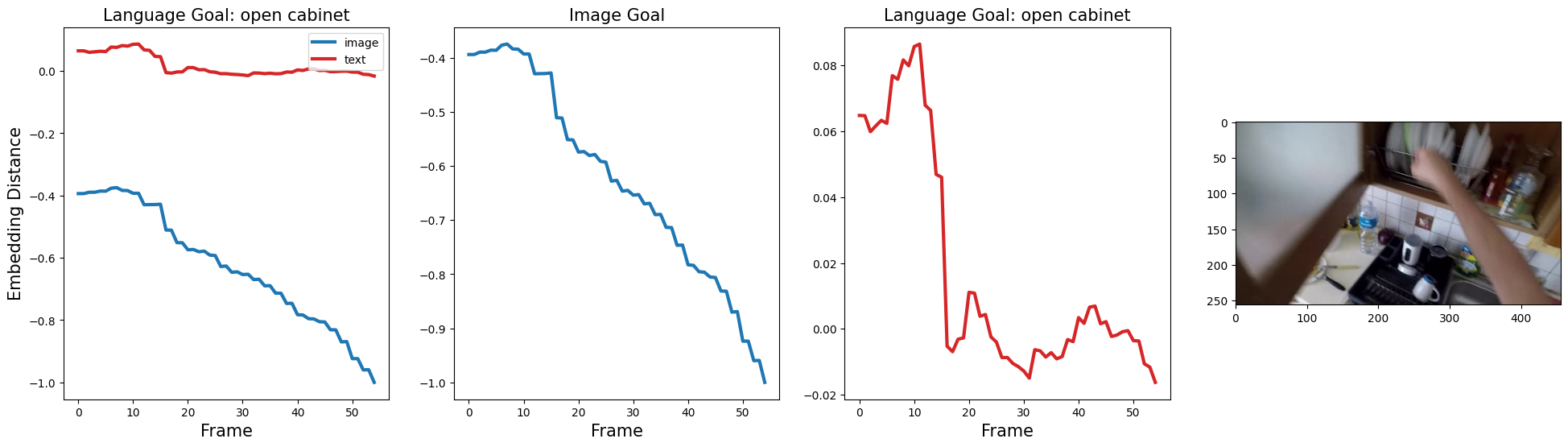}\hfill
  \includegraphics[width=0.7\linewidth]{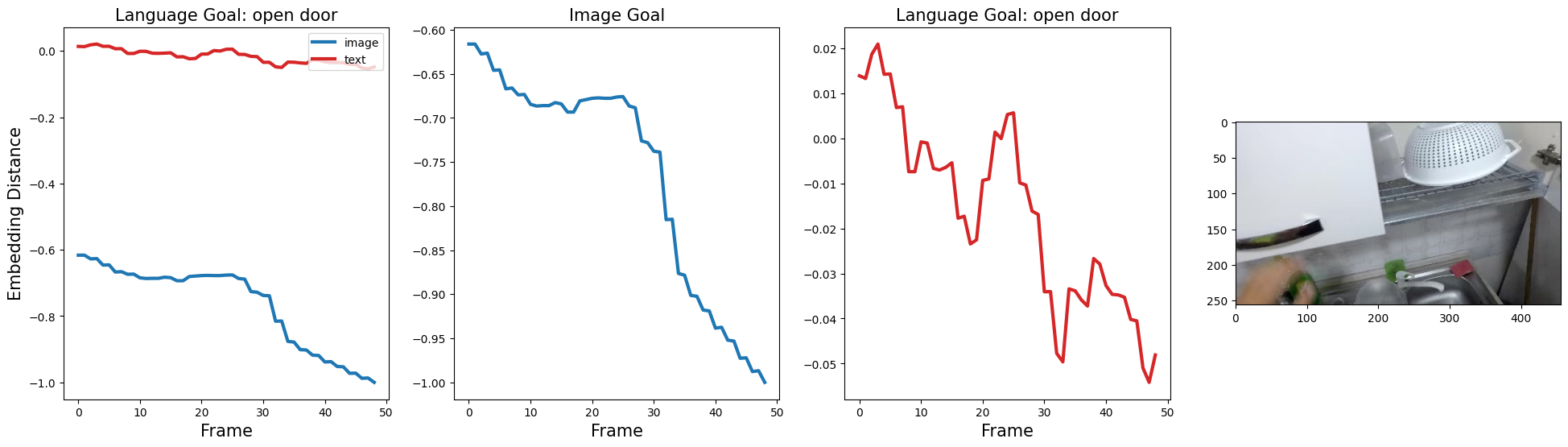}\hfill
  \includegraphics[width=0.7\linewidth]{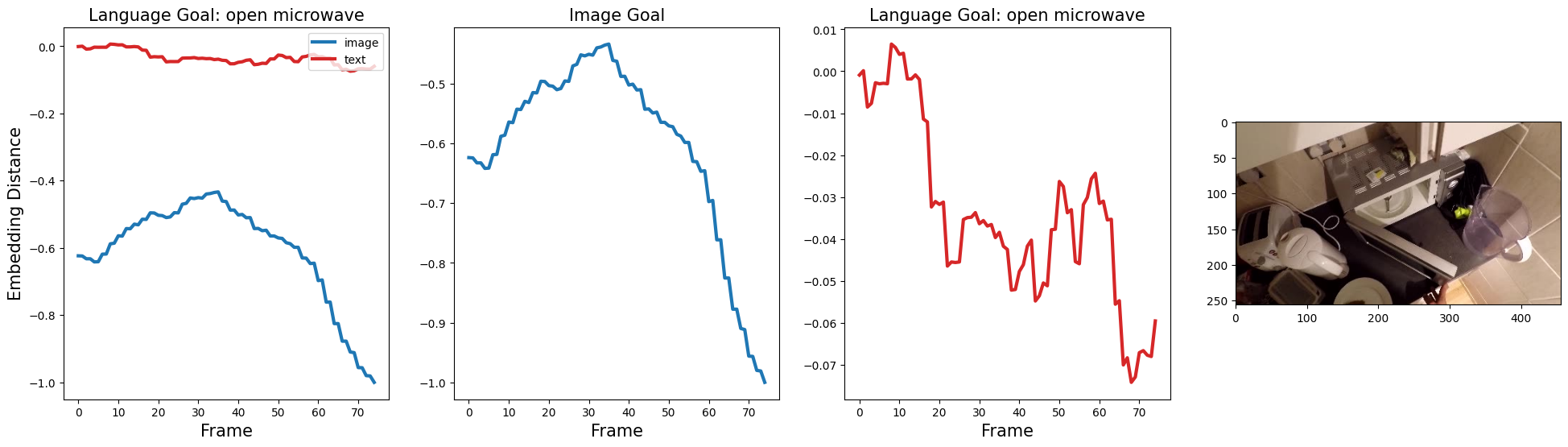}\hfill
  \caption{Pre-trained LIV-EPIC image and language goal reward curves on (unseen) EpicKitchen videos.} 
\label{figure:liv-epickitchen-qualitative-unseen}
\end{figure}

\begin{figure} 
\centering
  \includegraphics[width=0.7\linewidth]{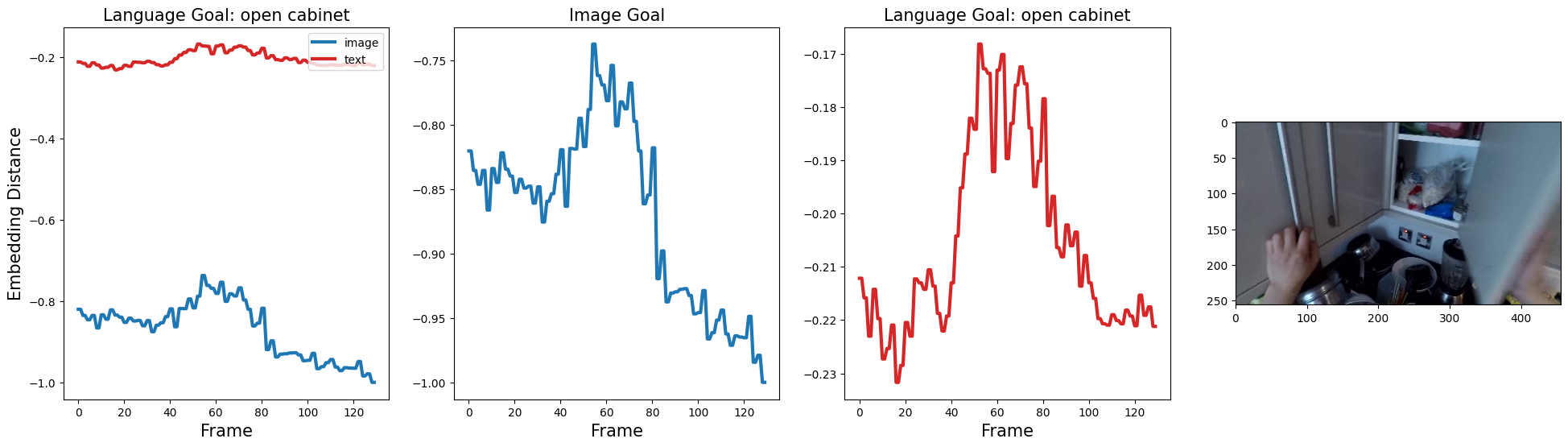}\hfill
  \includegraphics[width=0.7\linewidth]{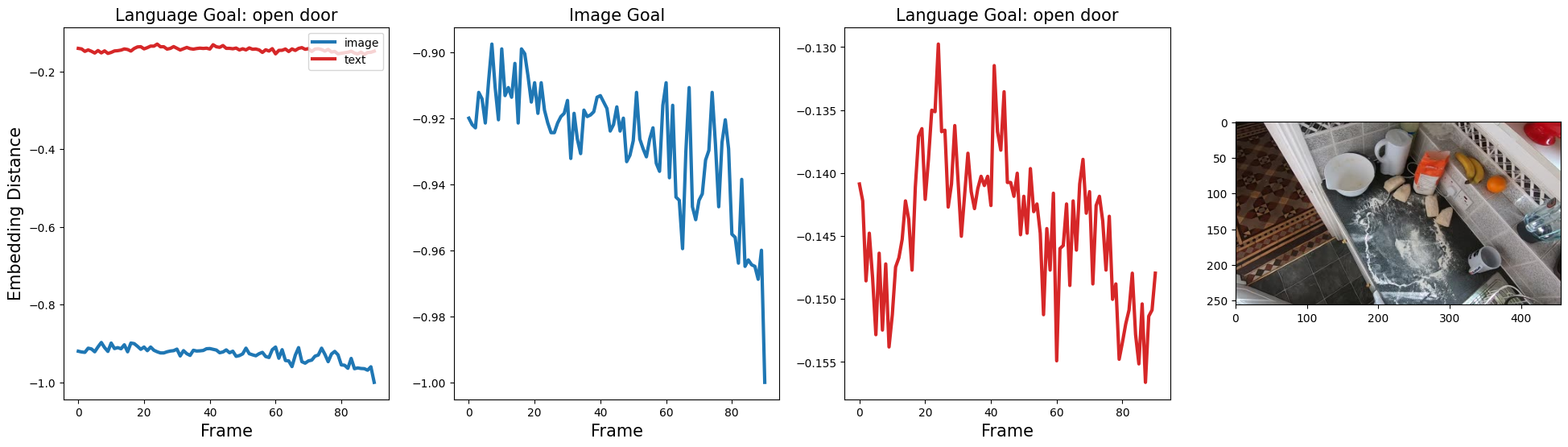}\hfill
  \includegraphics[width=0.7\linewidth]{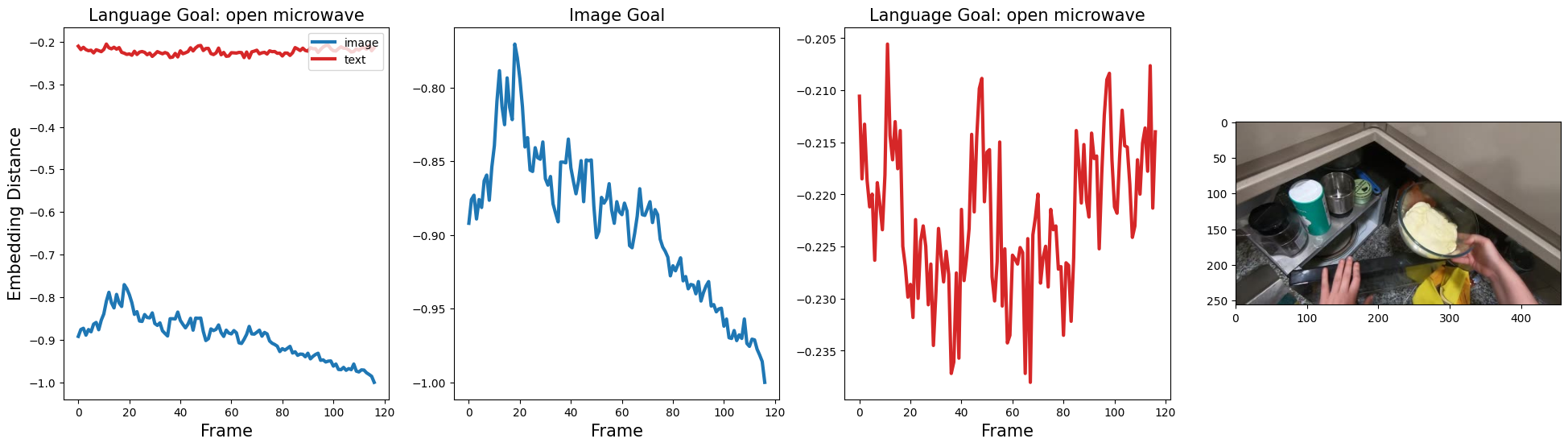}\hfill
  \caption{CLIP image and language goal reward curves on (seen) EpicKitchen (videos).} 
\label{figure:clip-epickitchen-qualitative}
\end{figure}

\begin{figure} 
\centering
  \includegraphics[width=0.7\linewidth]{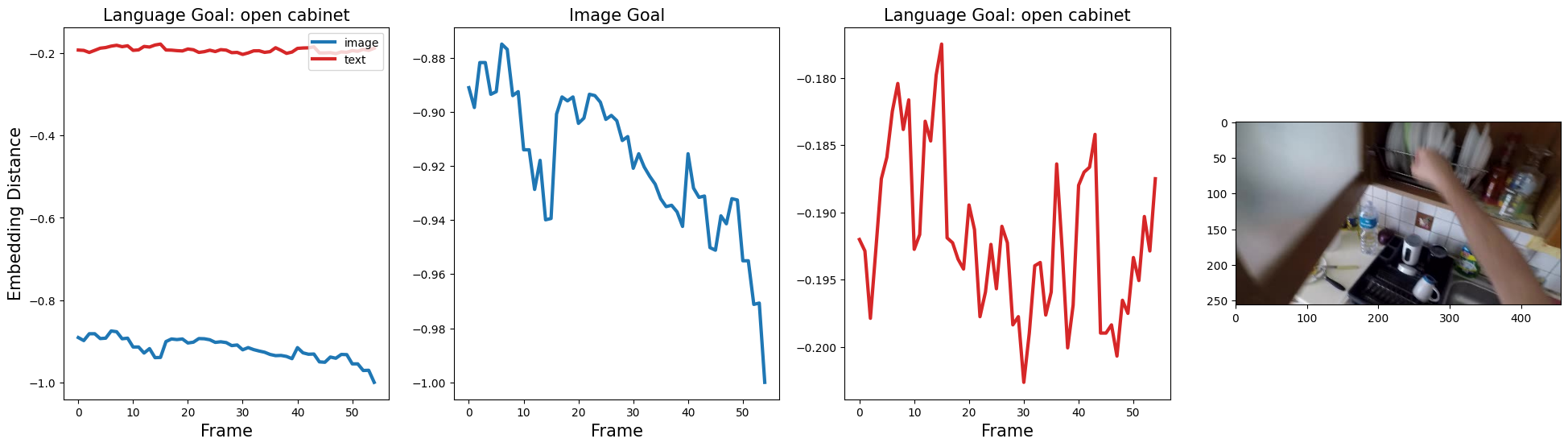}\hfill
  \includegraphics[width=0.7\linewidth]{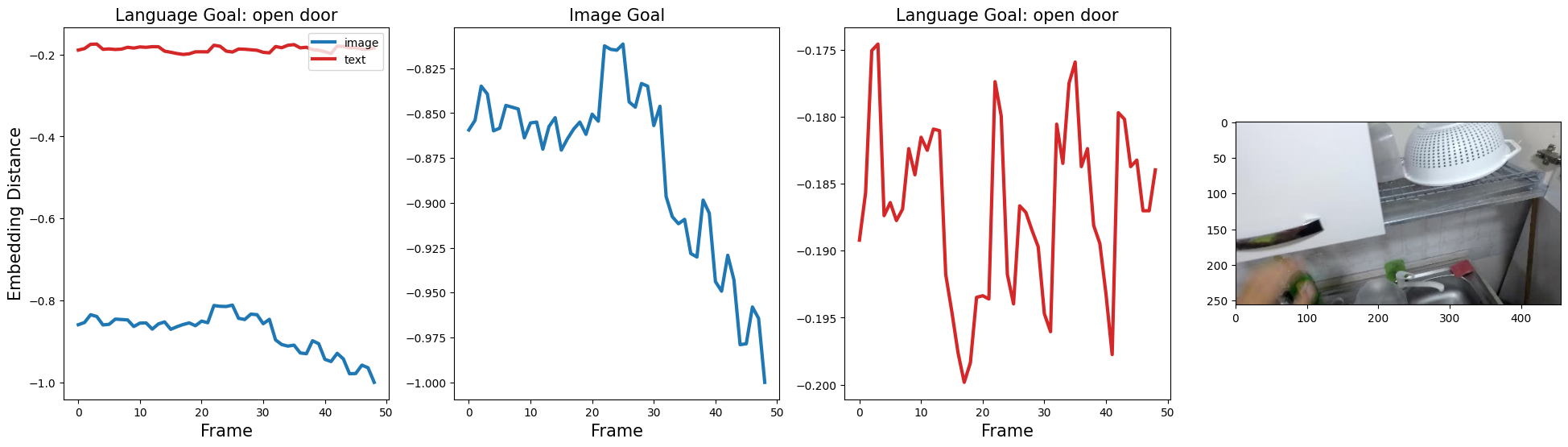}\hfill
  \includegraphics[width=0.7\linewidth]{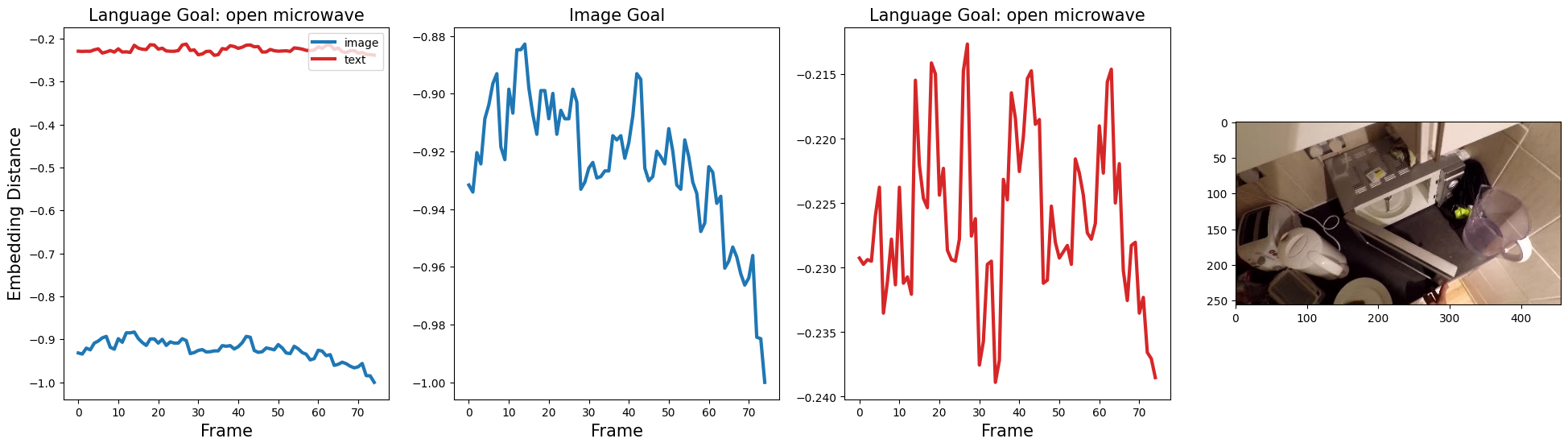}\hfill
  \caption{CLIP image and language goal reward curves on (unseen) EpicKitchen videos.} 
\label{figure:clip-epickitchen-qualitative-unseen}
\end{figure}

\begin{figure} 
\centering
  \includegraphics[width=0.7\linewidth]{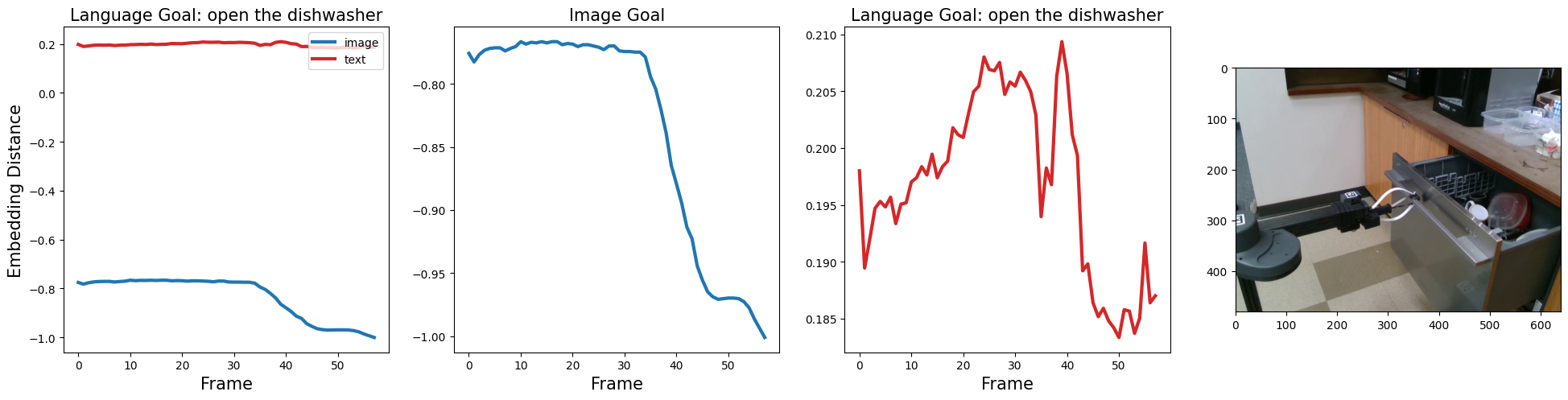}\hfill
  \includegraphics[width=0.7\linewidth]{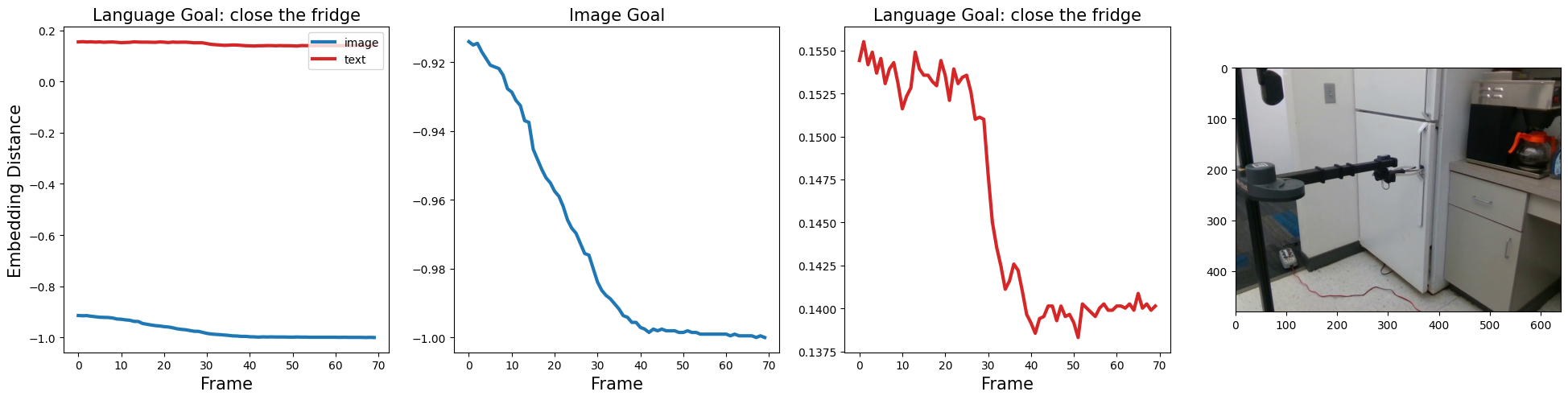}\hfill
  \includegraphics[width=0.7\linewidth]{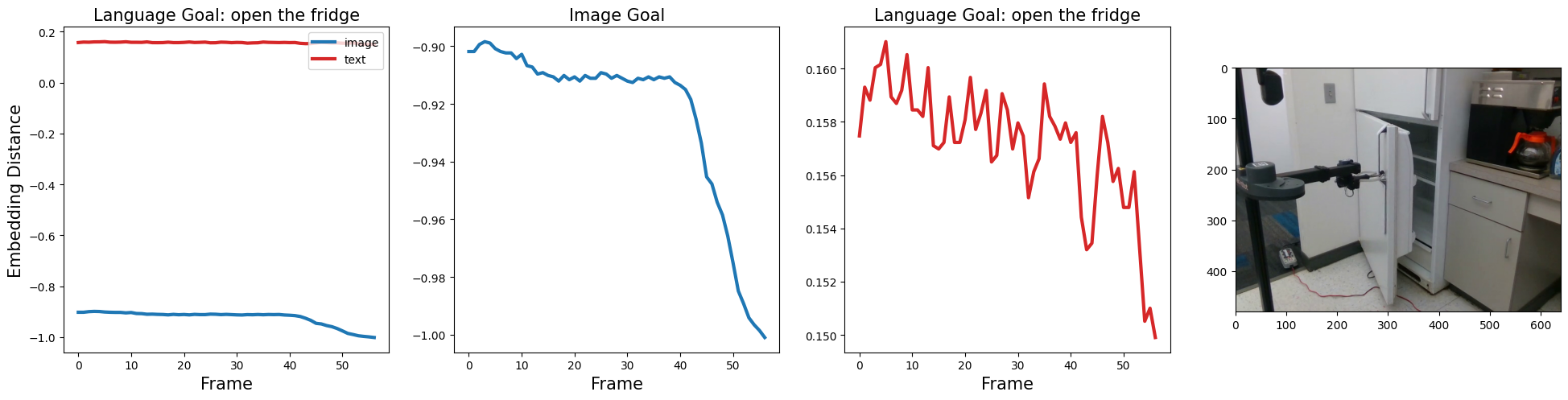}\hfill
    \includegraphics[width=0.7\linewidth]{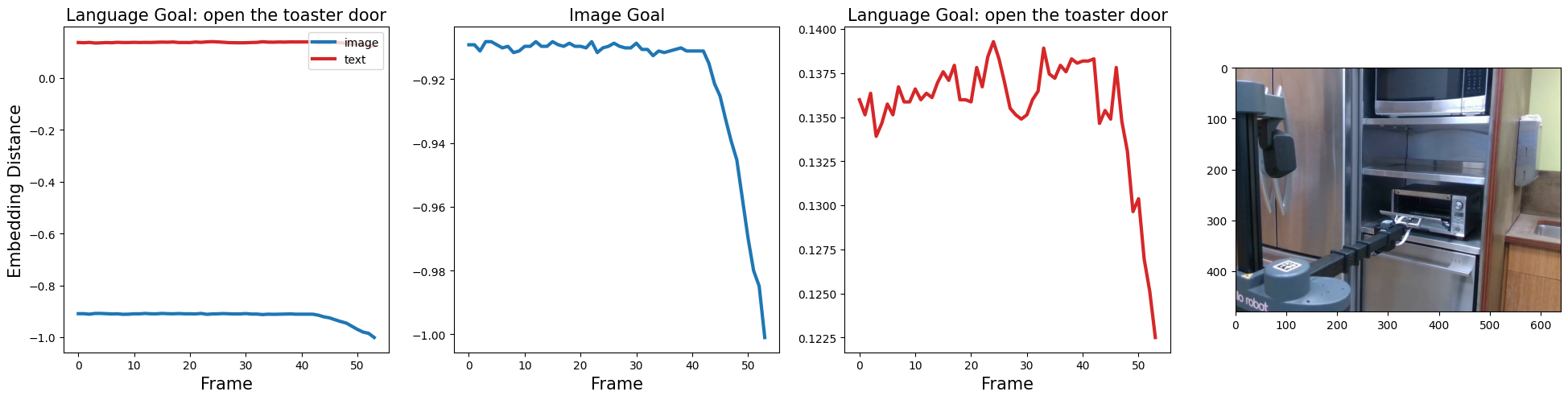}\hfill
      \includegraphics[width=0.7\linewidth]{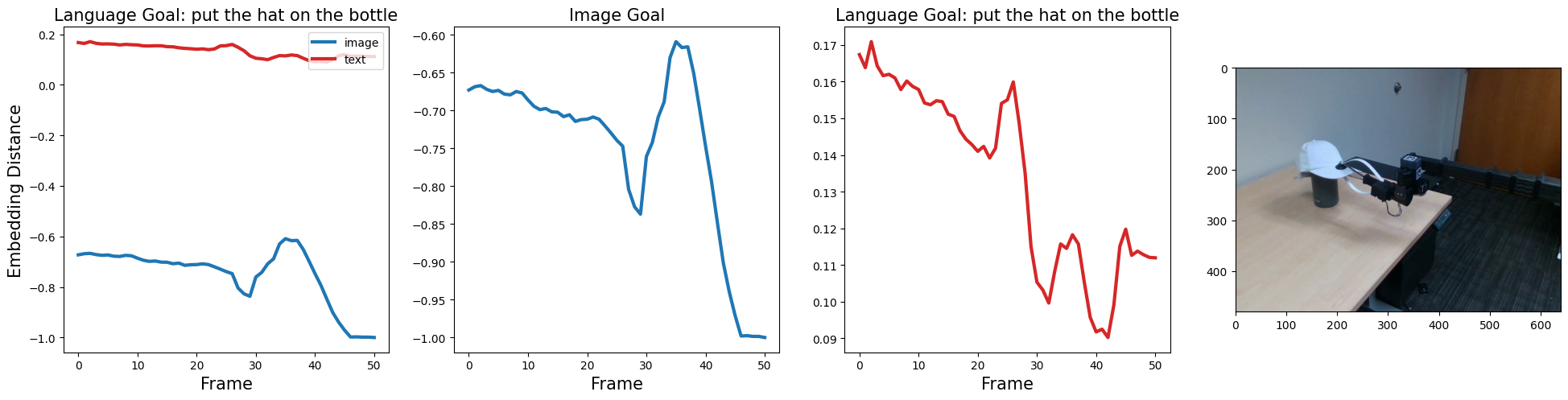}\hfill
        \includegraphics[width=0.7\linewidth]{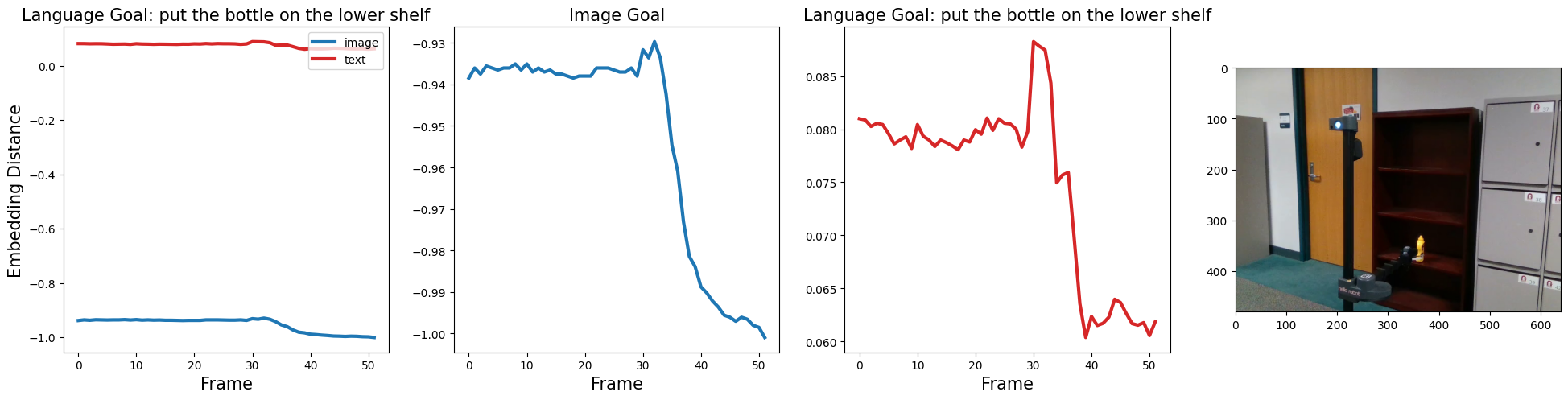}\hfill
  \caption{Success cases of LIV image and language goal reward curves on (unseen) Robot videos.} 
\label{figure:liv-hellorobot-success}
\end{figure}

\begin{figure} 
\centering
  \includegraphics[width=0.7\linewidth]{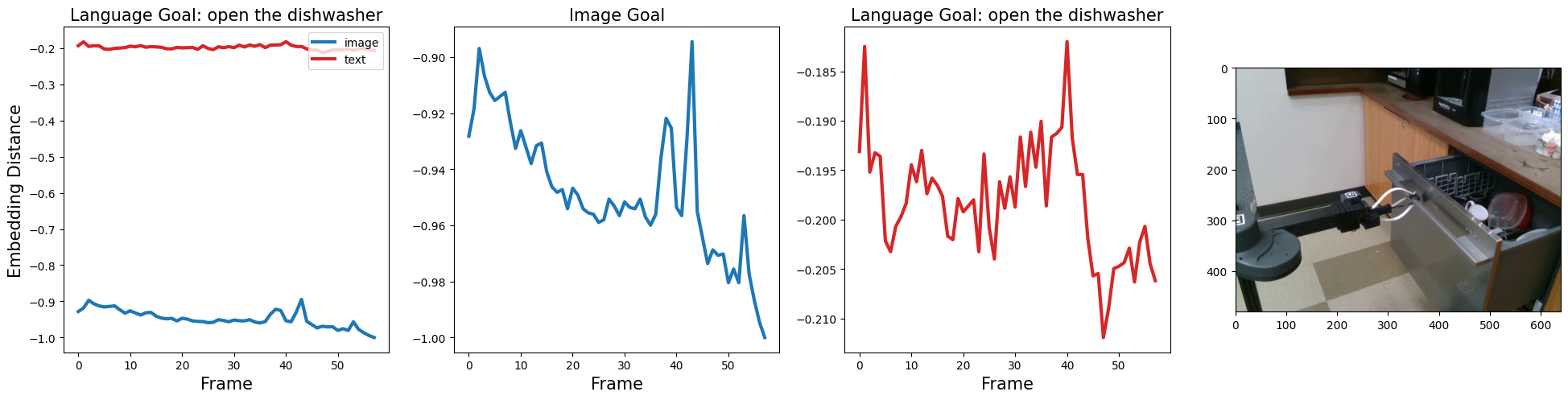}\hfill
  \includegraphics[width=0.7\linewidth]{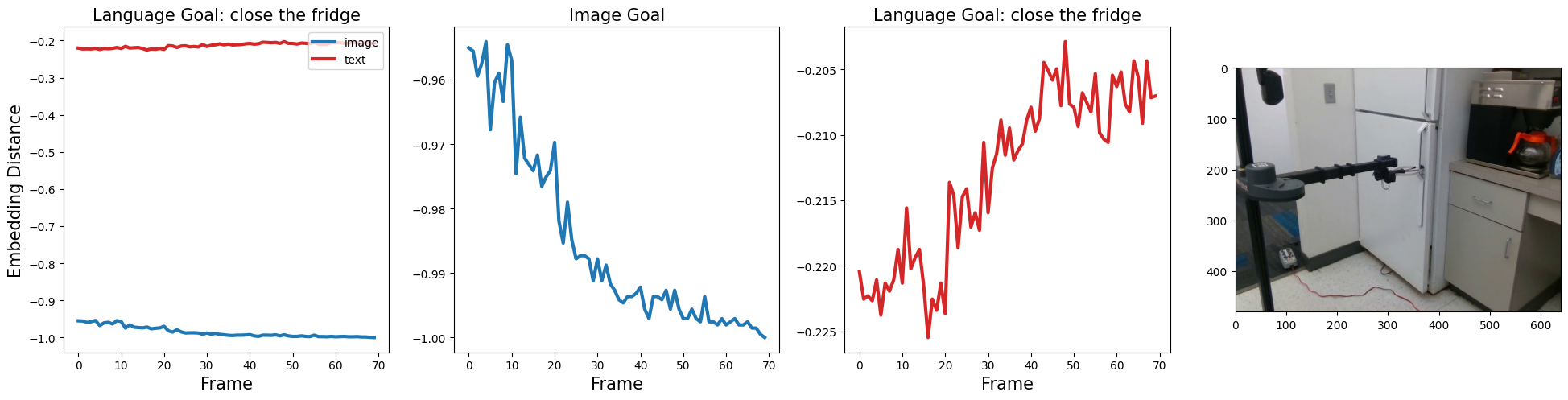}\hfill
  \includegraphics[width=0.7\linewidth]{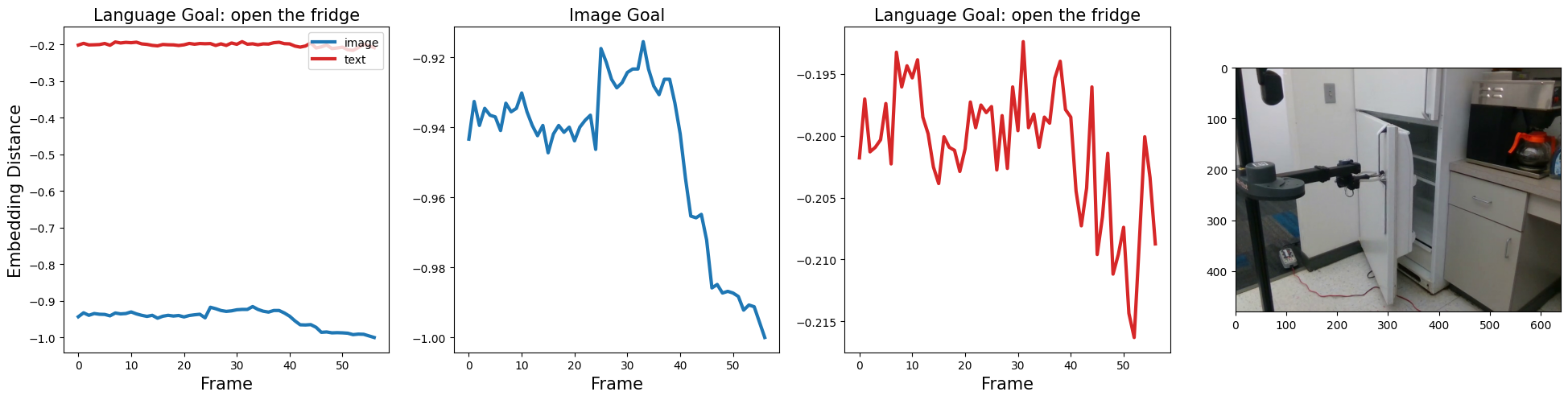}\hfill
    \includegraphics[width=0.7\linewidth]{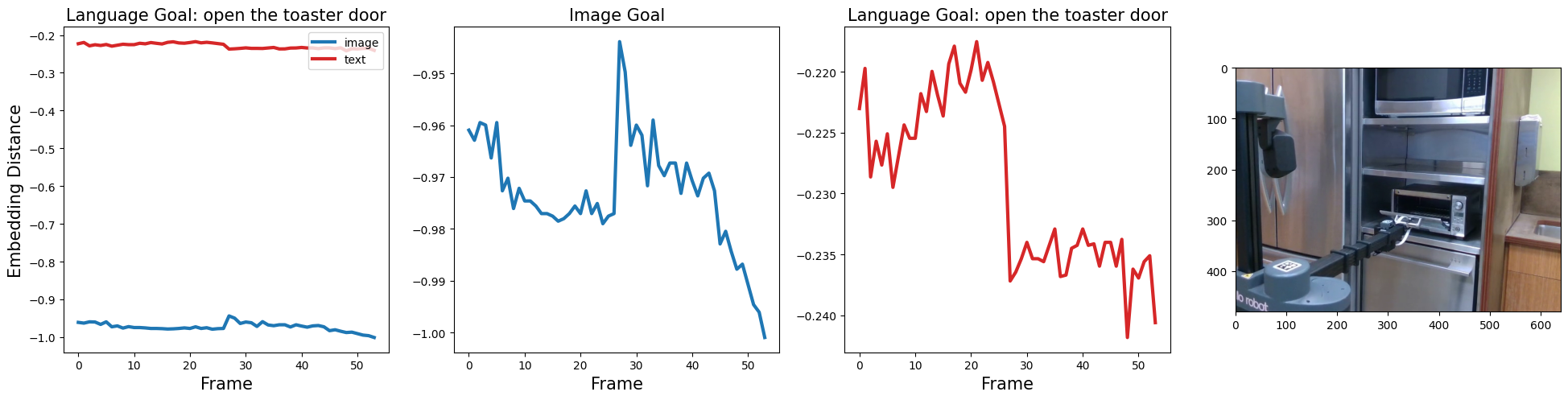}\hfill
      \includegraphics[width=0.7\linewidth]{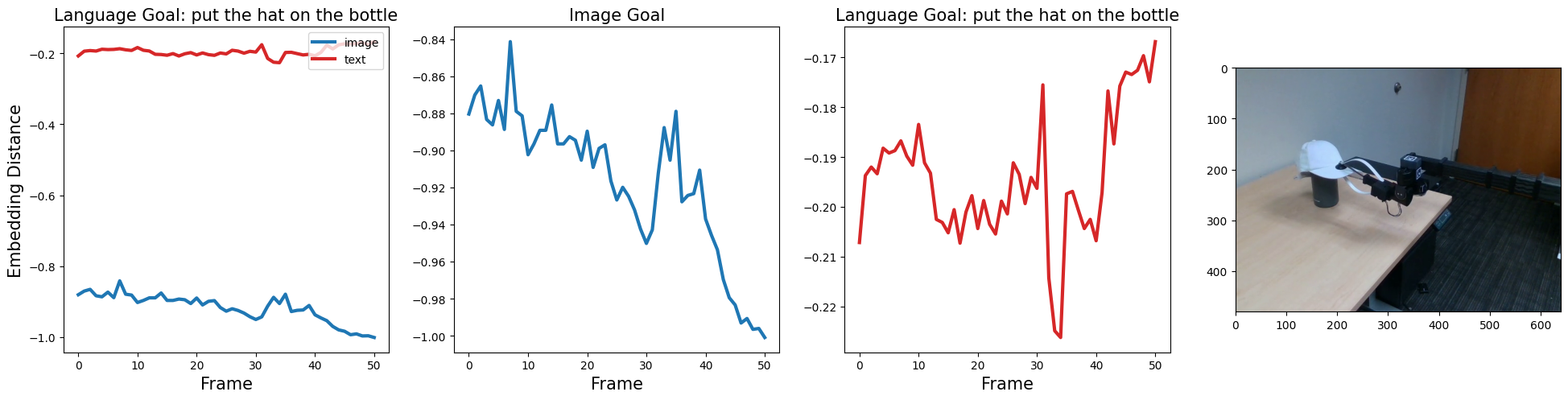}\hfill
        \includegraphics[width=0.7\linewidth]{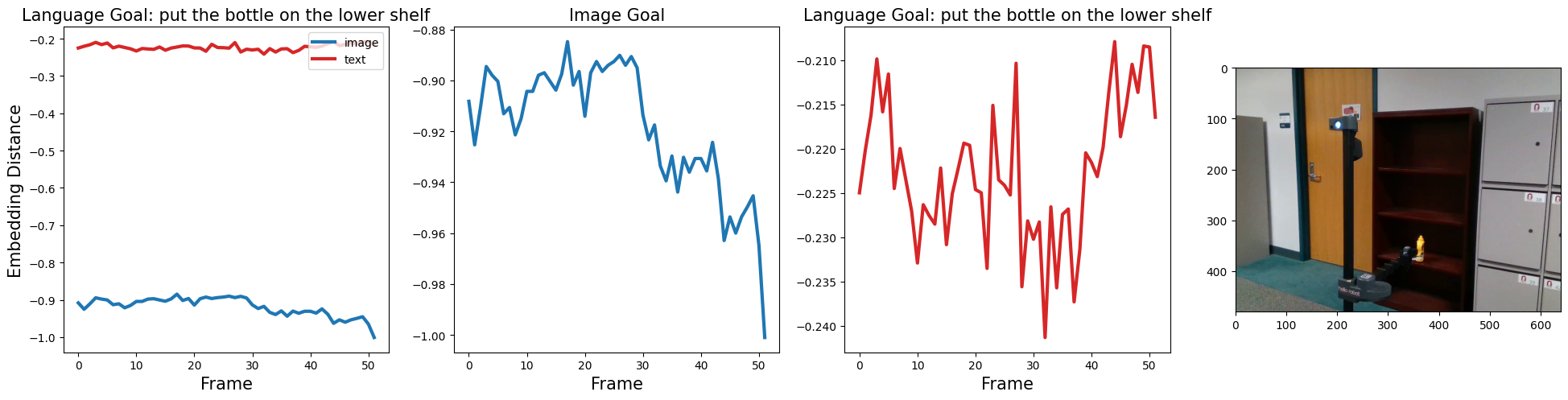}\hfill
  \caption{CLIP image and language goal reward curves on the same set of unseen Robot videos.} 
\label{figure:clip-hellorobot-success}
\end{figure}

\begin{figure} 
\centering
  \includegraphics[width=0.7\linewidth]{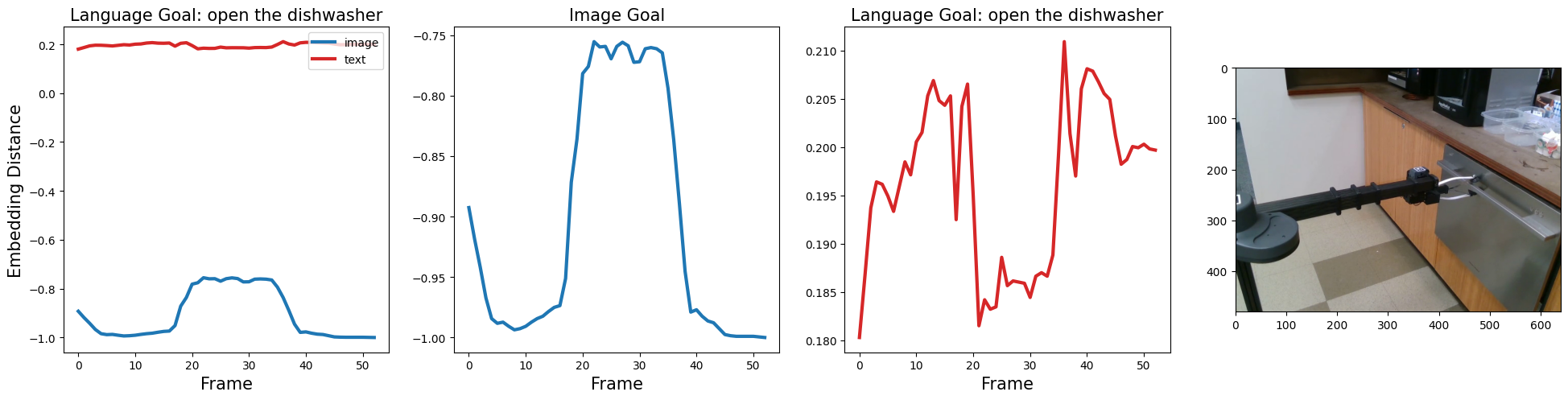}\hfill
  \includegraphics[width=0.7\linewidth]{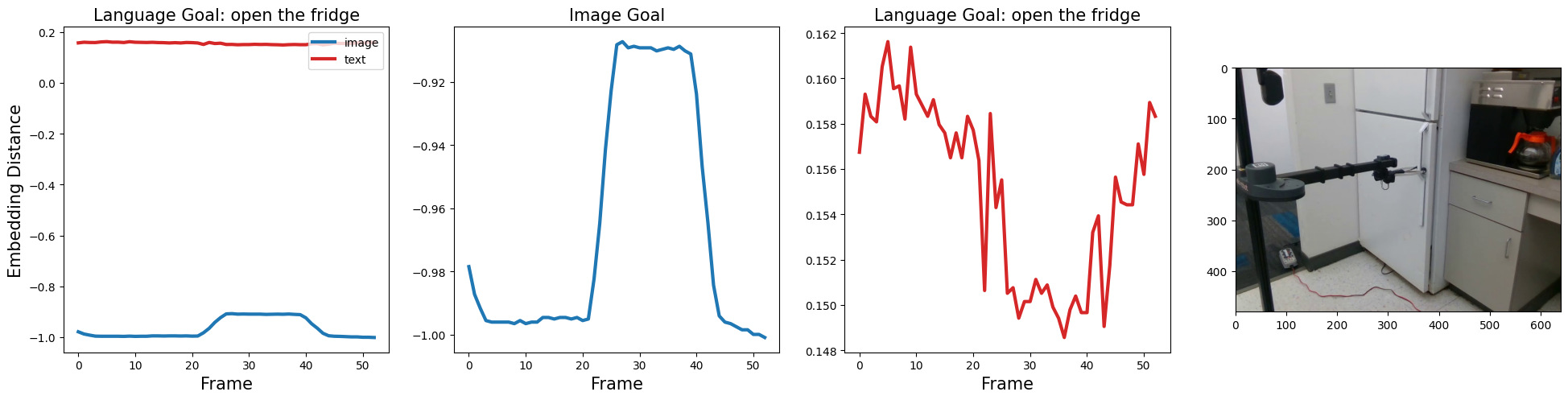}\hfill
  \includegraphics[width=0.7\linewidth]{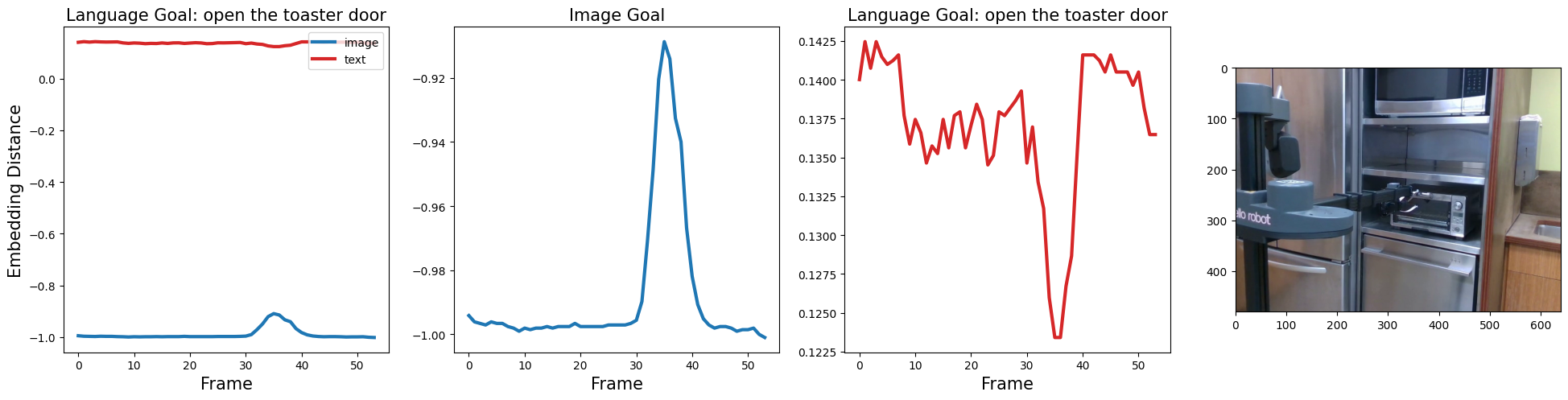}\hfill
  \caption{When the video contains opposite actions, LIV's image and language reward curves exhibit inverted trend because the image goal depicts the action completed in the last frame, which is opposite from the action described in the language goal that has already occurred in the middle of the video.} 
\label{figure:liv-hellorobot-long}
\end{figure}

\begin{figure} 
\centering
  \includegraphics[width=0.7\linewidth]{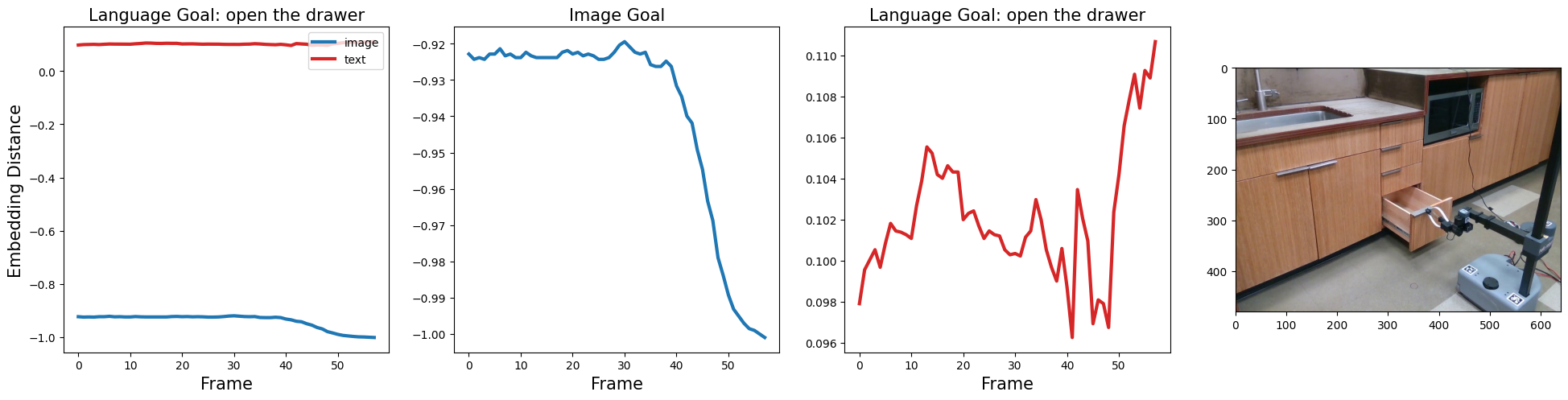}\hfill
  \includegraphics[width=0.7\linewidth]{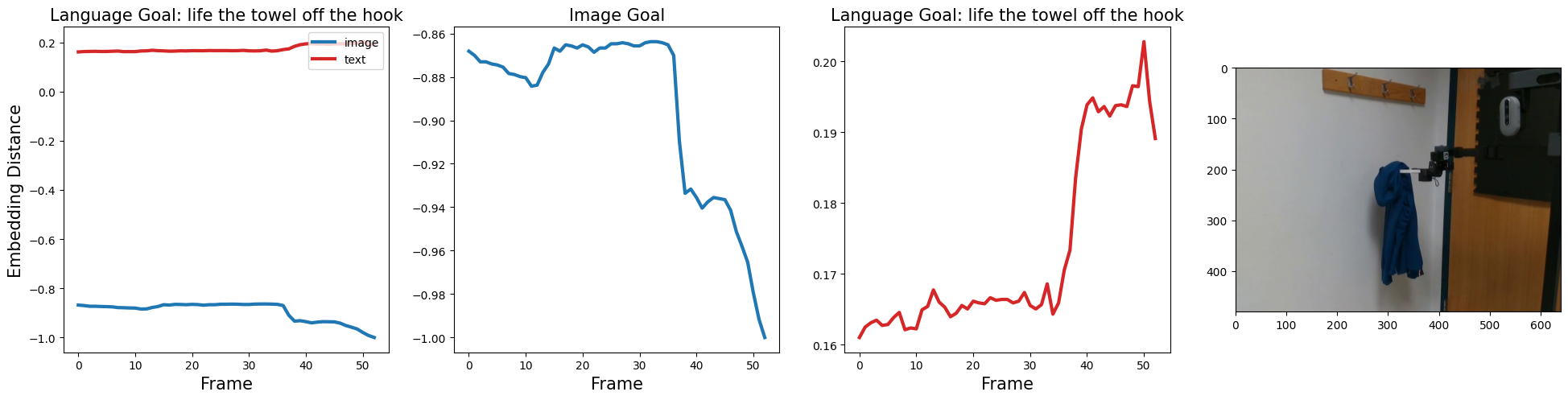}\hfill
  \includegraphics[width=0.7\linewidth]{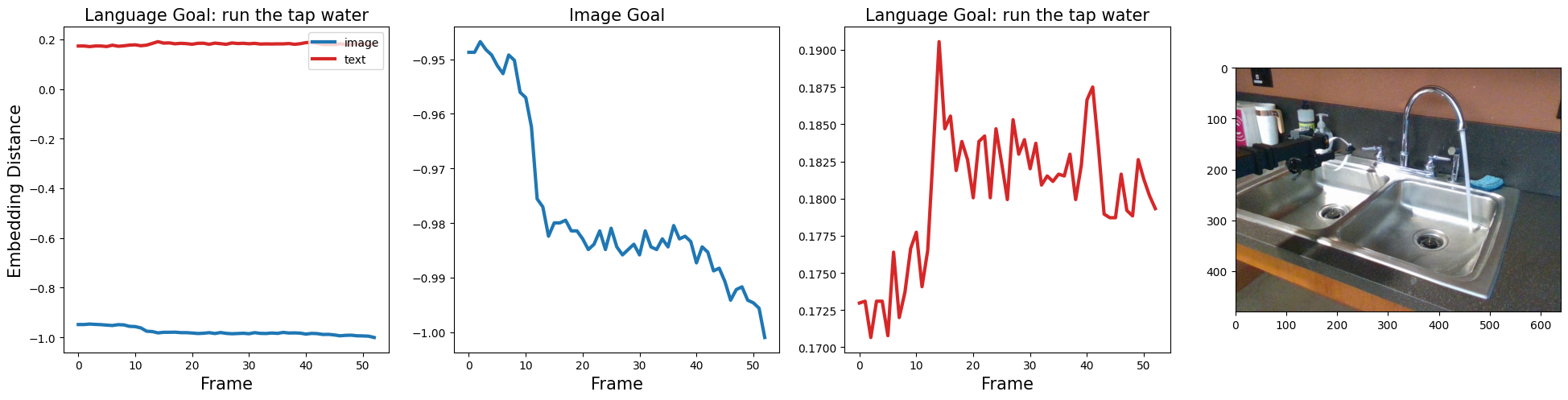}\hfill
  \caption{Failure cases of LIV image and language goal reward curves on (unseen) Robot videos.} 
\label{figure:liv-hellorobot-failure}
\end{figure}

\begin{figure} 
\centering
  \includegraphics[width=0.49\linewidth]{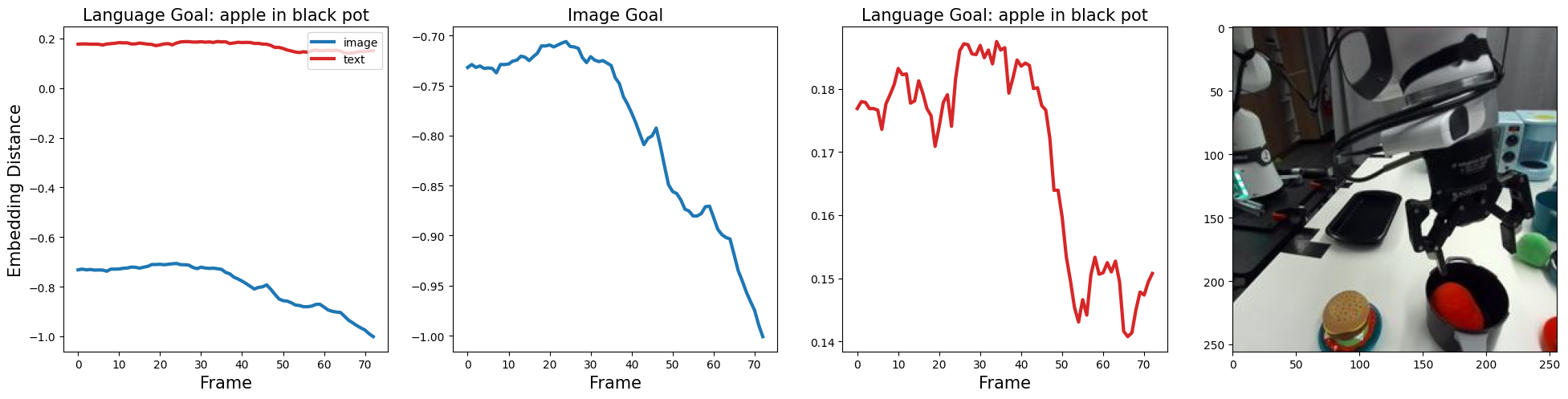}\hfill  \includegraphics[width=0.49\linewidth]{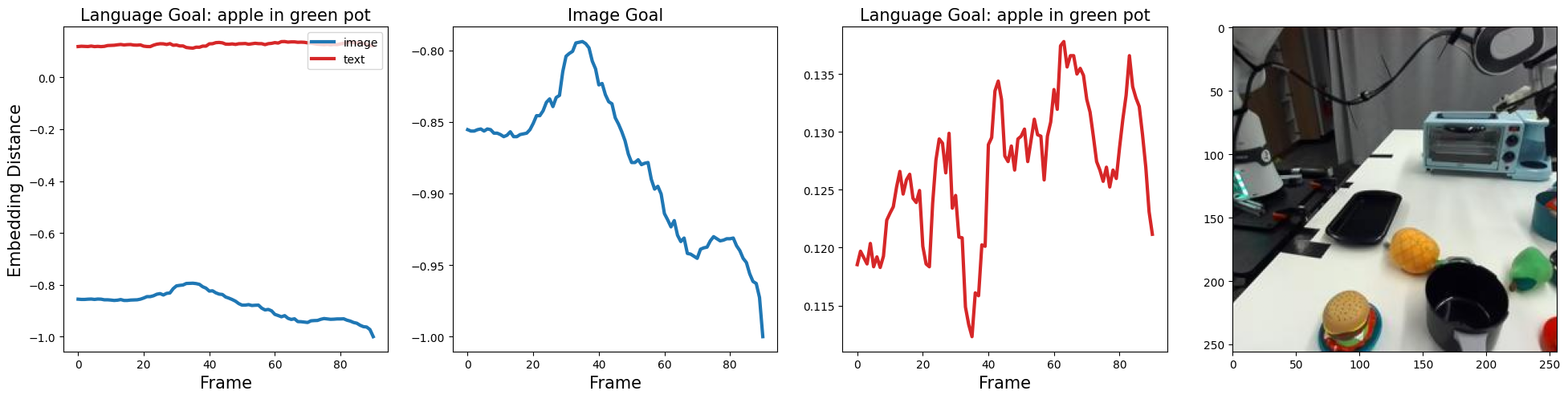}\hfill  \includegraphics[width=0.49\linewidth]{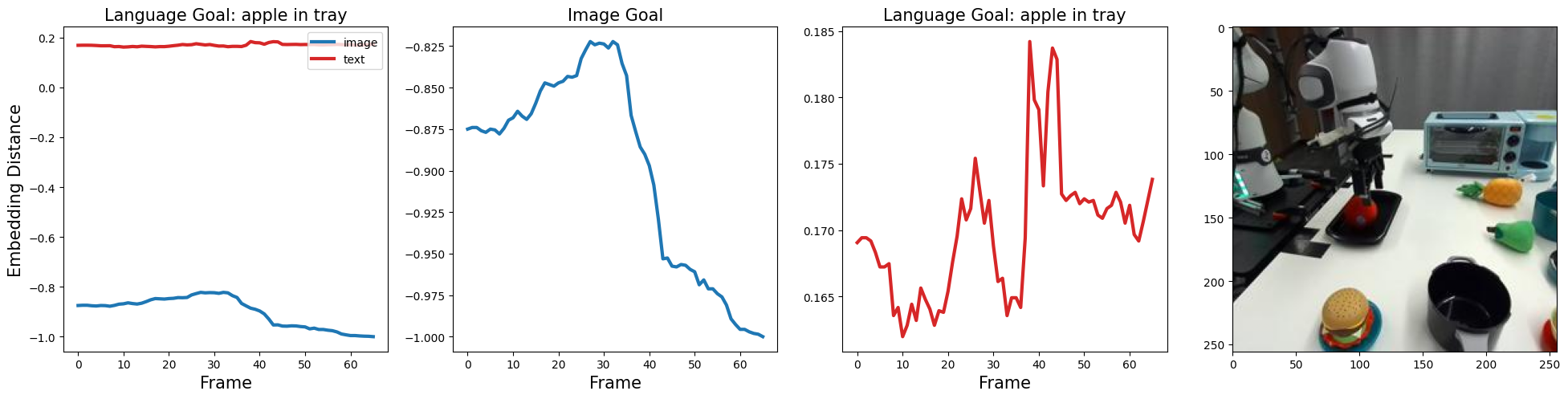}\hfill
    \includegraphics[width=0.49\linewidth]{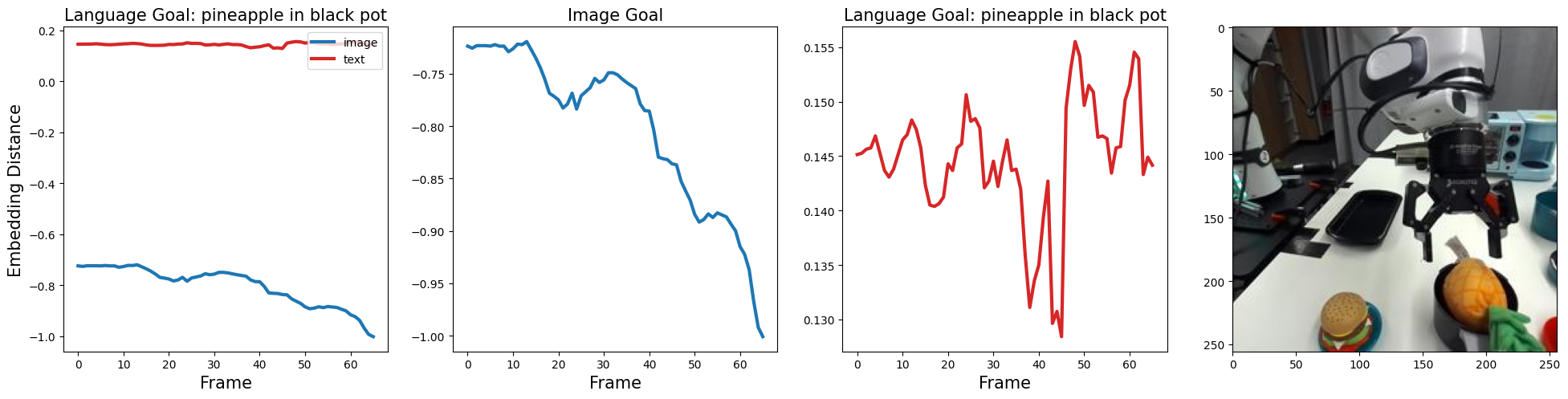}\hfill  \includegraphics[width=0.49\linewidth]{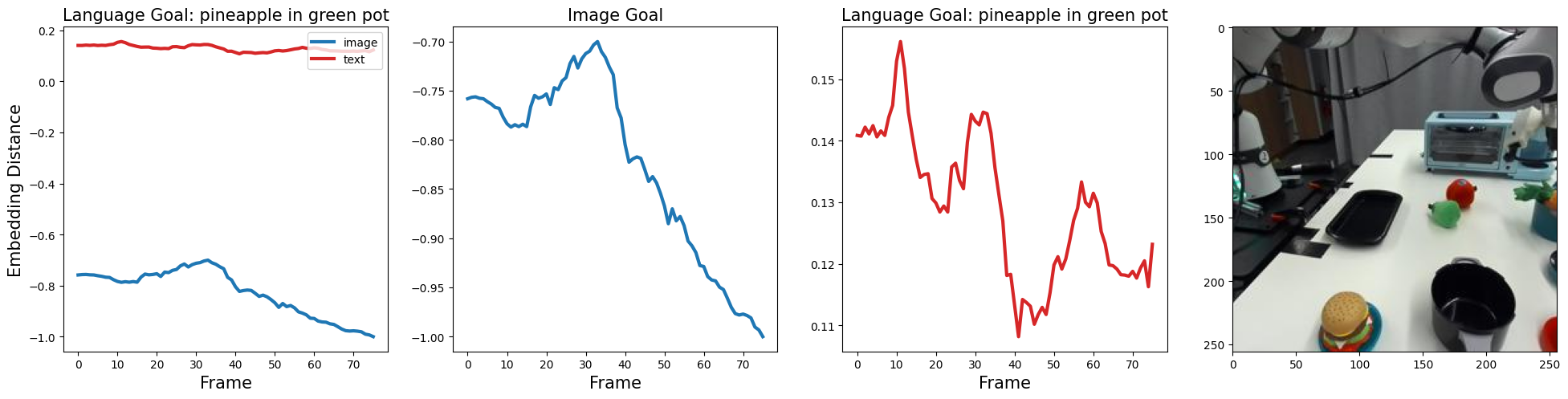}\hfill  \includegraphics[width=0.49\linewidth]{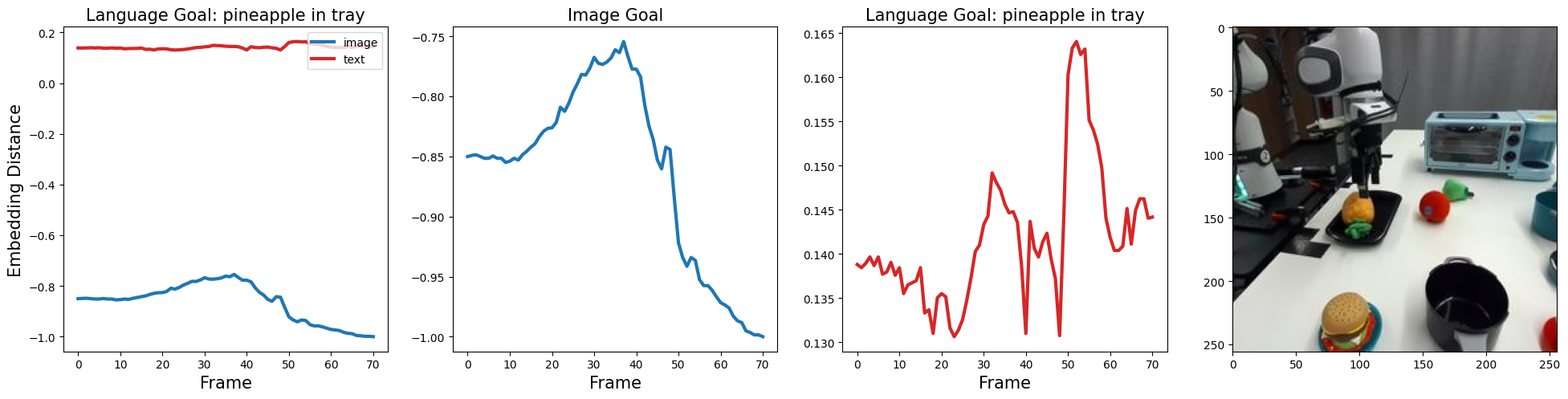}\hfill
      \includegraphics[width=0.49\linewidth]{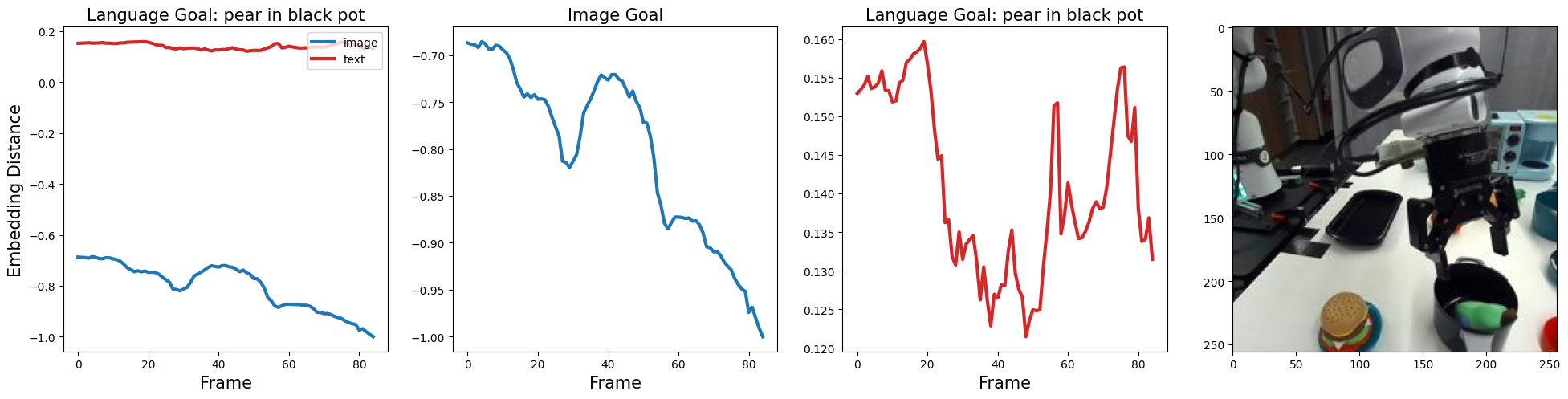}\hfill  \includegraphics[width=0.49\linewidth]{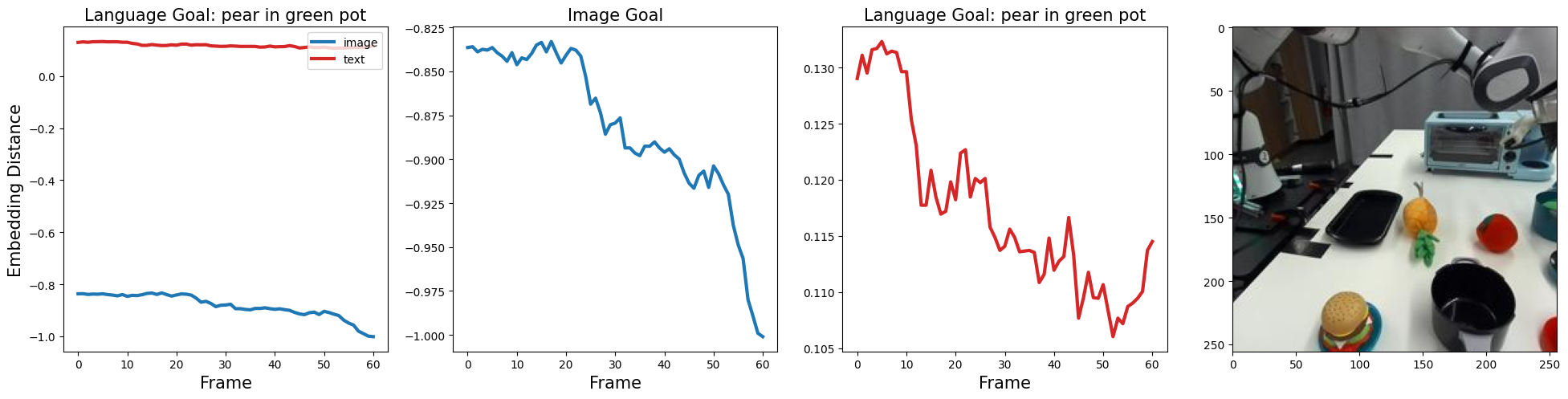}\hfill  \includegraphics[width=0.49\linewidth]{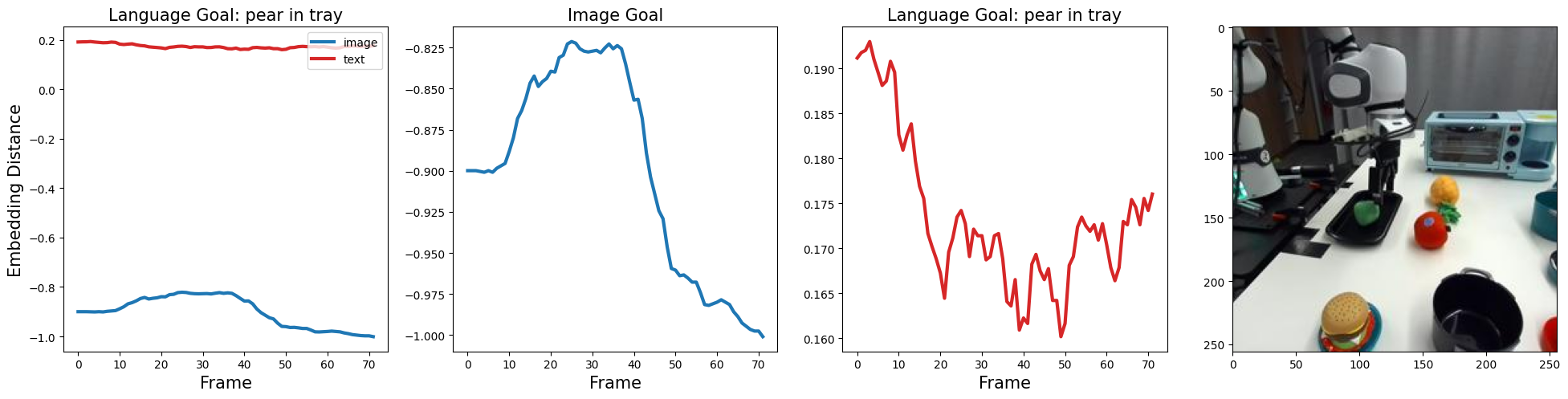}\hfill
  \caption{Pre-trained LIV image and language goal reward curves on RealRobot tasks.} 
\label{figure:liv-realrobot-qualitative}
\end{figure}

\begin{figure} 
\centering
  \includegraphics[width=0.49\linewidth]{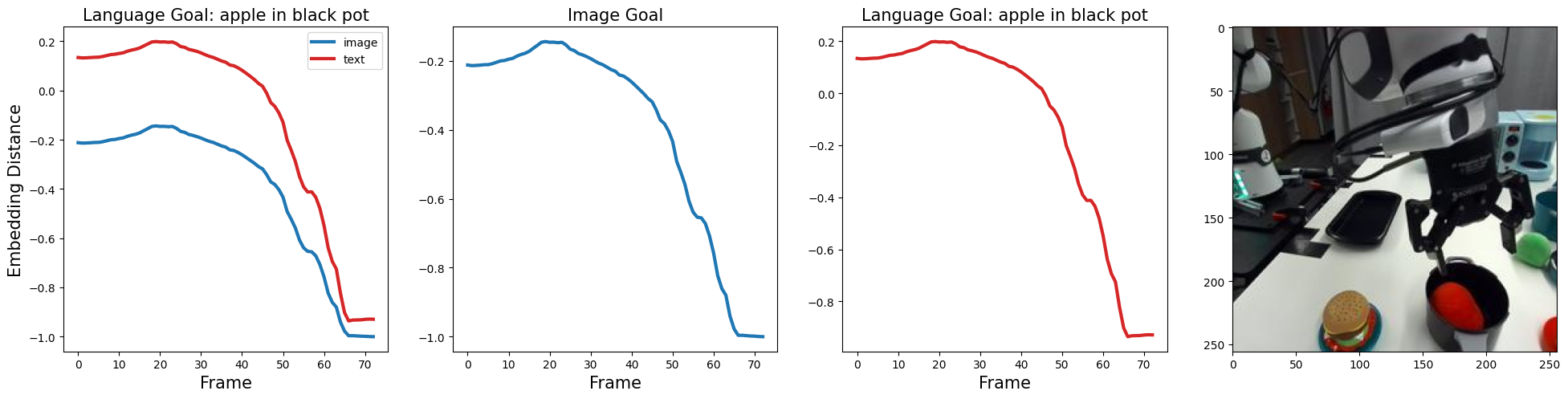}\hfill 
    \includegraphics[width=0.49\linewidth]{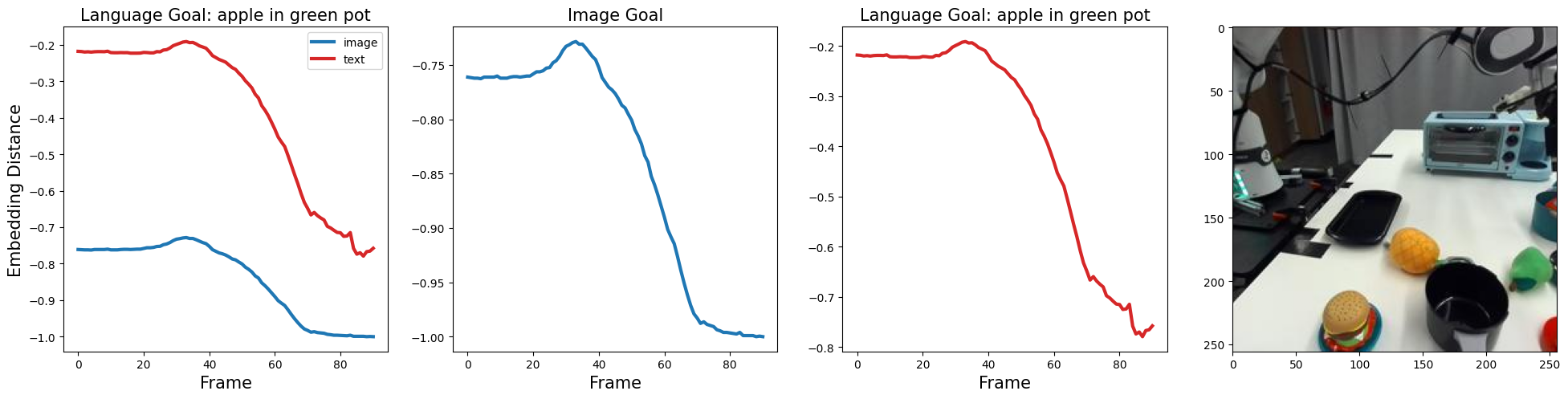}\hfill 
      \includegraphics[width=0.49\linewidth]{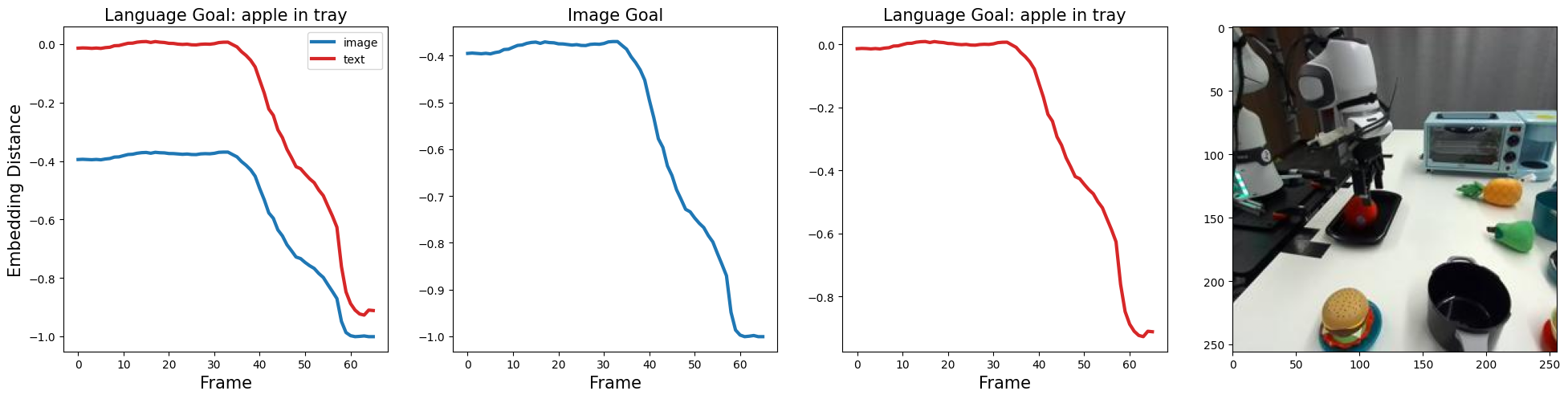}\hfill 
  \includegraphics[width=0.49\linewidth]{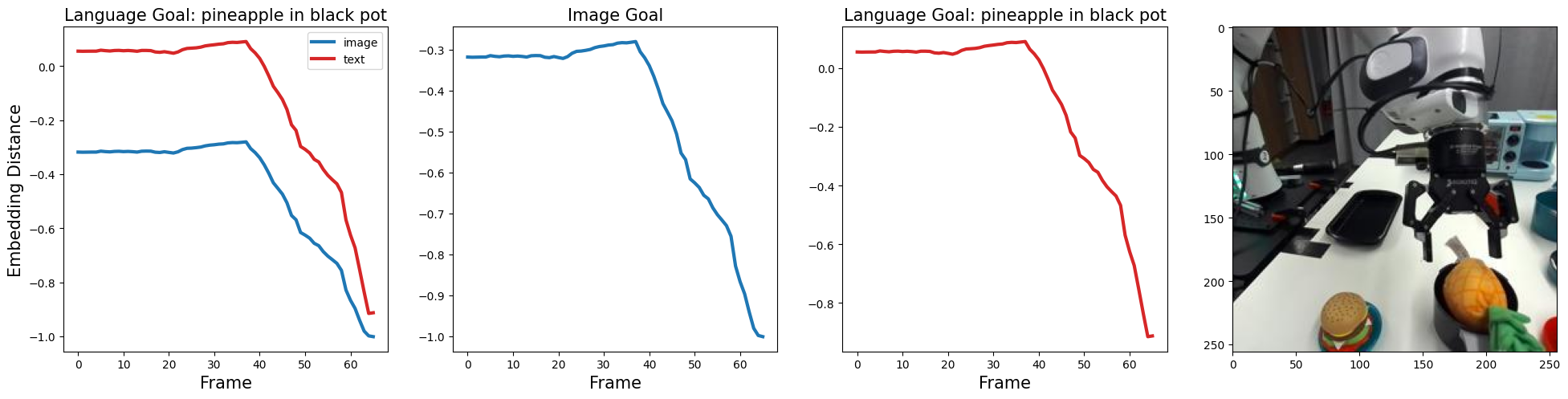}\hfill 
    \includegraphics[width=0.49\linewidth]{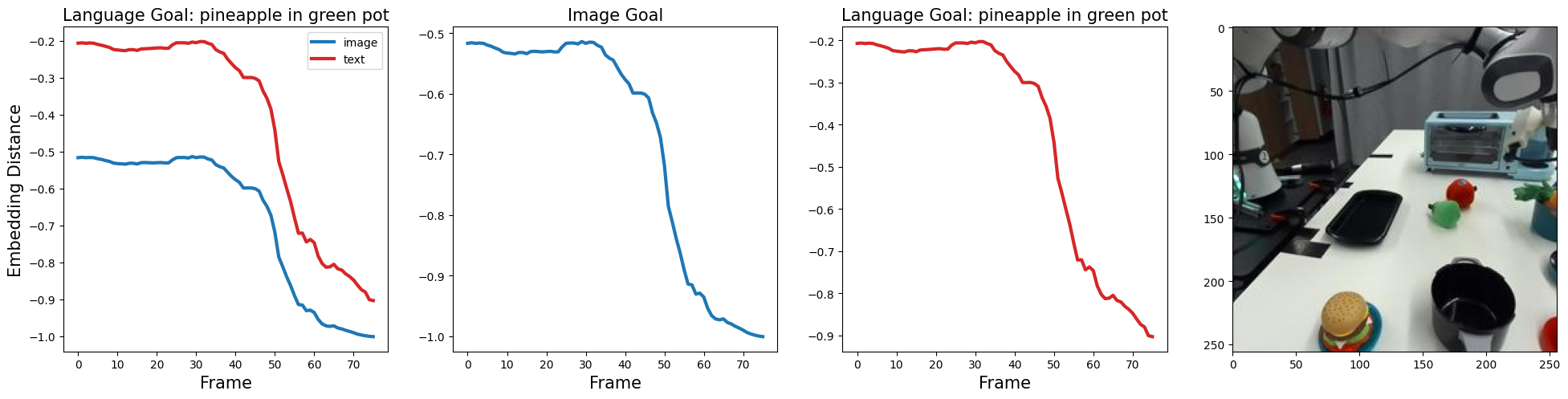}\hfill 
      \includegraphics[width=0.49\linewidth]{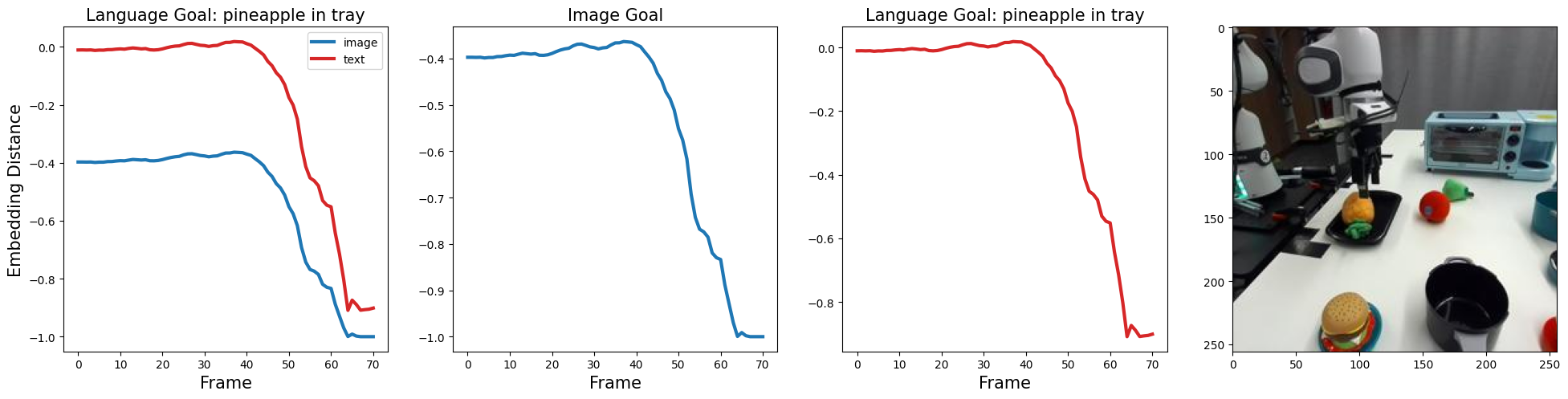}\hfill 
  \includegraphics[width=0.49\linewidth]{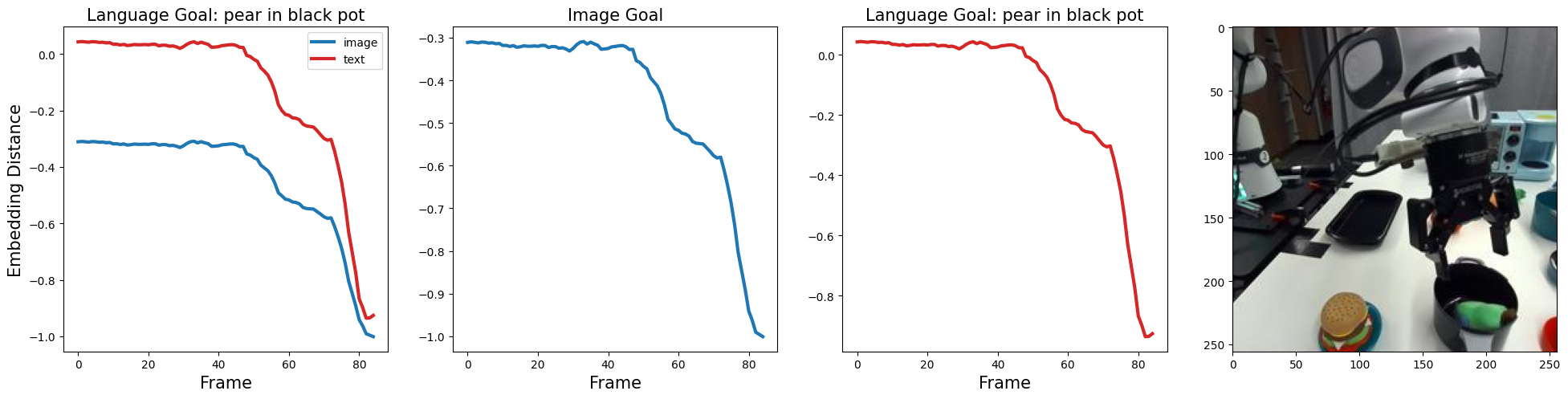}\hfill 
    \includegraphics[width=0.49\linewidth]{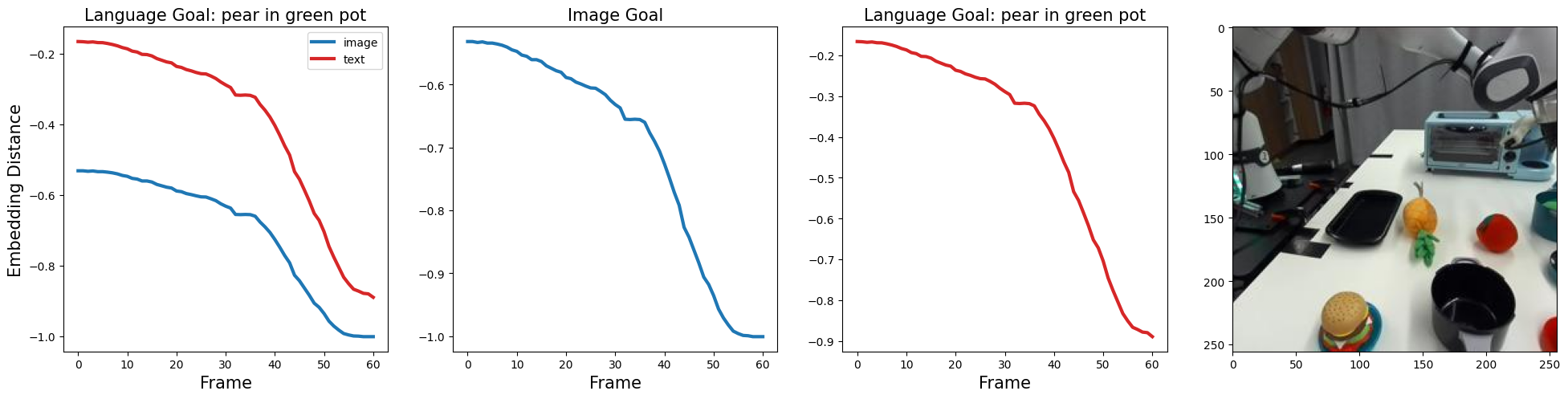}\hfill 
      \includegraphics[width=0.49\linewidth]{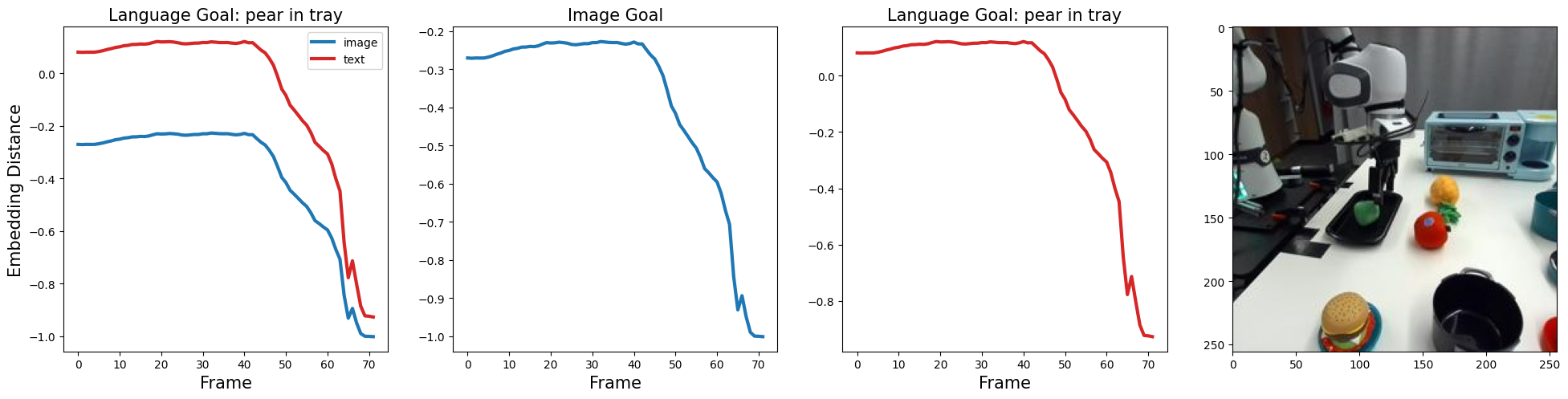}\hfill 
  \caption{LIV (LIV fine-tuned) image and language goal reward curves on RealRobot tasks.} 
\label{figure:liv-livfinetuned-realrobot-qualitative}
\end{figure}

\begin{figure} 
\centering
  \includegraphics[width=0.7\linewidth]{figures/liv_franka_qualitative/frankakitchen.png}\hfill
  \includegraphics[width=0.7\linewidth]{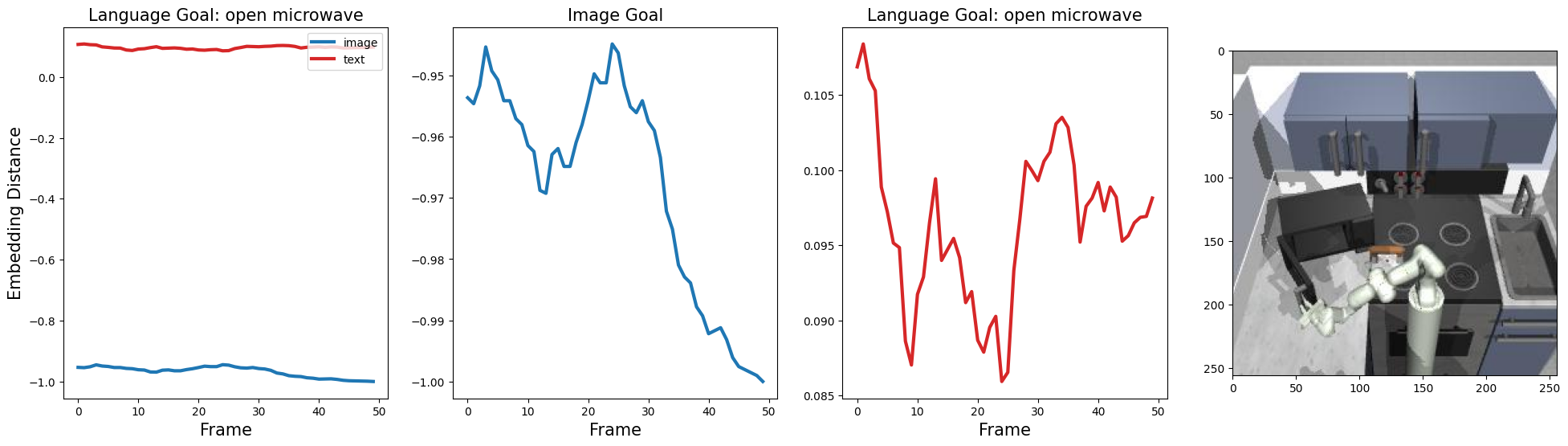}\hfill
  \includegraphics[width=0.7\linewidth]{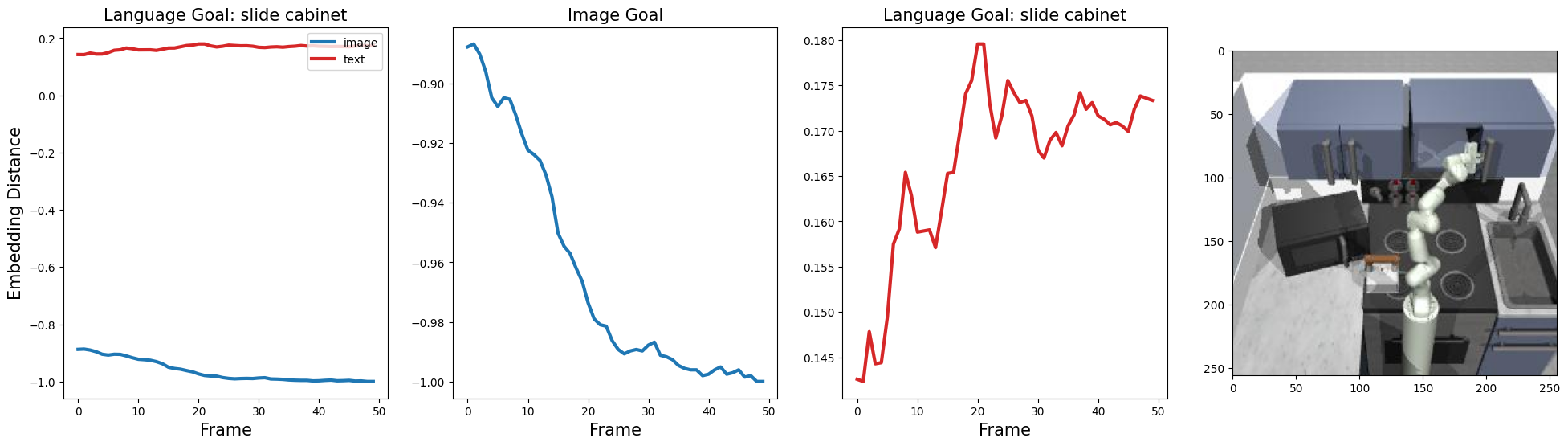}\hfill
  \includegraphics[width=0.7\linewidth]{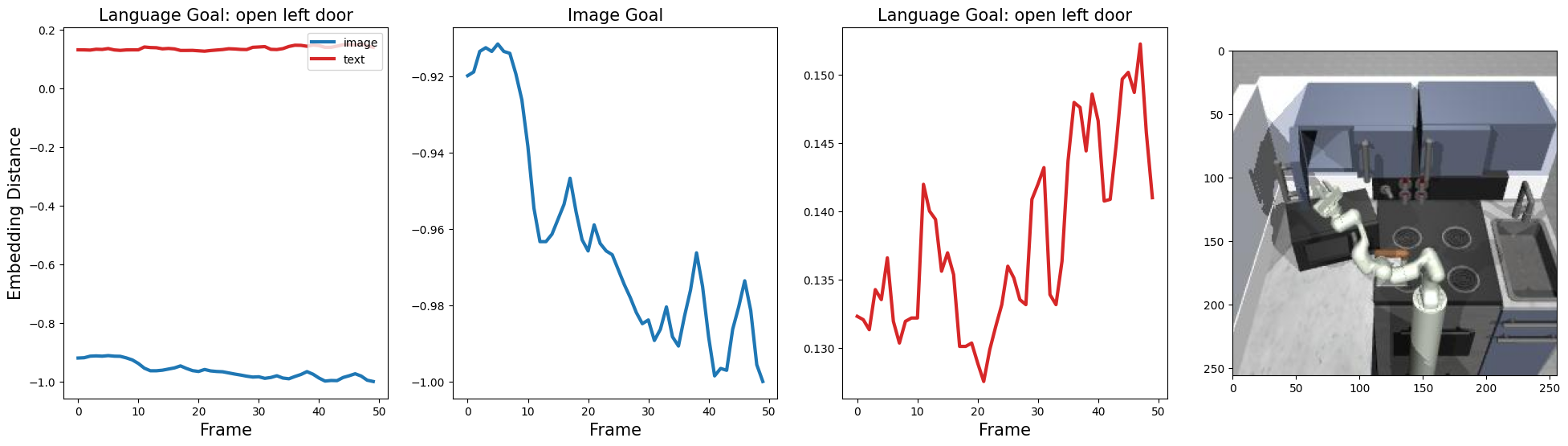}\hfill
  \includegraphics[width=0.7\linewidth]{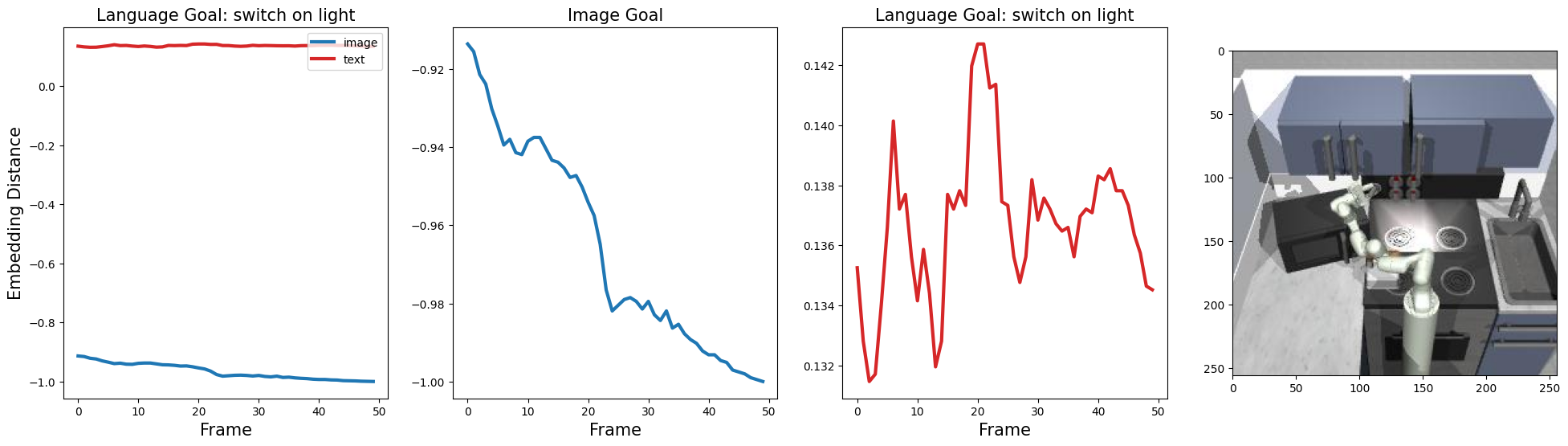}\hfill
    \includegraphics[width=0.7\linewidth]{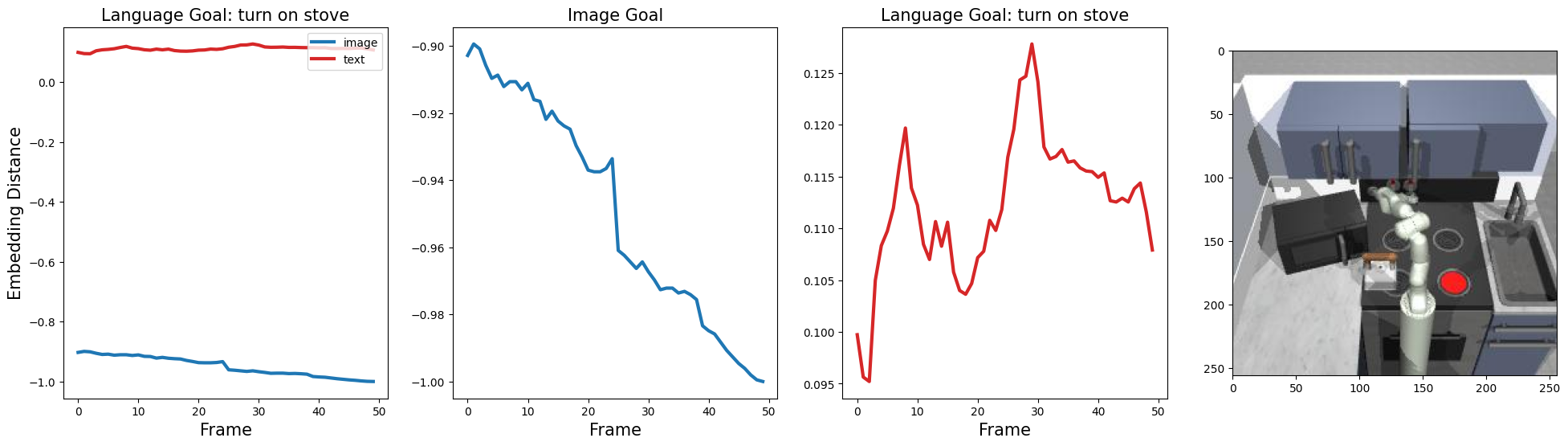}\hfill
  \caption{Pre-trained LIV-EPIC image and language goal reward curves on simulated FrankaKitchen tasks.} 
\label{figure:liv-frankakitchen-qualitative}
\end{figure}

\begin{figure} 
\centering
  \includegraphics[width=0.7\linewidth]{figures/liv_livfinetune_franka_qualitative/frankakitchen.png}\hfill
  \includegraphics[width=0.7\linewidth]{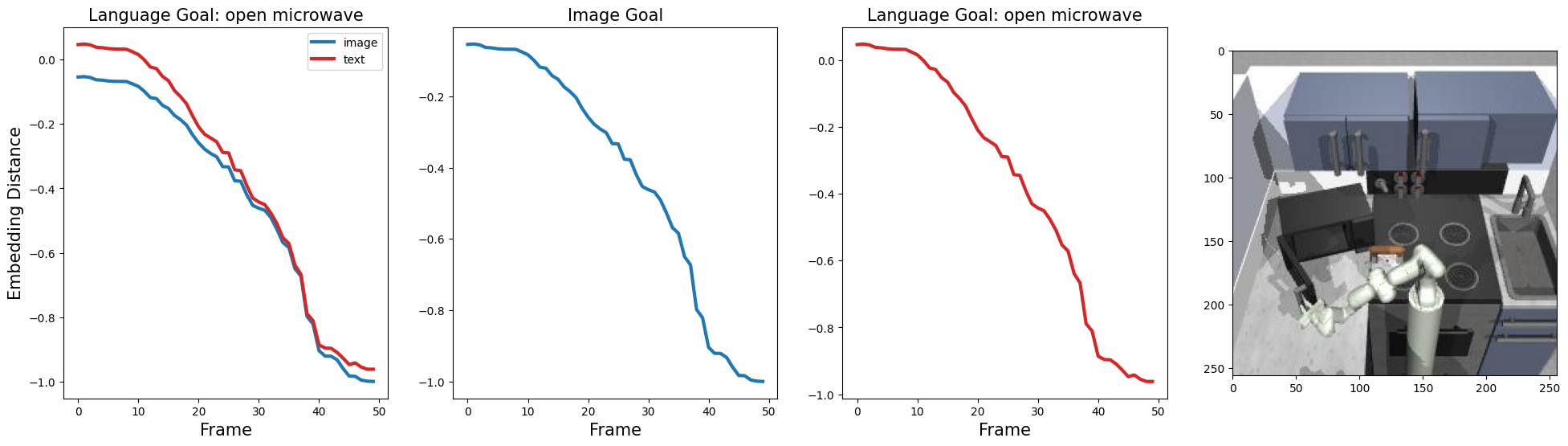}\hfill
  \includegraphics[width=0.7\linewidth]{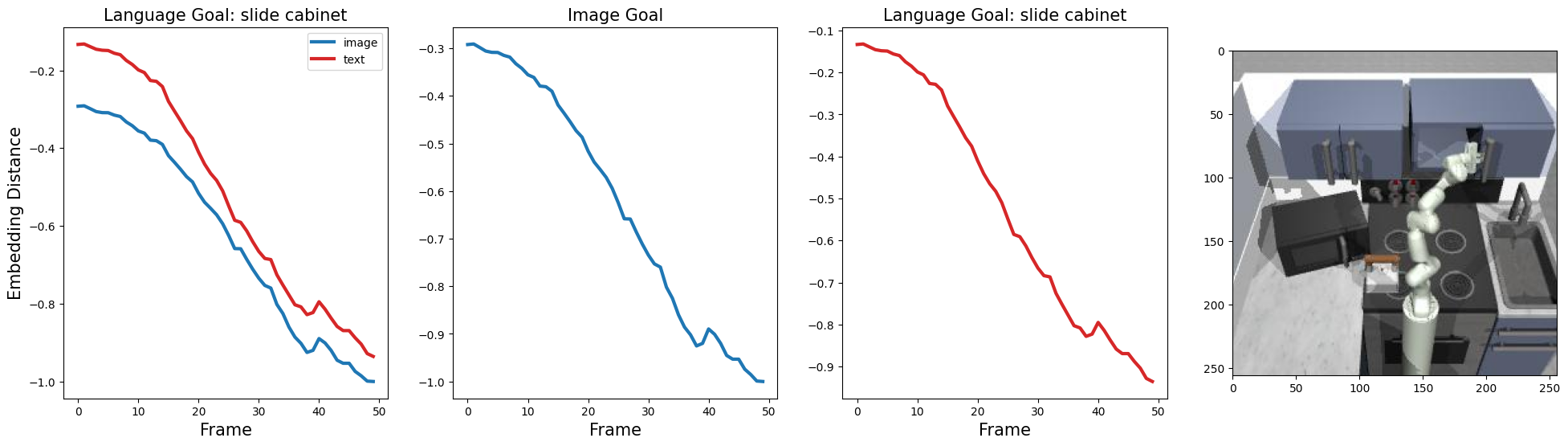}\hfill
  \includegraphics[width=0.7\linewidth]{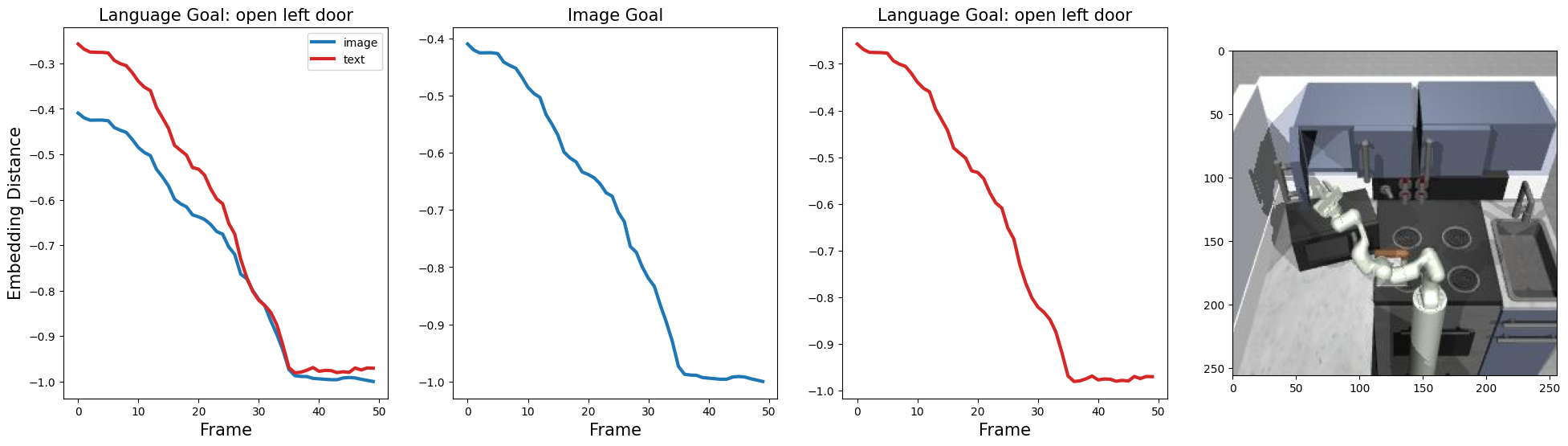}\hfill
  \includegraphics[width=0.7\linewidth]{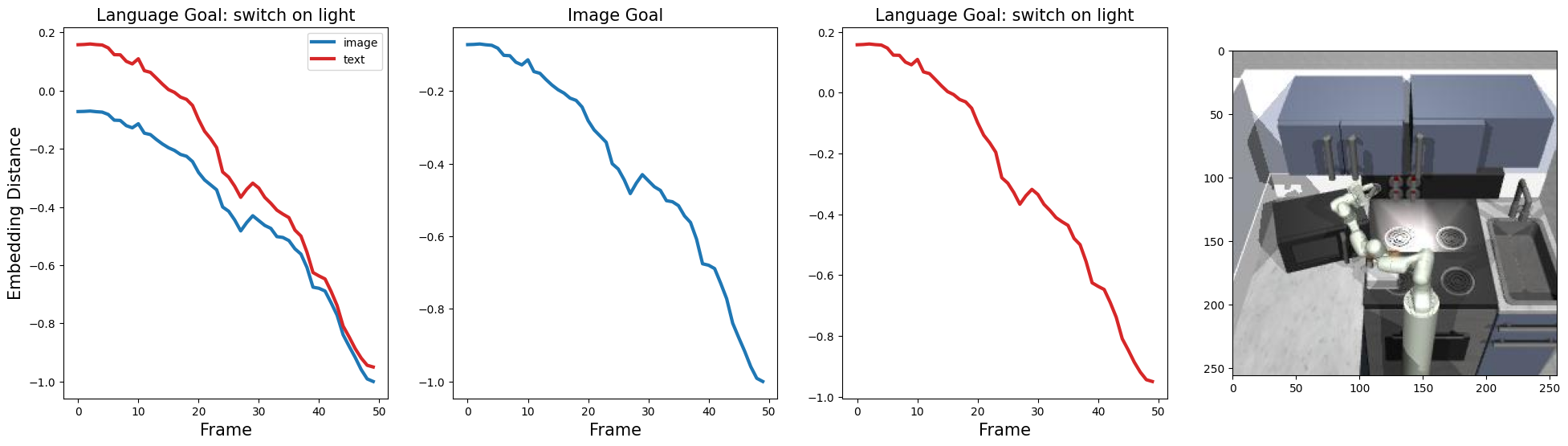}\hfill
    \includegraphics[width=0.7\linewidth]{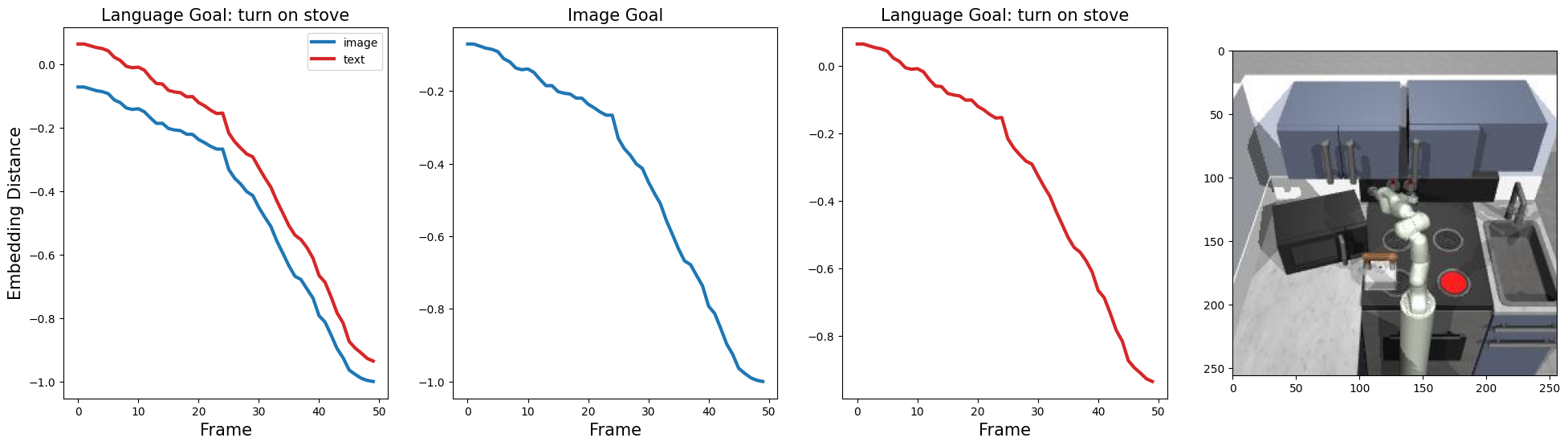}\hfill
  \caption{LIV-EPIC (LIV fine-tuned) image and language goal reward curves on simulated FrankaKitchen tasks.} 
\label{figure:liv-livfinetune-frankakitchen-qualitative}
\end{figure}

\begin{figure} 
\centering
  \includegraphics[width=0.7\linewidth]{figures/liv_clipfinetune_franka_qualitative/frankakitchen.png}\hfill
  \includegraphics[width=0.7\linewidth]{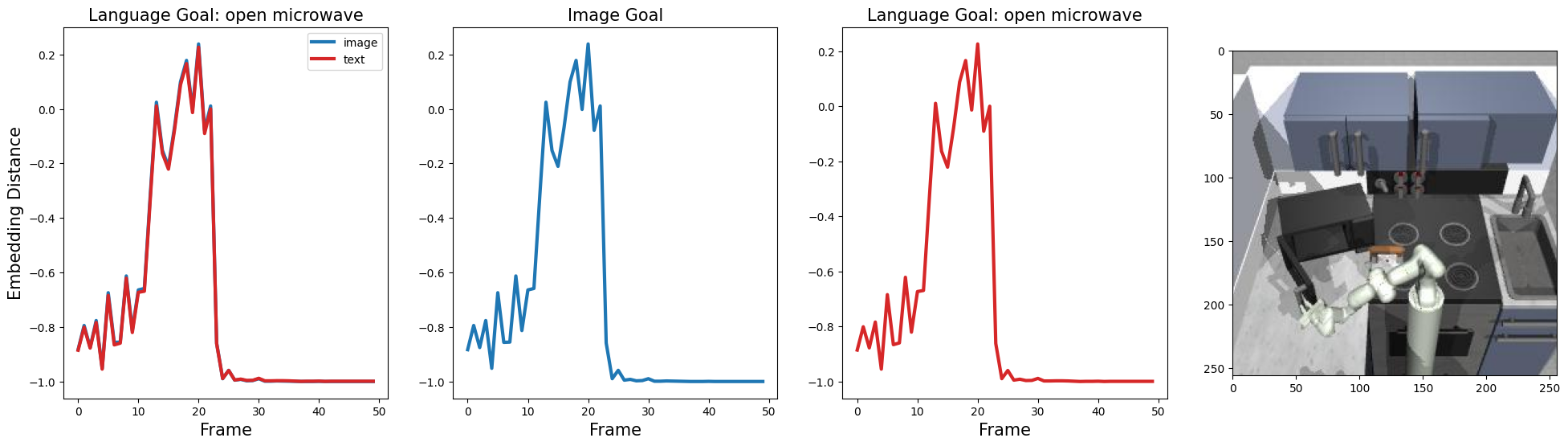}\hfill
  \includegraphics[width=0.7\linewidth]{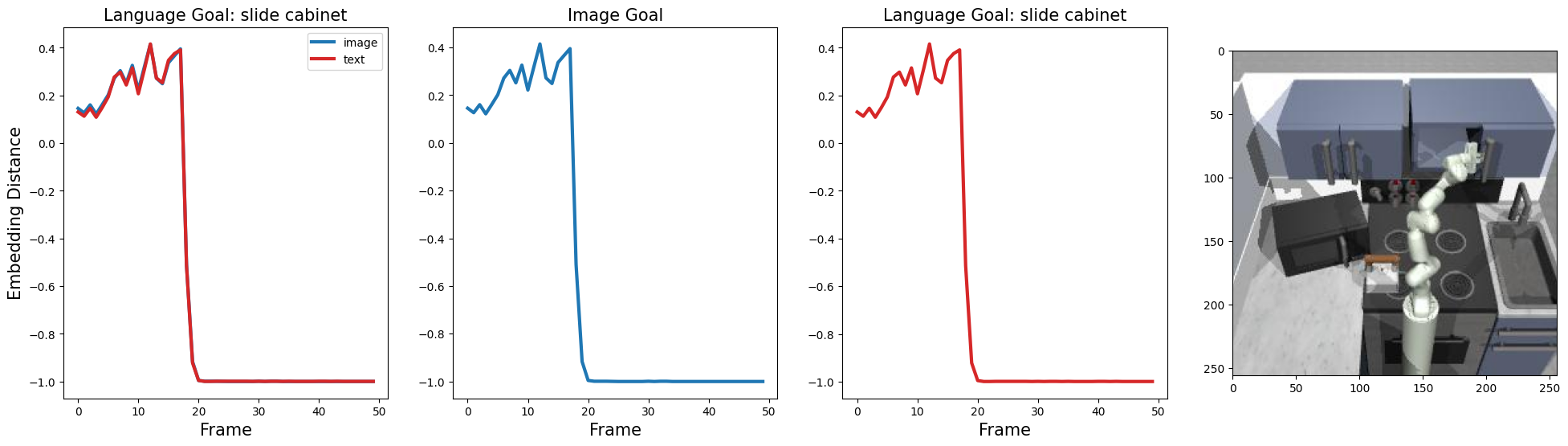}\hfill
  \includegraphics[width=0.7\linewidth]{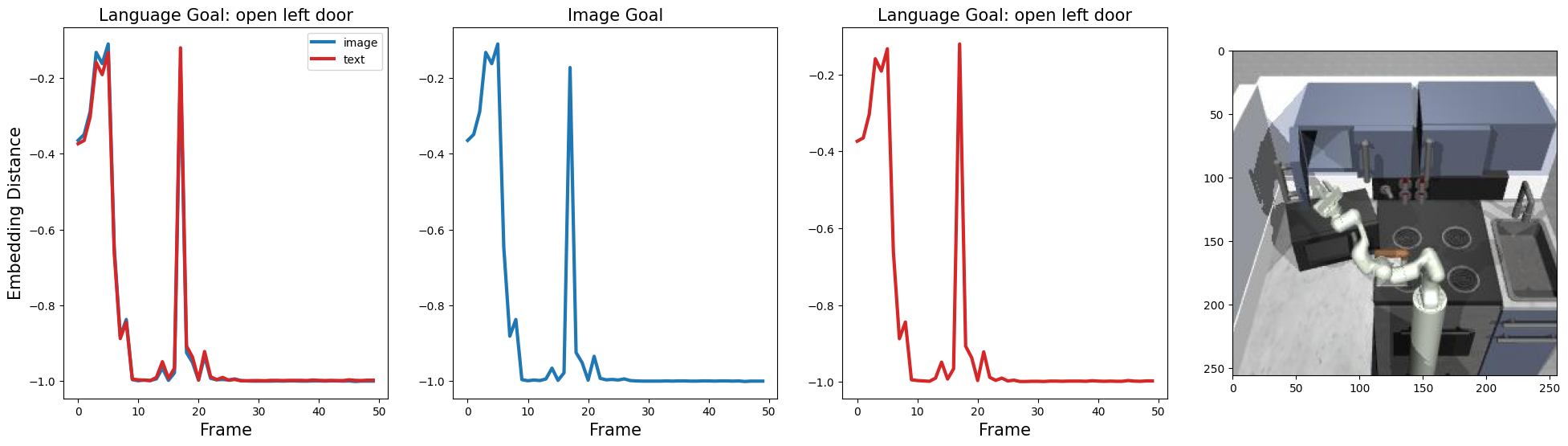}\hfill
  \includegraphics[width=0.7\linewidth]{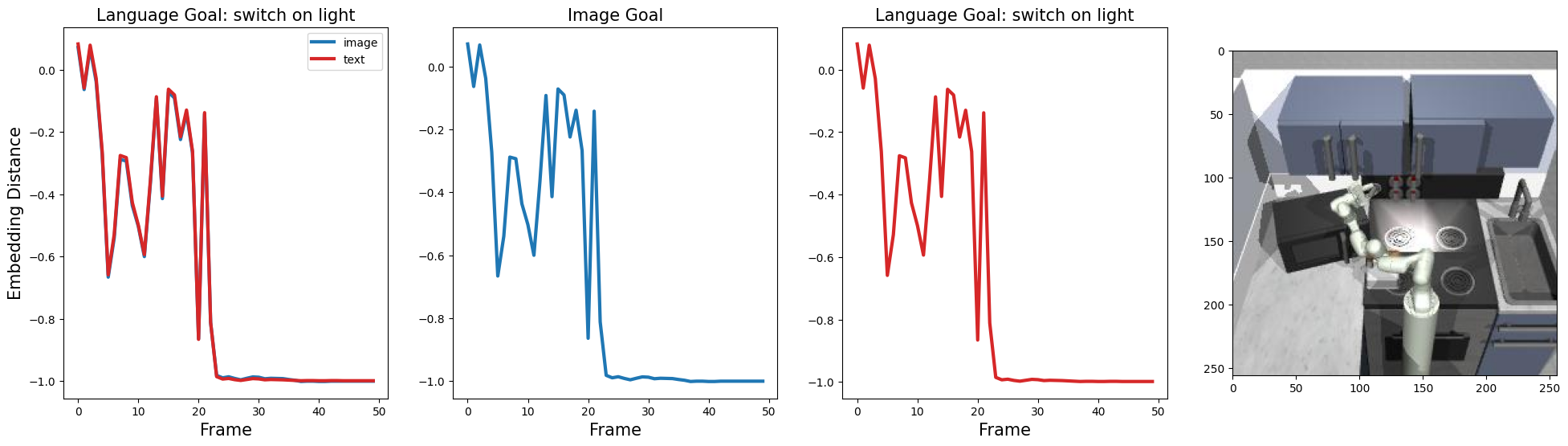}\hfill
    \includegraphics[width=0.7\linewidth]{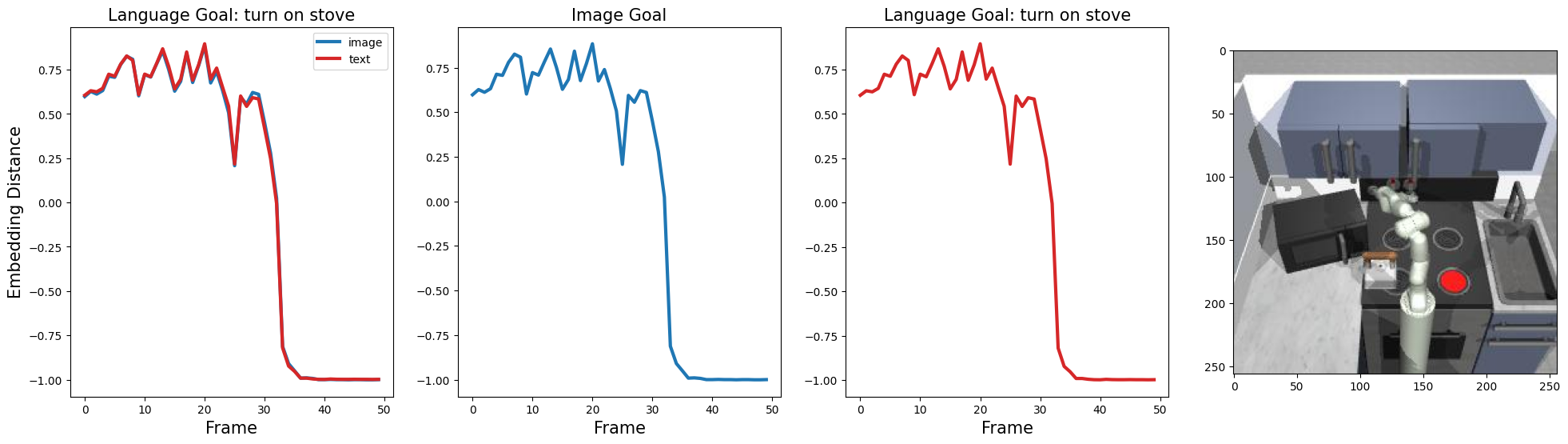}\hfill
  \caption{LIV-EPIC (CLIP fine-tuned) image and language goal reward curves on simulated FrankaKitchen.} 
\label{figure:liv-clipfinetune-frankakitchen-qualitative}
\end{figure}

\begin{figure} 
\centering
  \includegraphics[width=0.7\linewidth]{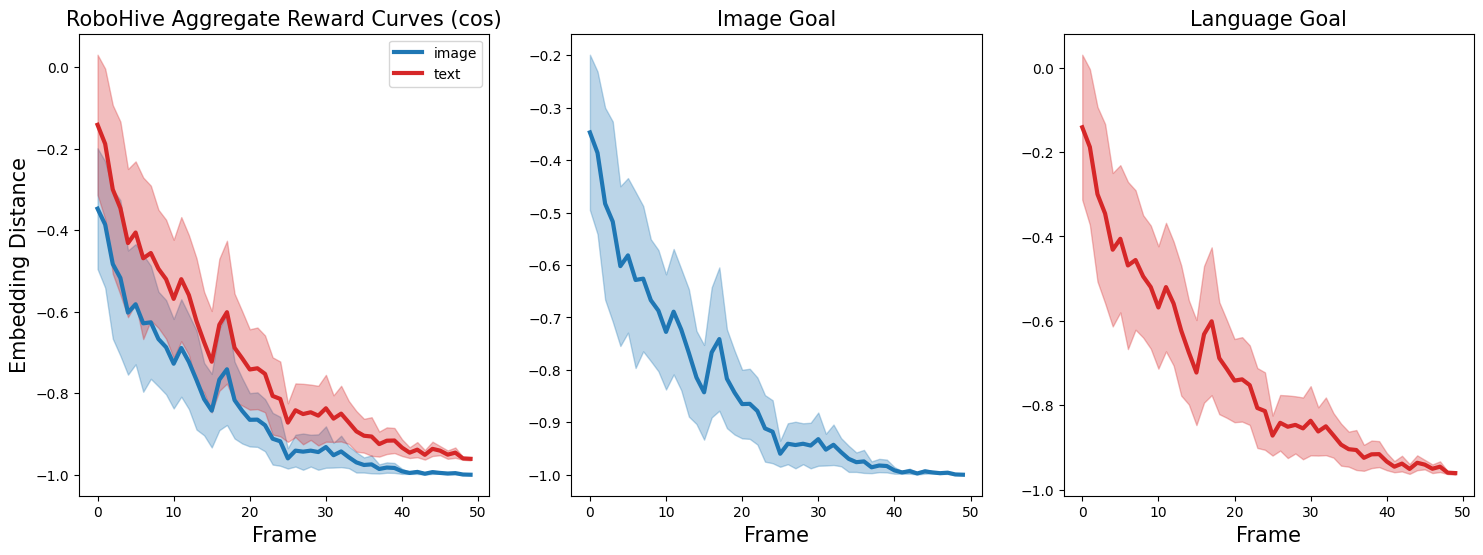}\hfill
  \includegraphics[width=0.7\linewidth]{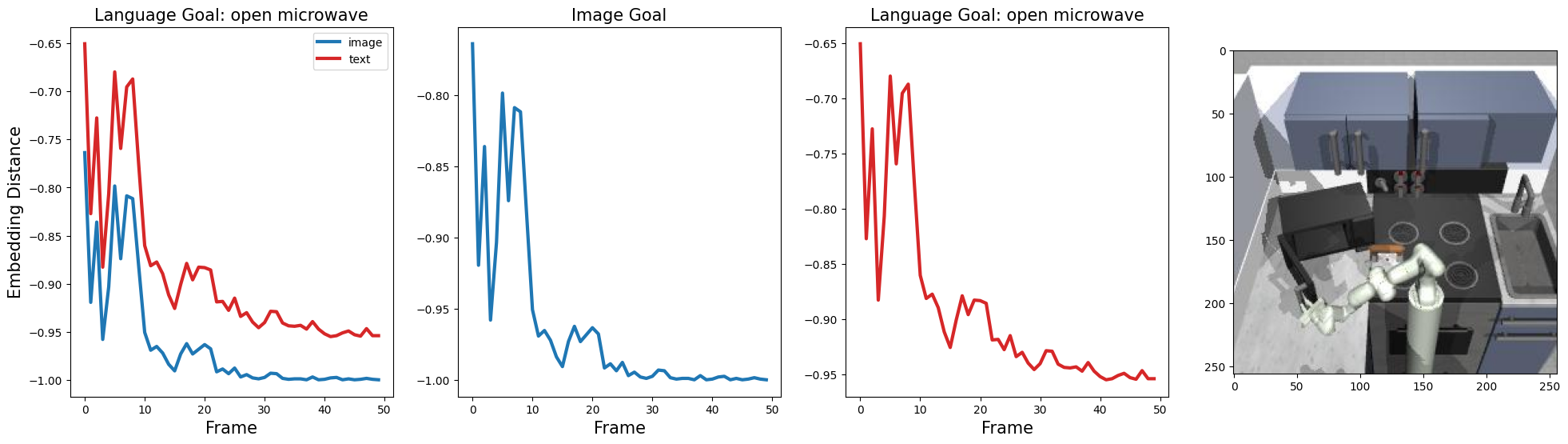}\hfill
  \includegraphics[width=0.7\linewidth]{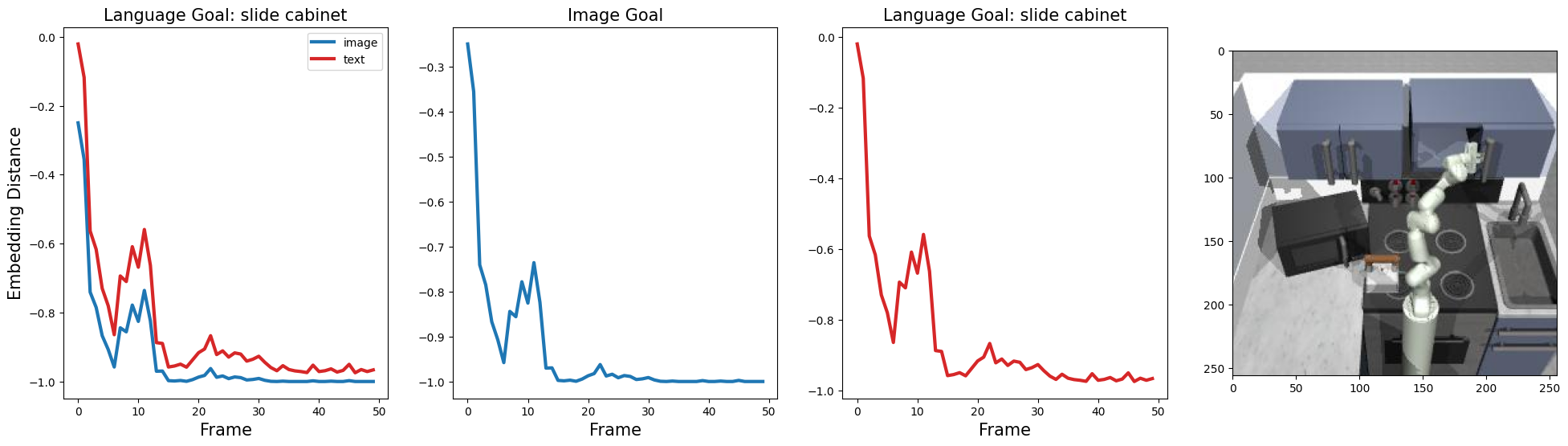}\hfill
  \includegraphics[width=0.7\linewidth]{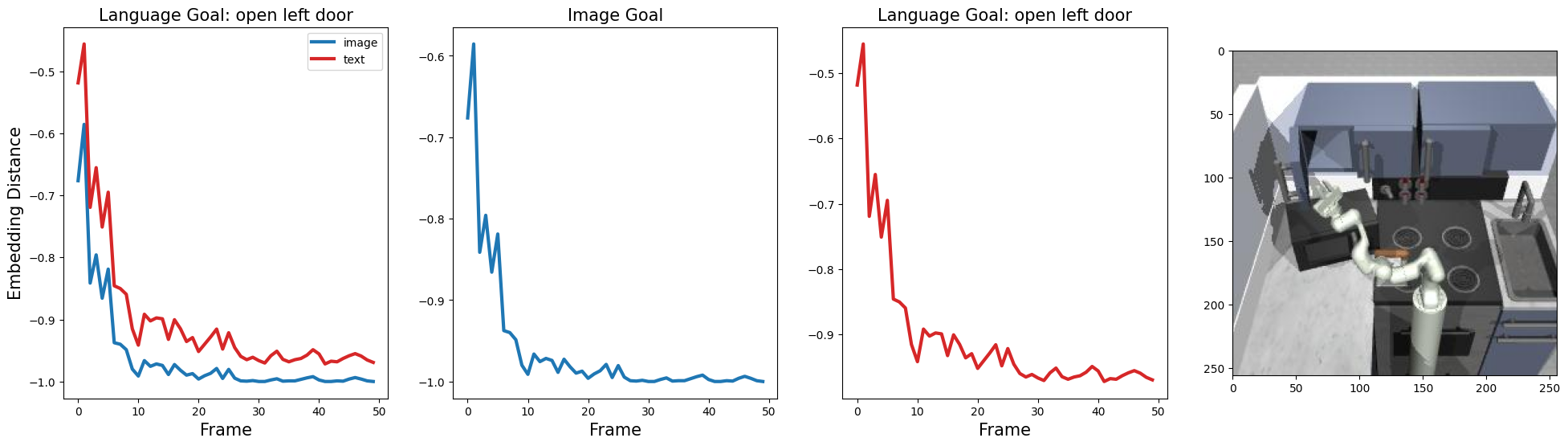}\hfill
  \includegraphics[width=0.7\linewidth]{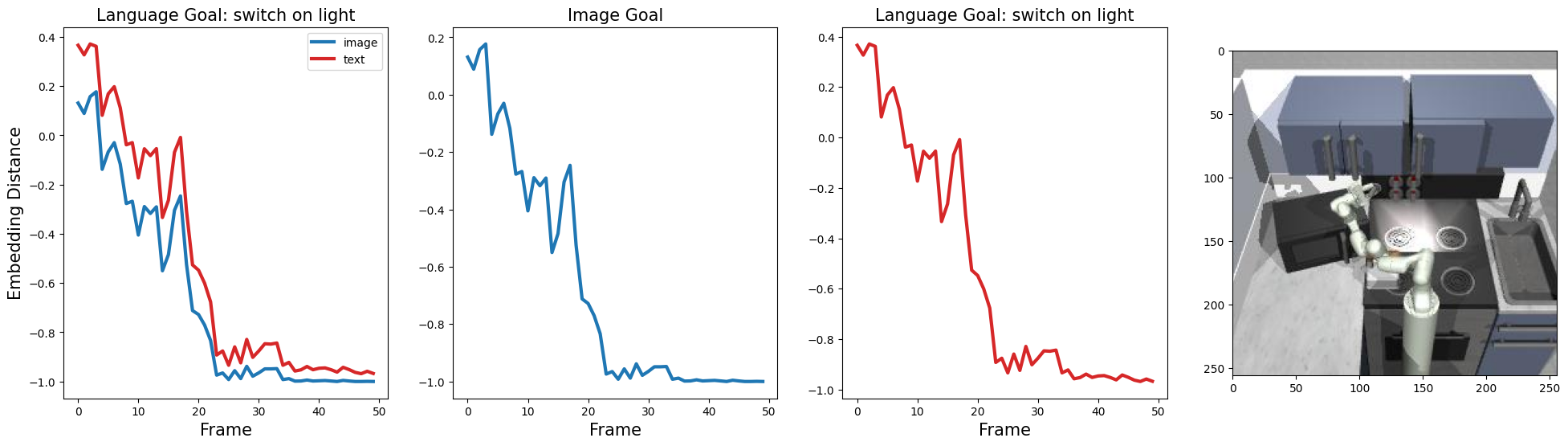}\hfill
    \includegraphics[width=0.7\linewidth]{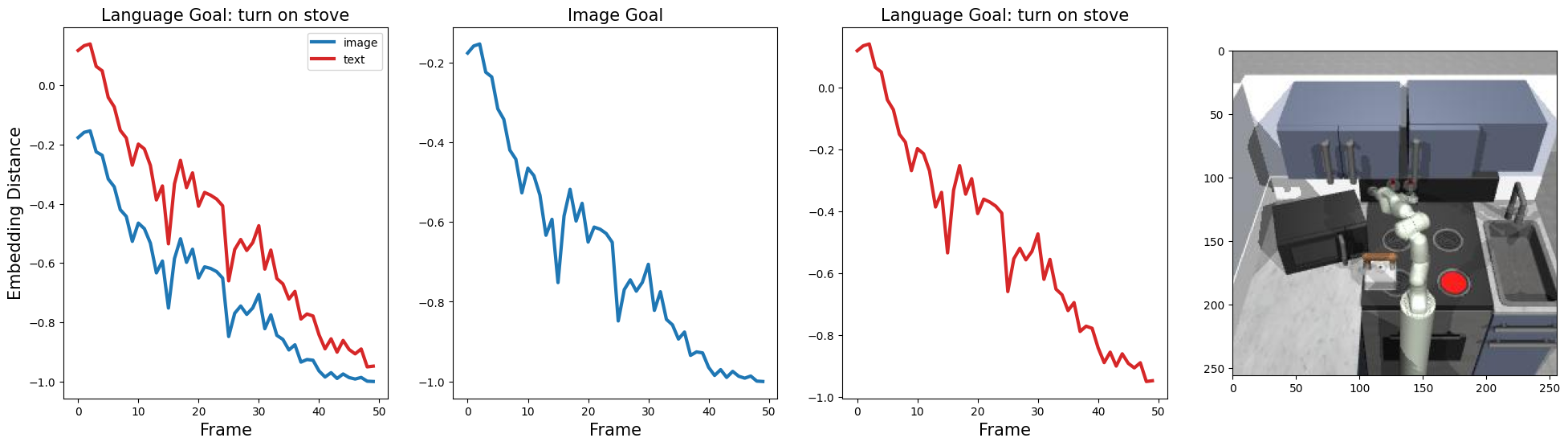}\hfill
  \caption{LIV-EPIC (TCN+CLIP fine-tuned) image and language goal reward curves on simulated FrankaKitchen.} 
\label{figure:liv-tcnclipfinetune-frankakitchen-qualitative}
\end{figure}
\end{document}